%% file: Main.tex
\documentclass[final,12pt]{colt2018} % Include author names

\definecolor{pearYe}{HTML}{FFB733}
\definecolor{pearOne}{HTML}{2C3E50}
\definecolor{pearTwo}{HTML}{A9CF54}
\definecolor{pearTwoT}{HTML}{C2895B}
\definecolor{pearThree}{HTML}{E74C3C}
\colorlet{titleTh}{pearOne}
\colorlet{bull}{pearTwo}
\definecolor{pearcomp}{HTML}{B97E29}
\definecolor{pearFour}{HTML}{588F27}
\definecolor{pearFith}{HTML}{ECF0F1}
\definecolor{pearDark}{HTML}{2980B9}
\definecolor{pearDarker}{HTML}{1D2DEC}
\makeatletter
\let\Ginclude@graphics\@org@Ginclude@graphics
\makeatother
\usepackage[load-configurations=version-1]{siunitx} % newer version

\title[Best of both worlds: Stochastic \& adversarial
 best-arm identification]{\textcolor{pearOne}{Best of both worlds:} \\ \textcolor{pearDark}{Stochastic} 
 	\textcolor{pearOne}{\&} \textcolor{pearThree}{adversarial} 
 	 best-arm identification}
\usepackage[OT1]{fontenc} % Preserve the original Times fonts with XeTeX.
\usepackage{times}
\coltauthor{
	\!\!\Name{Yasin {Abbasi-Yadkori}} \Email{abbasiya@adobe.com}\\
	\addr Adobe Research, USA
	\AND
	\Name{Peter L. Bartlett} \Email{peter@berkeley.edu}\\
	\addr University of California, Berkeley, USA
		  \AND
 \Name{Victor Gabillon} \Email{victor.gabillon@qut.edu.au}\\
 \addr Queensland University of Technology - ACEMS, Australia
 \AND
 \Name{Alan Malek} \Email{amalek@mit.edu}\\
 \addr Massachusetts Institute of Technology, USA
 \AND
 \Name{Michal Valko} \Email{michal.valko@inria.fr}\\
 \addr SequeL team, INRIA Lille - Nord Europe, France
 }

\usepackage{mathtools}
\usepackage{hyperref}
\usepackage{bm}
\usepackage{xcolor}
\usepackage{wrapfig}
\usepackage{tikz}
\usepackage{algorithm}
\usepackage{algorithmic}
\usepackage{thm-restate}

\usepackage{colortbl}
\usepackage{caption} 
\usepackage{enumitem} 
\usepackage[colorinlistoftodos, textwidth=26mm, shadow]{todonotes}
\let\oldtodo\todo
\renewcommand{\todo}[1]{\oldtodo{\scriptsize #1}}

\newcounter{fpcounter}
\renewcommand{\thefpcounter}{\arabic{fpcounter}}

\newenvironment{fp}[2]{%
\refstepcounter{fpcounter}%
\label{#1}%
\noindent{\raisebox{.04cm}{\textcolor{bull}{$\blacktriangleright$}} Experiment~\thefpcounter: }%
}%
{}%

\input{def}

\newcommand{\argmax}{\mathrm{argmax}}
\newcommand{\argmin}{\mathrm{argmin}}

\newcounter{scratchcounter}
\begin{document}

\maketitle

\input{abstract}


\hypersetup{
	colorlinks,
	citecolor=pearDark,
	linkcolor=pearThree,
	urlcolor=pearDarker}

\input{intro}
\input{problemFormulation}
\input{upLowGenAd}
\input{bestOfBoth}
\input{lowerBoundBOB}
\input{upperBoundBOB}
\input{Extensions}
%\input{Discussion}

%\newpage
\subsection*{Acknowledgements}
%\vspace{-.15cm}
We gratefully acknowledge the support of the NSF through grant IIS-1619362 and of the Australian Research Council through an Australian Laureate Fellowship (FL110100281) and through the Australian Research Council Centre of Excellence for Mathematical and Statistical Frontiers (ACEMS).
The research presented was also supported by European CHIST-ERA project DELTA, French Ministry of
Higher Education and Research, Nord-Pas-de-Calais Regional Council,
Inria and Otto-von-Guericke-Universit\"at Magdeburg associated-team north-european project Allocate, and French National Research Agency projects ExTra-Learn (n.ANR-14-CE24-0010-01) and BoB (n.ANR-16-CE23-0003).
We would like to thank Iosif Pinelis for a useful discussion on Bernstein inequalities and Yunkyu Song, Zhihan Xiong, Romain Camilleri, Maryam Fazel, Lalit Jain and Kevin Jamieson for pointing out  and helping correct errors in  previous versions of the proof of Theorem~\ref{th:lowBOB} (corrected in this version).

We used ChatGPT (OpenAI) to assist in reviewing
and revising proofs and improving the wording
of this updated manuscript. Victor Gabillon reviewed
the resulting AI-assisted changes.

\bibliography{biblio}

\newpage
\appendix
\input{appendices}	
\end{document}

%% file: def.tex
\DeclareBoldMathCommand{\I}{I}
\DeclareBoldMathCommand{\e}{e}
\DeclareBoldMathCommand{\f}{f}
\DeclareBoldMathCommand{\g}{g}
\DeclareBoldMathCommand{\a}{a}
\DeclareBoldMathCommand{\b}{b}
\DeclareBoldMathCommand{\c}{c}
\DeclareBoldMathCommand{\d}{d}
\DeclareBoldMathCommand{\m}{m}
\DeclareBoldMathCommand{\p}{p}
\DeclareBoldMathCommand{\q}{q}
\DeclareBoldMathCommand{\v}{v}
\DeclareBoldMathCommand{\V}{V}
\DeclareBoldMathCommand{\x}{x}
\DeclareBoldMathCommand{\t}{t}
\DeclareBoldMathCommand{\X}{X}
\DeclareBoldMathCommand{\Y}{Y}
\DeclareBoldMathCommand{\z}{z}
\DeclareBoldMathCommand{\Z}{Z}
\DeclareBoldMathCommand{\M}{M}
\DeclareBoldMathCommand{\n}{n}
\DeclareBoldMathCommand{\ssigma}{\sigma}
\DeclareBoldMathCommand{\SSigma}{\Sigma}
\DeclareBoldMathCommand{\OOmega}{\Omega}
\DeclareBoldMathCommand{\y}{y}
\DeclareBoldMathCommand{\U}{U}
\DeclareBoldMathCommand{\w}{w}
\DeclareBoldMathCommand{\W}{W}
\DeclareBoldMathCommand{\L}{L}

\DeclareBoldMathCommand{\s}{s}
\DeclareBoldMathCommand{\S}{S}
\DeclareBoldMathCommand{\A}{A}
\DeclareBoldMathCommand{\B}{B}
\DeclareBoldMathCommand{\C}{C}
\DeclareBoldMathCommand{\D}{D}
\DeclareBoldMathCommand{\E}{\mathbb{E}}
\DeclareBoldMathCommand{\G}{G}
\DeclareBoldMathCommand{\H}{H}
\DeclareBoldMathCommand{\P}{\mathbb{P}}
\DeclareBoldMathCommand{\Q}{Q}
\DeclareBoldMathCommand{\R}{R}
\DeclareBoldMathCommand{\X}{X}
\DeclareBoldMathCommand{\mmu}{\mu}
\DeclareBoldMathCommand{\ones}{1}
\DeclareBoldMathCommand{\zeros}{0}

\newcommand{\cO}{\mathcal{O}}
\newcommand{\tcO}{\widetilde{\cO}}

\newcommand{\indicator}[1]{\mathbf 1\{#1\}}

\newcommand{\gap}{\Delta}
\newcommand{\sample}{X}
\newcommand{\pulledArm}{I}
\newcommand{\nArms}{K}
\newcommand{\timeHorizon}{n}

\newcommand{\cumulGain}{G}
\newcommand{\setArms}{[\nArms]}
\newcommand{\mean}{\mu}
\newcommand{\gainVector}{\g}
\newcommand{\estGainVector}{\tilde\gainVector}
\newcommand{\estCumulGainVector}{\tilde\cumulGain}
\newcommand{\estBiasCumulGain}{\hat\cumulGain}
\newcommand{\recomArm}{J}
\newcommand{\bestArm}{k^\star}

\newcommand{\bestArmSto}{\bestArm_{\textsc{STO}}}
\newcommand{\learnerDist}{\p}
\newcommand{\pullsNumber}{T}
\newcommand{\advPro}{\textsc{ADV}}
\newcommand{\stoPro}{\textsc{STO}}
\newcommand{\basePro}{\textsc{BASE}}
\newcommand{\refarm}{\bar k}

\newcommand{\dividefac}{a}

\newcommand{\setArmsmo}{[2:\nArms]}
\newcommand{\earlyPhase}{L}
\newcommand{\SR}{{\normalfont \texttt{SR}}}

\newcommand{\SH}{{\normalfont \texttt{SH}}}

\newcommand{\ProbabilityONE}{\normalfont \texttt{\textcolor[rgb]{0.5,0.2,0}{ProbabilityONE}}}
\newcommand{\RULE}{{\normalfont  \texttt{Rule}}}
\newcommand{\Unif}{{\normalfont \textsc{UNIF}}}
\newcommand{\Both}{{\normalfont \textsc{{BOB}}}}
\newcommand{\Pone}{{\normalfont \texttt{\textcolor[rgb]{0.5,0.2,0}{P1}}}}
\newcommand{\complexity}{H}
\newcommand{\complexitySR}{\complexity_{\SR}}
\newcommand{\complexityUnif}{\complexity_{\Unif}}
\newcommand{\complexityBoth}{\complexity_{\Both}}
\newcommand{\complexityProOne}{\complexity_{\Pone}}
\newcommand{\Exp}{\E}
\newcommand{\Pro}{\P}

\newcommand{\ProErr}{e}
\newcommand{\eventc}{\xi}

\newcommand{\eventbob}{\xi}
\newcommand{\expectCumul}{M}

\newcommand{\TODO}[1]{
\ifmmode
\text{\textcolor{red}{TODO: #1}}
\else
\textcolor{red}{TODO: #1}
\fi
}

\newcommand{\CommaBin}{\mathbin{\raisebox{0.5ex}{,}}}

\renewcommand{\epsilon}{\varepsilon}
\renewcommand{\hat}{\widehat}
\renewcommand{\tilde}{\widetilde}
\renewcommand{\bar}{\overline}

\newcommand{\var}{\sigma^2}
\newcommand{\range}{b}

\newcommand{\lowBrv}{\ell}
\newcommand{\upBrv}{u}

\newcommand{\phaseTime}{n}
\newcommand{\propTime}{a}
\newcommand{\propTimeVec}{\a}
\newcommand{\propTimeVecSpa}{\A}

\newcommand{\Integer}{\mathbb{N}}

\newcommand{\Gaps}{\boldsymbol{\Delta}}

%% file: abstract.tex
% !TEX root = Main.tex
%
\begin{abstract}%   <- trailing '%' for backward compatibility of .sty file
We study bandit best-arm identification 
 with arbitrary and potentially adversarial rewards. 
    A simple random uniform
   learner obtains the optimal rate of error in the 
    adversarial scenario. However, this 
    type of strategy is suboptimal when the rewards are sampled 
    stochastically. Therefore, we ask:
    \emph{Can we design a learner that
   performs optimally in both the stochastic  
   and adversarial problems while not being aware of the nature of the rewards?}
        First, we show that designing such a learner is impossible in general. 
        In particular, to be robust to adversarial rewards, we can 
        only guarantee optimal rates of
         error on a subset of the stochastic problems. We introduce a new complexity measure, $\complexityBoth$,
         and establish a lower bound on the stochastic error
         rates achievable by learners that must also remain
         robust to adversarial rewards. 
Finally, we design a simple parameter-free algorithm and show that its 
probability of error matches (up to log factors) the lower bound in 
stochastic problems, and it is also robust to adversarial ones.
\end{abstract}%

%% file: intro.tex
% !TEX root = Main.tex
\section{Introduction}
In best-arm identification~\citep{Maron93HR,Bubeck09PE}, 
the \emph{learner} tries to identify the \emph{arm} (option,
decision) with the highest (expected) average \emph{reward} among $\nArms$
given arms.  At each round $t$ of the \emph{game} (the interaction
of the learner with its environment), each arm $k$ is assigned a
reward $\gainVector_{k,t}$. On this very same round $t$, a learner
chooses an arm $\pulledArm_t$ and \emph{only} observes the reward of
that arm, $\gainVector_{\pulledArm_t,t}$, while the rest of the vector
$\gainVector_{t}$ is hidden from the learner.

Typically, we assume stochastic rewards: Each arm $k$ is associated with
a distribution $\nu_k$ and, for all $k$ and $t$, the
$\gainVector_{k,t}$ is sampled i.i.d.\,from $\nu_k$. In this paper, we
aim for a robust solution for this setting and we allow the possibility
that the rewards are non-stochastic. The rewards could  have been chosen even by an oblivious \emph{adversary}: The
rewards are fixed before the start of the game but they are not necessarily
drawn i.i.d.\,from a distribution. We focus on \emph{fixed-budget}, where a
total number of arm pulls $\timeHorizon$ is fixed and the learner
wishes to identify, as accurately as possible, the arm that attains the
highest cumulative reward. We also discuss possible extensions to the fixed-confidence 
case, \textit{adaptive} adversaries, and other settings in Section~\ref{s:exte}. % are also discussed therein.% Section~\ref{s:exte}.

% Because there is no underlying distribution
% generating the rewards, it is difficult to define the
% \emph{fixed-confidence} version of the problem: the high probability
% guarantees are usually with respect to the randomness of the algorithm
% and the rewards, but the later do not have a probability measure
% defined on them.

Given $\gainVector$ and round $\timeHorizon$, we define the
cumulative gain of arm $k$ as
$\cumulGain_k=\sum_{t=1}^{\timeHorizon}\gainVector_{k,t}$.  Same
as for adversarial bandits for cumulative-regret~\citep{Auer02NM}, the best arm in hindsight is defined as
$\bestArm_{\gainVector}= \argmax_{k\in\setArms}
\cumulGain_{k}$.\footnote{An alternative definition of best-arm identification
	could be predicting
	$\argmax_{k\in\setArms}\gainVector_{k,\timeHorizon+1}$. However, this
	is impossible in the adversarial case where
	$\gainVector_{k,\timeHorizon+1}$ can be chosen arbitrarily without
	any dependence on $\gainVector_{k,[\timeHorizon]}$.}
%\todomi{is there any other paper that we can cite, that uses this definition of an adversarial best arm?
%if not, maybe we could spend couple sentences arguing, why is this a good definition. 
%Or at least to the motivation examples below.
%}
% Note that in 
%the stochastic case the best arm is the arm with the highest mean 
%$\bestArm_{\gainVector}= \argmax_{k\in\setArms} \mean_{k}$.
In a similar way, we define the gaps in hindsight with respect to~$\gainVector$ between two arms $k$ and 
$j$ as $\gap^\gainVector_{k,j}\triangleq\frac{1}{\timeHorizon}
\sum_{t=1}^{\timeHorizon}(\gainVector_{k,t}-\gainVector_{j,t})$, 
giving a good proxy for the difficulty of discriminating between
these two arms even in an adversarial environment. 
We design a learner and show that the probability\footnote{Note that in our setup, 
	given~$\gainVector$, the randomness comes solely from the potential 
	internal randomization of the learner.} of error at round $\timeHorizon$ \emph{given} $\gainVector$ 
	against 
any fixed adversarial reward design~$\gainVector$
	is 
bounded by a measure of complexity depending on the 
gaps in hindsight with respect to~$\gainVector$ and~$\timeHorizon$. 
%We consider primarily the fixed-budget, i.e., the number of rounds 
%of the game is fixed in advance
% to $\timeHorizon$ and known both by the adversary and the 
%learner. 

Next, we discuss the motivations for studying non-stochastic best-arm identification. 
First, learning in the presence of adversarial data implies robustness. In a real-world best-arm 
identification, such as clinical trials or online ad recommendation, 
the assumption of i.i.d.\,data may not be valid. For instance, successive rewards from the same arm could
be correlated. It is also possible that an  
adversary is trying to obscure the correct results: For example, the adversary might use a 
botnet to make the learner sell more ads. %\todovout{maybe we should talk about this eaxmple}. 
As discussed 
by~\citet[Section~3]{Bubeck12RA},
a deterministic learner or a learner that eliminates arms with low observed cumulative reward in an early stage of the game
could be easily fooled by an 
adversary feeding it uninformative rewards in each of 
its deterministic pulls or in early stages of the game. 
Therefore, an efficient learner needs to
employ internal randomization and pull each arm with a positive probability 
$\learnerDist_{k,t}$, conditional on the information available immediately
before the action at round $t$.
%Additionally, if the learner stops pulling an 
%arm after some time because it is believed to be clearly 
%suboptimal, the adversary could then change the value of 
%the arm for the rest of the game and fooling again the 
%learner. Therefore the learner will have to pull every 
%arm with a minimal probability.
%\todov{'only?' no! 'already' maybe (P1 is not very conservative and also is near optimal)}

Best-arm identification with
completely adversarial rewards is so difficult that a very conservative approach is already near-optimal. In Section~\ref{s:ULGenAd} we show that by playing the arms \emph{uniformly at random}, the  
learner obtains the optimal gap-dependent rates of error 
against worst-case adversarial sequences.
In the stochastic case, however, picking arms uniformly at random is 
suboptimal. This reveals the \textbf{\textcolor{pearOne}{best of both worlds (BOB)}} 
question: \emph{Is there a learner that attains the 
	optimal error rates in both the stochastic and adversarial settings 
	without knowledge of the environment?}

%Such an algorithm would  solve our motivating examples: we 
%want to solve best-arm identification but in a way that is robust to 
%non-stochastic rewards, and such a BOB learner would allow this robustness 
%without any loss of correctness guarantees. There is also hope, as 
%near optimal BOB results exist in the cumulative regret setting~\citep{Bubeck12BB,Auer16AA}.

%\todoa{I think this entire paragraph should go, except for the last
% sentence or two\\ Victor : not sure, when I talked to people about it, they thought first that it was an easy problem and we were thinking that too for a long time!}

To study the above question, we face a number of \textbf{challenges}: 
How can a learner maintain the randomization and
exploration needed for adversarial robustness while
sampling the most promising arms more frequently? 
Can the learner detect
if the rewards are stochastic or adversarial? What is an appropriate estimate 
of $\cumulGain_{k}$? % is to be used for the arm $k$? 
In the stochastic case, after arm $k$ has been sampled at least once,
$\estBiasCumulGain_{k} \triangleq
\frac{\timeHorizon\sum_{t=1}^{\timeHorizon}
	\indicator{\pulledArm_{t}=k}\gainVector_{k,t}}{\sum_{t'=1}^{\timeHorizon}\indicator{\pulledArm_{t'}=k}}
$ %\todom{more clean?}
is commonly used, but if not used carefully, it can be easily biased 
by an adversary. Before the first sample, an initialization convention
is required. In the adversarial case,
$\estCumulGainVector_{k} \triangleq \sum_{t=1}^{\timeHorizon}
\frac{\gainVector_{k,t}}
{\learnerDist_{k,t}} \indicator{\pulledArm_{t} = k} $ is usual but it can have a 
high variance if $\learnerDist_{k,t}$ is small (for rewards in $[0,1]$,
the sum of its one-step conditional variances is bounded by
$\sum_{t=1}^{\timeHorizon} (1/ \learnerDist_{k,t})$).  Controlling this
potentially large variance, especially when dealing with a stochastic
problem, is one of the main challenges that we face. In particular,
high variance can happen %in the beginning of the game 
if
a learner explores uniformly for too long.

\paragraph{Our contributions}
%To the best of our knowledge, %\todovout{is it fair to say that, knowing there is a competing submission at ICML?}
We consider a new 
%we are the first to consider this 
formulation of the adversarial 
best-arm identification. 
%and show that uniform sampling is optimal. 
We study whether 
a BOB result is possible for fixed-budget best-arm
identification. We answer the question negatively and show that for a class of stochastic problems,
no robust (to adversary) learner can achieve the optimal stochastic error rates.
%that BOB results are only possible in certain 
%scenarios depending on the structure of the gaps in the stochastic problems.
To prove this result, we introduce a new measure of complexity of the 
stochastic part of the task. This measure of complexity gives the 
problem-dependent error rates that a robust learner
can guarantee in any stochastic problem. 
% This complexity is related to a lower-bound on the
%probability of error that one learner can obtain in the stochastic
%case while still being consistent in the adversarial case. 
We study
this new complexity in different stochastic regimes and provide several 
positive examples where BOB is possible but also several negative ones. 
%This results proves that their
%exist regimes in which any optimal learner for the stochastic case can
%be tricked by adversaries while there also exist regimes in which
We notice that even in stochastic problems where our lower bound seems 
to indicate that BOB is achievable, the robust uniform learner is clearly suboptimal.
Therefore there is a need for a
new algorithm to bridge the gap.  In Section~\ref{s:UpperBOB}, we
design a simple parameter-free learner and show that its error rate matches, up to log
factors, the lower bound in the stochastic case
as well as those of the uniform strategy in the adversarial case.  
%In
%Section~\ref{s:exte}, we discuss extensions of our
%results.\vspace{.2cm}
%
%\subsection{Related Work}
%
% TODO(Lean release): Replace the placeholder below
% with the public GitHub repository URL after the
% complete formalization and verification are finished.
\paragraph{Formal verification}
The principal theoretical results of this paper have been formalized in Lean 4.
The machine-checked proofs and verification instructions are publicly available
at \url{https://github.com/victorgabillon/best-of-both-worlds-bandits}.

\paragraph{Related work} The stochastic best-arm identification was introduced in 
the fixed-budget setting by~\cite{Bubeck09PE} and~\cite{Audibert10BA}.
Refined upper and lower bounds 
can be found respectively in the works of~\cite{Karnin13AO} and~\cite{Carpentier16TB}. 
For cumulative regret, the BOB
question was raised by~\cite{Bubeck12BB}. \cite{Seldin14OP} 
gave a practical algorithm for addressing the same problem 
(see also~\citealp{Seldin17IP}). A lower bound and a refined upper bound for the 
problem were given by~\cite{Auer16AA}.

Best-arm identification has been studied in a different
non-stochastic setting by~\cite{Jamieson16NS} and~\cite{Li16HA}. 
At round $t$ for its $m$-th pull of arm $k$, 
their learner observes $\gainVector_{k,m}$, whereas
our learner  observes $\gainVector_{k,t}$.
% Therefore 
%in their setting, as $\timeHorizon$ grows, all 
%$\gainVector_{k,t}$ are observed and 
%no information is left hidden forever,\todom{our setting is potentially harder, but can we claim is it provable harder?\\
%Victor: dont know} it is 
%delayed until next pull. 
Moreover, their work % of~\cite{Jamieson16NS}
is specifically tailored for online hyperparameter 
optimization of learning methods, where the value of each 
hyperparameter is assumed to converge, at some unknown rate, 
to its true value as it is given more resources (e.g., data or time). 
Therefore, their objective is to identify the best arm once the 
convergence has happened while in our setting, we do not assume 
any convergence and are interested instead in comparing the 
cumulative rewards $\cumulGain_k$ of each arm. By assuming 
convergence of the arms and asymptotically having the learner 
observe all the rewards, they prove that a state-of-the-art deterministic learner from the vanilla stochastic best-arm identification also solves their setting.

%\todov{it would be good to understand what differentiates us from Jamieson more deeply. Is it that in our case the time has more importance? important also to get to have a good motivation example section! or somehow the fact that if we access one data we miss the rest}

\cite{Allesiardo2017SL} and \cite{Allesiardo17NS} analyze
a non-stationary stochastic best-arm problem in a fixed-confidence 
setting but under the assumption that the game can be split into independent
sub-games where the identity of the best arm does not change. This 
assumption precludes many of the hard examples where the adversary 
tricks the learner in the early stages of the game. The
learner can then simply use a randomized version
of Hoeffding Races and safely stop pulling arms. 
Also, while our notion of gaps is defined for round $\timeHorizon$, 
they define gaps at any intermediate round of the game 
but then consider the minimum gaps over time for each arm.
This leads to a larger notion of complexity and permits ignoring the variance of the estimates.
%They also study a stochastic problem with added corruption 
%which is outside the scope of this paper.
\vspace{.2cm}

%%\cite{Allesiardo2017SL,Allesiardo17NS} analyze
%a fixed confidence setting where several notions of 
%non-stochasticity are assumed. First the identity of 
%the best arm does not change throughout the game, second, 
%the adversary uses a limited and known amount of corruption that can be added 
%to the stochastic rewards, third, the best arm can change a 
%limited number of times during the game and needs to be 
%tracked. The proposed algorithms are randomized variants 
%of Hoeffding races with no specific guarantee in the 
%stochastic cases.\vspace{.2cm}
%
%
%\subsection{Motivating Examples}\label{ssec:MotivMe}
%
\paragraph{Two  types of applications}  
%\todov{it would be nice to have good section here with good 
%example and a good understanding of what our algorithm is good for}
%Our setting is well suited for problems when $\timeHorizon$ 
%different tests can be run but only for one option out of 
%$\nArms$ at a a time while we would like to discriminate the 
%best option compared on all the potential test. \todom{here it looks like the "tests" could be potentially different}
%We make no  
%stationarity or independence assumption on the rewards as 
%those do not
%hold in practice for many applications. 
%For instance one could test one product out of $\nArms$ 
%every day for  year before deciding to commercialize only 
%one. There could be $\timeHorizon$ different configurations 
%on which we have to test $\nArms$ ads but can only test per 
%configurations.
%Our method will be then robust to dependence on the season, 
%the time of the day, or 
%internal complex dependencies.
% However adopting 
% a worse case approach can lead to having over-conservative 
% learner. 
% Here we are looking for learner with robustness guarantees 
% and able to take advantage of \emph{easy data} in case there is 
% no such complex dependence or non-stationarity.

%In the best arm identification task, the learner identifies the 
%best arm from a limited 
%Our algorithm is applicable to two types of problems. 
In the first type, the learner believes that future rewards will be similar to the ones already observed. 
Here, from the $\timeHorizon$ observations, the learner could use the estimated best arm in the upcoming rounds. 
In an expert setting, we might test the quality 
of experts for a limited time before committing to one.
We may also optimize the hyperparameters of algorithms. 
Contrary to \cite{Jamieson16NS}, our objective is to find the 
arm (hyperparameter) with the highest cumulative value over a test set, rather than the performance that is achieved by the hyperparameter after convergence. %and not a performance that this hyperparameter reaches after convergence. %\todom{above we say that we are different from hyperband}
In the second type, there are no new upcoming rounds.
The identity of the best arm is used to take an action based on the collected data.
%In the first case,
% this would hold even hold if the environment is not fully i.i.d~ stochastic.
%In a fixed-budget setting the learner uses a test phase composed of $\timeHorizon$ rounds to find the best arm. This learner might only want to know the identity of the best arm for this fixed set of data or might want to use this knowledge for new upcoming data. In the latter case,
% thought 
%we do not assume formally any relation between the samples in the test
%phase and the rest of the time after the test phase, we still can expect that 
%there is a link between those and that finding the best 
%arm for the initial phase might have a link with the rest 
%of time. Therefore one should just try to identify as well as possible the best for the given exploration part.
%
%In the this latter case, one can think of a
%. A company checking 
%regularly over the quality of different manufacturing 
%devices might want to operate change at the end of the 
%year on the most deficient one.
As an example, consider a law-enforcing agency that collects 
information periodically from different targets, one target per week, and at the end of the year, it  
decides which target to investigate 
thoroughly over its activities in the past year.
This problem can 
be seen as a game between the agency and the malignant 
targets. Therefore, it would be useful to have algorithms that are 
robust to worst-case scenarios but still take advantage 
of easy data in case malignant targets actually do not take precautions to avoid being caught. %\todom{want? perhaps: ? }
Finally, in our setting, we are not trying to be robust to 
corruption of the data, as we want to find the best arm whether it includes corruption or not,
 unlike %\vspace{.2cm}
\citet{Altschuler18BA} who study corrupted bandits.
%\todom{and new paper of Emilie and Pratik?Victor: \\
%its cumulative regret}
%
%%%%%%%%%%%%%%%%%%%%%%%%%%%%%%%%%%%%%%%%
%

%% file: problemFormulation.tex
% !TEX root = Main.tex 
\section{Problem formulation for adversarial and stochastic rewards}
\label{s:set}
%
%We complete the description of the adversarial and stochastic 
%problems initiated in the introduction.    
Let $[a:b]=\{a,a+1,\ldots,b\}$ with
$a,b\in\Integer$, $a\leq b$, and $[a]=[1:a]$.
%For a vector $\w\in \Real^n$, $n\in\Integer$,  $\maxnorm{\w}= \max_{k\in[n]} |\w_k|$. 
A vector indexed by both round $t$ and a specific element
index $k$ is $\w_{k,t}$. We detail the general game protocol in Figure~\ref{fig:protocol}.
\begin{figure}[H]
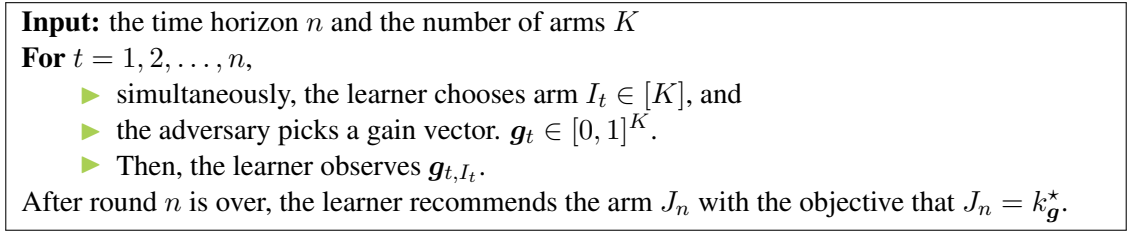

\centering
\framebox{
	\begin{minipage}{.95\textwidth}
		\textbf{Input:} the time horizon $\timeHorizon$ and the number of arms $\nArms$\\
		\textbf{For} $t=1,2, \ldots, \timeHorizon$, 
		
$\qquad$\raisebox{.00cm}{\textcolor{bull}{$\blacktriangleright$}}~ simultaneously, the learner chooses arm $\pulledArm_t\in [\nArms]$, and

$\qquad$\raisebox{.00cm}{\textcolor{bull}{$\blacktriangleright$}}~  the adversary picks a gain vector
$\gainVector_t\in [0,1]^\nArms$.

$\qquad$\raisebox{.04cm}{\textcolor{bull}{$\blacktriangleright$}}~ Then, the learner 
observes $\gainVector_{\pulledArm_t,t}$.

		After round $\timeHorizon$ is over, the learner 
		recommends the arm  $\recomArm_{\timeHorizon}$ with the objective that $\recomArm_{\timeHorizon} = \bestArm_{\gainVector}$.
	%	\todov{it should be $\recomArm_{n+1}$ ot at least dont say the game stop when $t=n$ be careful with conventions! (starting to change now!)}
	\end{minipage}
}
\caption{General protocol of the adversarial setting. The adversary 
	is oblivious.}\label{fig:protocol}
\end{figure}
%\todomi{may be confusing that adversary observes $\gainVector_{t,\pulledArm_t}$ in General Protocol}
%
\paragraph{Adversarial case} 
      The adversary is denoted as \advPro{}. It is the process that 
     generates $\gainVector$.
     %
   %  In hindsight, at time $t$, %at the end of the game when 
  %   $t=\timeHorizon$, 
    % we can order the arm with respect to their 
     %cumulative gains $\cumulGain_{k}= \sum^{\timeHorizon}_{t=1} \gainVector_{k,t}$. %\todo{X or g?} 
%At the end of the game, $t=\timeHorizon$. 
 Let $(m)$ denote 
the index of the $m$-th best arm in $\setArms$
 and $\cumulGain_{(m)}$ its corresponding 
 cumulative gain
 so that $\cumulGain_{(1)}> \cumulGain_{(2)} \geq \ldots 
 \geq \cumulGain_{(\nArms)}$. Dually to $(\cdot)$, $\langle k \rangle$~denotes 
 the rank of arm $k$ after sorting according to~$\cumulGain_{(\cdot)}$ so that $\langle (k) 
 \rangle=(\langle k \rangle)=k,~ \forall k\in \setArms$.
%\todov{using accolade for the rank is too confusing}
Without loss of generality, note that we assumed there 
exists
 a unique best arm $\bestArm_{\gainVector}=(1)$. %\todov{really no loss of generality?}
%
%\todov{think of how to define here or not \advPro{} }
%
For each arm $k\in \setArms$, we define the gap 
$\gap^\gainVector_k$ as
\begin{align*} 
\timeHorizon\gap^\gainVector_{k} \triangleq \begin{cases}  
\cumulGain_{(1)}-\cumulGain_{k} & \text{if } k \neq \bestArm_{\gainVector}, \\
\cumulGain_{(1)}-\cumulGain_{(2)}       & \text{if } k =\bestArm_{\gainVector}.
\end{cases}%\cdot
\end{align*}
%
%\todov{choose how much we want to put that $\gainVector$ as an index everywhere}
The gap can also be written as 
$\timeHorizon\gap^\gainVector_{k} = \big|\max\limits_{i \neq k} 
\cumulGain_{i} - \cumulGain_{k}\big|$. 
%
%
%\todov{call him learner or forecaster?}
%\todov{discuss adaptive or oblivious adversary and think about
% what is the difference for us for the results}
% \todomi{indeed, I think P. Auer makes a distinctions about these}
 $\recomArm_{\timeHorizon} \in \setArms$ is the  arm 
returned by the learner at the end of the exploration phase.  
Given a budget
  $\timeHorizon$ and a fixed adversarial set of rewards 
  $\gainVector$ designed by an  adversary \advPro{}, the
   performance of the learner is measured by the 
   probability
    $\ProErr_{\advPro(\gainVector)}(\timeHorizon)$ 
    of not identifying the best 
    arm, i.e.,~$\ProErr_{\advPro(\gainVector)}(\timeHorizon)
     \triangleq \Pro\left(\recomArm_{\timeHorizon} \neq \bestArm_{\gainVector}\right).$
     The smaller $\ProErr_{\advPro(\gainVector)}(\timeHorizon)$, 
    the better the learner. The probability 
    is taken over the randomness of the learner as $\gainVector$ is fixed.
    An alternative definition of  the best arm, i.e., with
     highest $G_{k}$ in expectation over adversarially 
     sampled $\gainVector$, would lead to an impossible problem. 
     Indeed, it could happen that the best arm in expectation is actually 
     always clearly suboptimal on any realization $\gainVector$.
%
%\todomi{remark: It is somehow uncommon to have $\gainVector$, \emph{sampled} from the adversary..., so  we should comment that it makes sense here.}
%
%
%
%The learner will often chose the arm from a distribution
% at time $t$, $\learnerDist_t$, where $\learnerDist_{k,t}=\Pro(\pulledArm_t=k)$.
%
We define the random variables $\estGainVector_{k,t}$ as
\begin{equation*}
\estGainVector_{k,t} 
\triangleq
 \frac{\gainVector_{k,t}\indicator{\pulledArm_t = k}}
 {\learnerDist_{k,t}}\CommaBin
\end{equation*}
for 
arm $k$ at round $t$ and for which  $\Exp_{\pulledArm_t\sim\learnerDist_{t}}
\left[\estGainVector_{k,t}\right]=\gainVector_{k,t}.$
We similarly define $\estCumulGainVector_{k,t} \triangleq \sum_{t'=1}^{t} \estGainVector_{k,t'}.
$
%we have for all $k\in\setArms$,  $\Exp_{\pulledArm_{[t]}\sim\learnerDist_{[t]}}\left[\estCumulGainVector_{k,t}\right]=\cumulGain_{k,t}$.
%\todov{improve the simulation under the expectation}
%
%We say that the adversary is oblivious if $\gainVector_{k,t}$ is independent of $\pulledArm_1,\ldots,\pulledArm_{t-1}$ for all $k\in\setArms$ and $t\in[\timeHorizon]$. Otherwise the adversary is said to be adaptive. 
 %
 %
\paragraph{Stochastic case}
 In the stochastic setting~\citep{Audibert10BA}, denoted
  \stoPro{}, the distribution~$\nu_{k}$ of arm $k$ is bounded in $[0,1]$ with mean $\mu_{k}$. % and variance $\sigma^2_{k}$. 
 The ordering $(\cdot)$ is such that $\mean_{(1)}>\mean_{(2)}
 \geq\ldots\geq\mean_{(\nArms)}$, as we assume  the uniqueness of the best arm without loss 
 of generality.
 The distributions $\{\nu_{k}\}$ are unknown to the learner.
 %At each round $t$, $\gainVector_{k,t}$ is an independent sample drawn from the 
 %distribution $\nu_{k}$.
% Let $\pullsNumber_{k}(t)$ be the number of times that arm 
% $k$ has been pulled by the end of round $t$.
% %
% Usually the forecaster estimates the expected value of each arm by computing the
% average of the samples observed over time $\widehat{\mu}_{k}(t) = \frac{1}{T_{k}(t)}\sum_{t=1}^{\timeHorizon}\indicator{\pulledArm_t=k}\gainVector_{k,t}$.
 %
 The best arm to be identified is  $\bestArmSto= \argmax_{k\in\setArms} \mean_k$.
 % involved
 % Indeed $\bestArmSto$ is not define for any $\gainVector$ but in expectation with the means. With very high probability (smaller than the probability of errors of interest) the best arms are the same anyway
 %
 Similar to the adversarial case, the gaps are 
 $\gap_{k} \triangleq |\max_{i \neq k} \mean_{i} - \mean_{k}|$, ranked as
$\gap_{(1)} \triangleq \gap_{(2)}\leq\ldots\leq\gap_{(\nArms)}$, 
 and the error 
 %\todom{here the definition of the error included the randomness over rewards as well, right? so maybe not so similar\\
% Victor:
 %That is what i try to explain in the next sentence, that the diff is very small, it is not clear?}
 of the learner is $\ProErr_{\stoPro}(\timeHorizon) 
 \triangleq \Pro\left(\recomArm_n \neq \bestArmSto\right)$.  
 However, unlike in the adversarial case, this definition of the 
 error includes the randomness of rewards.
 For this comparison, if several arms tie for the largest cumulative gain,
 let $\bestArm_{\gainVector}$ be the one with the smallest index.
 By Hoeffding's inequality and a union bound, a table
 $\gainVector\sim\stoPro$ satisfies $\bestArmSto\neq\bestArm_\gainVector$
 with probability at most $2\nArms\exp(-\timeHorizon\gap^2_{(1)}/2)$.
 The two error indicators can differ only on this event, so the absolute
 difference between their probabilities, both averaged over the
 stochastic rewards and learner randomness, is bounded by this
 disagreement probability.
 
\paragraph{Notions of complexity} Given 
gaps $\gap_{k}$ for $k\in\setArms$, we define two notions of 
complexity of the identification task, 
$\complexitySR$ and $\complexityUnif$, that were introduced by~\cite{Audibert10BA}. 
% In the stochastic case, these notions are defined 
In particular,
 \[
 \complexitySR
 \triangleq
 \max_{k\in\setArms}\frac{k}{\gap^2_{(k)}}
 \qquad\text{and}\qquad
 \complexityUnif 
\triangleq
% \nArms\max_{k\in\setArms}\frac{1}{\gap_{(k)}^2}=\frac{\nArms}{\min_{k\in\setArms}\gap_{(k)}^2}
% ~=~
% \nArms\max_{k\in\setArms}\frac{1}{\gap_{(k)}^2}
% =
 \frac{\nArms}{\gap_{(1)}^2}\cdot
 \]
  %\todov{which way to write tit? }
  %
  %\todomi{isn't this a good place to relate to the complexities in the prior papers?\\
 % Victor: These are the complexity of previous paper}
   $\complexitySR$ relates to the complexity of the stochastic 
   case. $\complexityUnif$ will be used both in the 
   adversarial and stochastic cases. 
In the adversarial case, the complexity and the gaps 
are defined with respect to $\gainVector$, 
but we also sometimes  write the uniform complexity as  $ \complexityUnif (\gainVector)$ for clarity.

\paragraph{Class of problems}  We define a set of classes to group problems with
very similar gap structure and with complexities that are only a 
constant multiplicative factor apart.
For any $0<\gap_1=\gap_2\leq\gap_3\leq \ldots\leq 
\gap_\nArms\leq 1/8$, we define a \textit{problem class} 
$\Gaps_{c}$ with
$c\geq 1$. 
Given these gaps and $c$, in the adversarial case, we say that 
$\gainVector\in\Gaps_{c}$  if  for all $k\in\setArms$, $\gap_k/c
 \leq\gap^{\gainVector}_k\leq c\gap_k$ except for only one 
 arm~$\refarm$  whose gap is related to the smallest gap as $\gap_{1}/c\leq\gap^{\gainVector}_{\refarm}\leq c\gap_{1}$.
In the stochastic case, $\stoPro\in\Gaps_{c}$ if the same inequalities,
including the exception for arm~$\refarm$, hold with the actual gaps
$|\max_{i \neq k} \mean_i-\mean_k|$ in place of $\gap^{\gainVector}_k$,
while keeping the reference gaps $\gap_k$ fixed.

The exceptional arm allows our lower-bound
constructions to change the identity of the
best arm while keeping the gaps of the
remaining arms comparable to the same
reference gap profile.

%Our results will expressed in a
%context where \advPro{} generates deterministically
%$\gainVector$ before the start of the game and therefore 
%in that case \advPro{} and $\gainVector$ are essentially 
%the same thing.
%However in some proofs we will construct adversaries \advPro{} that generate 
%randomly  $\gainVector$ before the game starts. In those cases we will be more 
%careful to discriminate  between $\gainVector$ and \advPro{} 
%in the notations.

%\section{Notation and technical results for the proofs}
%\todov{unused at the moment}
%Following the notations in\cite{Carpentier16TB}:
%
%For two distributions $\nu$, $\nu'$ defined on $\Reals$ and that are such that $\nu$ is absolutely continuous with
%respect to $\nu'$, we write
%\[
%KL(\nu,\nu')
%~=~
%\int_{\Reals}\log\left(\frac{d\nu(x)}{d\nu'(x)}\right)d\nu(x)
%\]
%for the Kullback Leibler divergence between distribution $\nu$ and $\nu'$.

%% file: upLowGenAd.tex
% !TEX root = Main.tex
\section{General adversarial case: An optimal 
	learner can play uniformly at random}
\label{s:ULGenAd}
In this section, we define a simple learner  (\RULE{}) and in Theorem~\ref{th:UPARU}, we 
provide an upper bound on its probability of error 
 depending on the gap in hindsight. 
We also give  a matching lower bound for the general 
adversarial best-arm identification (Theorem~\ref{th:LOWadPRO}) 
proving that \RULE{} obtains the optimal gap-dependent rates of error against worst-case adversaries.
%
%\begin{figure}
%    \centering
%    \framebox{
%        \begin{minipage}{.85\textwidth}
%            \textbf{Given:} $K$ \\
%            For $t=1,2, \ldots, \timeHorizon$\\
% \textcolor{white}{erere}Select arm $\pulledArm_t\in [\nArms]$ 
% with probability $\learnerDist_{k,t}=\Pro(\pulledArm_t=k)
% =\frac{1}{\nArms}$ for all $k\in\setArms$.
%
%            \textbf{Returns:} $\recomArm_\timeHorizon=\argmax_{k\in\setArms} \estCumulGainVector_{k,\timeHorizon}$ 
%        \end{minipage}%\todov{inclue pt in this def?}
%    }
%    \caption{The \emph{Random Uniform Learner (\RULE{})} }\label{fig:uniform}
%\end{figure}
%

%The \emph{Random uniform  learner} (\RULE{}) is detailed in 
%Figure~\ref{fig:uniform}. 
%\todov{if in need for space, we can remove the figure to explain a uniform sampling!}
%Select arm $\pulledArm_t\in [\nArms]$ 
% with probability $\learnerDist_{k,t}=\Pro(\pulledArm_t=k)
% =\frac{1}{\nArms}$ for all $k\in\setArms$
At round $t$, the \emph{random uniform  learner} (\RULE{}) 
selects an arm $\pulledArm_t\in [\nArms]$ uniformly at random, i.e., 
with probability $\learnerDist_{k,t}\triangleq\Pro(\pulledArm_t=k)
\triangleq 1/\nArms$ for all $k\in\setArms$.
Each round uses a fresh uniform draw, independent of past actions and observations.
%he choice of $\pulledArm_t$ is made uniformly  at 
%random from $\setArms$.
At the end of the game, the recommended 
arm $\recomArm_\timeHorizon$ is the one with highest 
estimated  
%\todom{estimator is unbiased but here we have an actual realization that is biased\\
%Victor: What do you mean biased? the realisation is what it is.} 
cumulative gain, 
$\recomArm_\timeHorizon \triangleq \argmax_{k\in\setArms}
\estCumulGainVector_{k,n}$. 
%We first prove an upper bound on the probability of error of ARU.
%
\begin{restatable}[\textcolor{titleTh}{Upper bound for \RULE{} in the adversarial case}]{theorem}{tata}\label{th:UPARU}
	For any horizon $\timeHorizon$, given rewards $\gainVector$ 
	chosen by an oblivious adversary, with $\gainVector_{k,t}\in [0,1]$ 
	for all $t\in[\timeHorizon]$ and for all $k\in\setArms$, \RULE{}  
	outputs an arm $\recomArm_\timeHorizon$
	with the guarantee that its probability of error 
	$\ProErr_{\advPro(\gainVector)}(\timeHorizon)$ verifies    
	%
	%    \todov{is that true against an adaptive adversary be precise}
	\[
	%\mathbb{P}\big[\recomArm_\timeHorizon \neq \bestArmAd\big] 
	\ProErr_{\advPro(\gainVector)}(\timeHorizon)
	\leq
	\nArms\exp\left(-\frac{3\timeHorizon}{28\complexityUnif(\gainVector)}\right)\!\cdot
	\]
	%    
	%Note that the factor $\nArms$ in front of the exponential is easily replaced by a $\log\nArms$ using a decomposition of the error into groups of halves as in \cite{Karnin13AO}. \todov{should we mention that? should actually get the log from the proof?}
\end{restatable}
%\todov{maybe define $\mean_{k,t}$,$\gap_{k,t}$, is this needed?}
%\todov{should we have $\gainVector$ here in the equation?}
%
%\todov{should be able to remove the factor 2 because it appears i think if you use the two-sided Bernstein inequality... and we might only need one side}

\noindent\textbf{Sketch of the proof} \textit{(full proof in 
	Appendix~\ref{app:proofupadPRO})}:
The deviation of
$\estCumulGainVector_{k,t}$ from $\cumulGain_{k,t}$ is bounded by
a Bernstein bound: for each fixed arm $k$, given the fixed table
$\gainVector$, the $\estGainVector_{k,t}$ are independent across rounds
because sampling uses fresh uniform draws. They are scaled
Bernoulli$(1/\nArms)$ indicators with variance at most $\nArms$.
Bernstein uses this variance bound of order $\nArms$, whereas Hoeffding
would use the squared range of order $\nArms^2$.
A union bound combines one lower tail for the best arm with one
upper tail for each competitor.
%
%\todomi{did we check that all these assumptions BELOW are compatible?\\
%Victor: Why would they?
%the gaps are arbitrary: no assumption. Then K is large. then you set n large enough so that the condition is valid \\
%M: is there is at least one combination of parameters when $K\ge 4095$ and ...
%}
\begin{restatable}[\textcolor{titleTh}{Lower bound for the adversarial problem}]{theorem}{toto}\label{th:LOWadPRO}
	Given any problem class $\Gaps_{3}$ with associated complexity
	$\complexityUnif$,    
	for any learner, for any horizon $\timeHorizon$ such that    
	$\nArms\exp\left( -\timeHorizon\gap^2_1 /128\right)\leq 1/128$ and $\nArms\geq 4096$, there exist 
	$\gainVector^1\in\Gaps_{3}$ and $\gainVector^2\in\Gaps_{3}$ so that  
	the  probabilities of error suffered by the learner on 
	$\gainVector^1$ and $\gainVector^2$, denoted $\ProErr_{\gainVector^1}(n)$ 
	and  $\ProErr_{\gainVector^2}(n)$ respectively, verify
	\[
	\max\left(\ProErr_{\gainVector^1}(n),\ProErr_{\gainVector^2}(n)\right)
	\geq
	\min\left(
	\frac{1}{128}\exp\left(-\frac{32\timeHorizon}{\complexityUnif}\right)\CommaBin
	\frac{1}{32}
	\right)\!\cdot
	\]     
\end{restatable}
\noindent\textbf{Sketch of the proof} \textit{(full proof given 
	in Appendix~\ref{app:prooflowadPRO})}:
The key idea is to identify an arm that the learner samples infrequently and
use it to construct two environments with different best arms but observations
that are difficult to distinguish. If the learner cannot tell which environment
it faces, it has little information on which to base a different recommendation
in each. Yet the correct recommendation differs between the two environments,
so the learner cannot reliably identify the best arm in both. An
information-theoretic testing inequality quantifies the error that it must
incur in at least one environment.

Consider a fixed learner run with horizon $\timeHorizon$, and a base
problem \basePro{} with Bernoulli means $\mean_1^{\basePro}=1/2$ and
$\mean_j^{\basePro}=1/2-\gap_j$ for $j\in\setArmsmo$.
During the first $m=\lfloor\timeHorizon/2\rfloor$ rounds, Dirichlet's
box principle gives a fixed suboptimal arm $\refarm$ whose expected
number of pulls is at most $m/(\nArms-1)\leq\timeHorizon/\nArms$.
We select this arm using expectations under \basePro{}, before drawing
the reward tables, and use it to construct two similar problems.

In the first phase of $m$ rounds, $\advPro_1$ follows \basePro{} on
every arm. The problem $\advPro_2$ differs only in arm $\refarm$,
whose Bernoulli mean is $1/2-\gap_{\refarm}+2\gap_1$ rather than
$1/2-\gap_{\refarm}$. In both problems, the second phase also has
$m$ rounds, from $m+1$ through $2m$: arm $\refarm$ receives the
deterministic reward $\gap_{\refarm}-\gap_1$ in each round, and
all other arms receive zero. If the horizon is odd, the remaining
final round gives zero reward to every arm. The complete table is
drawn before interaction, and the learner still acts for all
$\timeHorizon$ rounds and recommends an arm after the last round.

The common deterministic second phase makes arm $\refarm$ a close
competitor to arm $1$, irrespective of its original gap. In the two
problems, the expected cumulative gains satisfy
\begin{align*}
\Exp_1[\cumulGain_1]=\Exp_2[\cumulGain_1]&=m/2,\\
\Exp_1[\cumulGain_{\refarm}]&=m(1/2-\gap_1),\\
\Exp_2[\cumulGain_{\refarm}]&=m(1/2+\gap_1).
\end{align*}
Thus the selected arm lies just below arm $1$ in expected total
reward in the first problem and just above it in the second.

Only first-phase pulls of arm $\refarm$ distinguish the problems.
Although its reward laws differ during $m$ rounds, the learner
obtains at most $\timeHorizon/\nArms$ such observations in expectation
under $\advPro_1$, whose prefix agrees with \basePro{}.
Since the mean difference is $2\gap_1$, the relative entropy
$D(P_1\Vert P_2)$ of the observed-transcript laws under $\advPro_1$
and $\advPro_2$, respectively, is bounded by a constant times
$\timeHorizon\gap_1^2/\nArms=\timeHorizon/\complexityUnif$, where
$\complexityUnif=\nArms/\gap_1^2$ is the reference complexity.
The change-of-measure and testing bound in
Lemma~\ref{l:transcriptChangeMeas} gives
\[
P_1(\recomArm_{\timeHorizon}\neq1)
+P_2(\recomArm_{\timeHorizon}=1)
\geq\frac12\exp\left(-\frac{32\timeHorizon}{\complexityUnif}\right).
\]
There are different intended best arms, but too few informative
observations to distinguish the problems reliably.

Concentration of the complete reward tables preserves the intended
best arms, $1$ and $\refarm$, and membership in $\Gaps_3$, except on
sufficiently rare tables. Removing these bad tables from the testing
bound and averaging over the remaining finite supports yields two
deterministic witnesses in $\Gaps_3$, with the required error rate
on at least one of them. The appendix gives the complete argument.

\begin{remark}
The condition
$\nArms\exp(-\timeHorizon\gap_1^2/128)\leq1/128$ controls
fluctuations of the complete reward tables. Together with
$\nArms\geq4096$, it makes the bad-table probabilities small relative
to the exponentially small testing lower bound.
	The assumption on~$\nArms$ is
	likely to be an artifact of the proof.
The lower bound is a worst-case statement over the prescribed class,
at the exponential scale $\timeHorizon/\complexityUnif$.
The comparison with uniform sampling should be understood at this
scale; it does not rule out better performance on particular instances.
\end{remark}%

%% file: bestOfBoth.tex
% !TEX root = Main.tex
%
\section{The best of both worlds challenge}
\label{s:FormuBOB}
In this section, we ask if we can have a learner 
that performs optimally under adversarial and stochastic rewards. 
The lower bound in Theorem~\ref{th:lowBOB} shows that in general,
this is impossible.\vspace{.2cm}

%First we recall the stochastic problem and introduce its notations.

%therefore this not a very interesting problem in the sense that no smart thing can be done in the worst case. So why not having an algo that still possess this worst case guaranty like uniform but also has a nice behaviour on easier data such as, as a start, the stochastic case and perform there as good as SR.
%
\paragraph{Existing robust solutions?} In the stochastic setting,
a state-of-the-art algorithm, Sequential Halving (\SH{}, 
\citealp{Karnin13AO})---see also Successive Rejects (\SR{}) 
by~\cite{Audibert10BA}---guarantees 
$\ProErr_{\stoPro}(\timeHorizon) \leq \cO\left(\log\nArms\exp
\left(-\timeHorizon/(\complexitySR\log\nArms)\right)\right)$. 
However, as discussed in the introduction, \SR{} or \SH{} can fail  
against a worst-case adversary.
On the other hand, as discussed by~\cite{Audibert10BA},  
uniform-like algorithms (like \RULE{}) can only guarantee that
in the stochastic case, we get $\ProErr_{\stoPro}(\timeHorizon) \leq  
\tcO\left(\exp\left(-\timeHorizon/\complexityUnif\right)\right)$.
In general, $\complexitySR\leq\complexityUnif$ and 
in some problems, we even have $\complexitySR=(2/\nArms)\complexityUnif$. 
Therefore, \SH{} can notably outperform uniform 
algorithms in the stochastic case.%\vspace{.2cm}

%Un

%So the stochastic environment is already solved (reference the literature). State of the art algorithm are of the successive regret, SR, family. More over the uniform type algorithm are shown to be sub-optimal.
%As discussed somewhere (\todo{not done yet} not done yet! put it here or in the previous section? ) we have that SR is easily beaten in the adversarial setting.
%So we can notice that the algorithms that are optimal in one setting are suboptimal in the other. This raise a question:
%But what about solving it and 
%being as good as uniform against adversaries. can we have best of both worlds. 
%we think no and we first provide a lower bound.
%
\paragraph{The best of both worlds} This leads to the central question of the paper:
Can a learner achieve optimal error rates in both settings without knowing whether the rewards are stochastic or adversarial?
More precisely, the question is whether such a learner can simultaneously guarantee,
for every horizon $\timeHorizon$, every stochastic problem
\stoPro{}, and every adversarial reward sequence~$\gainVector$,
the following bounds on its respective probabilities of misidentification
$\ProErr_{\stoPro}(\timeHorizon) $ and
$\ProErr_{\advPro(\gainVector)}(\timeHorizon) $:
%for any iid problem and adversarial $\gainVector$,
%\todomi{in sto and adv, the expectation is taken (maybe) over different things. Is this not a problem when we compare them to each other?}
\[
\ProErr_{\stoPro}(\timeHorizon) 
\leq
\tcO\left(\exp\left(-\frac{\timeHorizon}{\complexitySR\log\nArms}\right)\right)
\text{ ~~~and~~~ }
\ProErr_{\advPro(\gainVector)}(\timeHorizon) 
\leq  
\tcO
\left(\exp\left(-\frac{\timeHorizon}{\complexityUnif(\gainVector)}\right)\right).
\]
%where $c_\nArms$ and $c'_\nArms$ are quantiti
%
\paragraph{Why is the BOB question challenging?}
The learner could choose, 
for arm $k$, at round $t$ after it has been sampled at least once, to use the 
cumulative gain  estimator  $\estBiasCumulGain_{k,t} = 
\frac{t\sum_{t'=1}^{t}\indicator{\pulledArm_{t'}=k}\gainVector_{k,t'}}{\sum_{t'=1}^{t}\indicator{\pulledArm_{t'}=k}} 
$.
This estimator can be potentially highly biased if it is used against a 
malignant adversary. For this reason, we base our approach on 
the estimator $\estCumulGainVector_{k,t} = \sum_{t'=1}^{t}
\frac{\gainVector_{k,t'}}{\learnerDist_{k,t'}}\indicator{\pulledArm_{t'} = k}$.
However, this usage potentially introduces high variance in our estimates;
the cumulative conditional variance of $\estCumulGainVector_{k,t}$
relevant to martingale concentration, namely the sum of the one-step
conditional variances, is bounded by $\sum_{t'=1}^{t} 1/ \learnerDist_{k,t'}$.
% scaling with $\sum_{t'=1}^{t}\frac{1}{\learnerDist_{k,t'}}$.
The high variance is most damaging in the stochastic case when trying 
to have a learner based on $\estCumulGainVector_{k,t}$ to obtain the 
optimal error rates of \cite{Karnin13AO}. Indeed, these 
optimal rates are obtained by algorithms using empirical means, as in
$\estBiasCumulGain_{k,t}$, for which strong concentration guarantees
are available in the stochastic setting.
Therefore,  
we strive to characterize the minimum amount of 
unavoidable variance of the mean estimators of each arm.
%As a consequence,  we are capturing the power of the adversary to fool the 
%learner.
%This is especially critical in the stochastic case where the learner is trying to obtain an error rate comparable to the one of the stochastic state-of-the-art algorithm that are usually  using $\estBiasCumulGain_{k,t}$ in a safe
%If the arms are pulled uniformly, the variance is of order $K$, 
%which is very large.
The learner would like to allocate more pulls at any round $t$ 
to the arms that are among the best arms, which means having large 
$\learnerDist_{k,t}$ for these arms. Indeed,  discriminating 
between them is the hardest part of the task and large 
$\learnerDist_{k,t}$ reduces the variance term $1/\learnerDist_{k,t}$. 
%However pulling more the best arm especially difficult for the 
%  Moreover, because of our use of $\estCumulGainVector_{k,t}$, the learner should try to pull more the best
%  arms  as early in the game as possible %\todom{not sure, this  statements depends on how $1/\learnerDist_{k,t}$ depends on $t$}
%   to reduce, for small $t$, the variance term $1/\learnerDist_{k,t}$.
%But pulling more the best arms is exactly the problem the 
%learner is already trying to solve.
However, it is natural to 
think that the learner is not able to guarantee that it pulls 
the best arms at the beginning of the game more than in a 
uniform fashion.
If the arms are pulled uniformly, each one-step conditional variance is
at most $\nArms$, which can be large.
The time spent accumulating variance before the best arms are well identified influences the concentration bound and hence the bound on the final probability of error. 
Intuitively, the lower bound in Theorem~\ref{th:lowBOB} 
constructs worst-case examples 
showing that any learner cannot pull the best arm more than 
a certain amount in some period of the game because it is difficult to identify the best arms. 
Therefore, 
this learner is susceptible to be tricked by an adversary. Our new learner
in Section~\ref{s:UpperBOB} tries to limit this effect by 
making early and almost costless bets on what are the 
best arms given the early rewards and starts to pull them more right away.
Note that a long initial uniform sampling phase can contribute a term
of order $\timeHorizon\nArms$ to the conditional-variance upper bound for
the importance-weighted cumulative estimator $\estCumulGainVector_{k,\timeHorizon}$.
This can limit the concentration guarantee obtained from that bound.

Another approach could be a learner that determines online 
if the observed rewards are stochastic. This was used by~\cite{Bubeck12BB} and~\cite{Auer16AA}. They 
detect if the difference between $\estBiasCumulGain_{k,t}$
and $\estCumulGainVector_{k,t}$
is way too large. However, their bound itself uses  terms 
depending
on $\nArms$ and the variance of each arm, 
$1/\learnerDist_{k,t}$, which leads to similar open questions 
as discussed just above. In this paper, inspired by the approach 
of~\cite{Seldin14OP},  we give a practical 
simple parameter-free and versatile algorithm.  Furthermore,  the algorithms
that are based on stochastic tests are usually cumbersome and complex,  as discussed by~\citealp{Seldin14OP}.

%% file: lowerBoundBOB.tex
% !TEX root = Main.tex
%
%\subsection{Lower Bound for the Stochastic \& Adversarial Rewards}
%\label{s:LowBOB}
\paragraph{Why is the best of both worlds unachievable?}
We define a new notion of complexity, $\complexityBoth$ as
\[
\complexityBoth 
\triangleq
\frac{1}{
\gap_{(1)}}
\max_{k\in\setArms}\frac{k}{\gap_{(k)}}\cdot
\]
The complexity $\complexityBoth$ captures the difficulty of
stochastic best-arm identification when the learner
must also remain robust to adversarial rewards.
As we detail  in Remark~\ref{rem:LOWbobCOMP}, 
$	\complexitySR\leq\complexityBoth\leq\complexityUnif.$
The gap between $\complexitySR$ and $\complexityBoth$ reflects the
additional difficulty imposed by adversarial robustness.
%Let $\mathcal S$ be the set of stochastic problem such that
%The upcoming Theorem~\ref{th:lowBOB} shows that
%for any $\nArms$ specified gaps, $0<\gap_1=\gap_2\leq\gap_3\leq \ldots\leq \gap_\nArms\leq 1/8$, 
%a stochastic problem can be built with a very similar gap landscape 
%and a complexity changed by at most a factor of $2$ such that on this problem
% it is not possible to have a probability of error 
%smaller than 
%$	\frac{1}{64}\exp\left(-512\frac{\timeHorizon}{\complexityBoth}\right)$
% %on any stochastic problem
%  without  having 
%  a probability of error  not converging to $0$ with large $\timeHorizon$
%   on some adversarial problem with
%  again a similar gap landscape.
%
\begin{restatable}[\textcolor{titleTh}{Lower bound for the BOB challenge}]{theorem}{thi}\label{th:lowBOB}
Fix $\nArms\geq 2$ and reference gaps
$0<\gap_1=\gap_2\leq\cdots\leq\gap_\nArms\leq 1/8$,
and let $\Gaps_4$ be the corresponding problem class.
For any learner and any positive integer horizon $\timeHorizon$
satisfying
\[
\nArms\exp\left(-\frac{\gap_1^2\timeHorizon}{128}\right)
\leq \frac{1}{32},
\]
there exists an i.i.d.\ stochastic problem $\stoPro\in\Gaps_4$
such that, if
\[
\ProErr_{\stoPro}(\timeHorizon)
\leq \frac{1}{64}\exp\left(
-\frac{2048\timeHorizon}{\complexityBoth^{\stoPro}}
\right),
\]
then there exists a deterministic gain table
$\gainVector\in[0,1]^{\nArms\times\timeHorizon}$
belonging to $\Gaps_4$ for which
\[
\ProErr_{\advPro(\gainVector)}(\timeHorizon)
\geq \frac{1}{16}.
\]
Here $\complexityBoth^{\stoPro}$ is the complexity computed
from the gaps of the constructed stochastic problem.
The constructed stochastic problem and gain table may depend
on the learner and on $\timeHorizon$.
      \end{restatable}    
\begin{remark}\label{rem:LOWbobCOMP}
	In general, we have
	\[
	\complexitySR\leq\complexityBoth\leq\complexityUnif.
	\]
	%
%We first compare the complexities in a general regime. Then we spell out three different regimes to intuitively explore when the inequalities in the previous equation are strict or not.
Below, we compare the three  complexities in 
three specific gap regimes in order to intuitively explore   whether
the inequalities in the previous equation are strict or not. 
Interestingly, while in two regimes  $\complexitySR=\complexityBoth$, in the third regime, called the `square-root gaps', we can obtain $\complexityBoth=\sqrt{\nArms/2}\complexitySR$.  
This equality shows that on some problems and for large values of $\nArms$,
% 	\todovout{actually really large value on K , like millions !!! :(, i am not sure we want to comment on that no?those numerical constants!}
 our lower bound on 
the complexity of the BOB problem is 
significantly larger than the complexity of the strictly stochastic case.
Uniformly scaling the gaps down to satisfy $\gap_{(\nArms)}\leq1/8$
leaves all complexity ratios unchanged.
The square-root separation also persists, up to universal constants,
for the stochastic instances constructed in the proof of
Theorem~\ref{th:lowBOB} (Appendix~\ref{app:prooflowBOB}).
For sufficiently large $\nArms$, the optimal purely stochastic rate,
if guaranteed uniformly over stochastic instances and horizons, would
satisfy the premise of Theorem~\ref{th:lowBOB} at all sufficiently large
horizons. Hence it cannot be combined with vanishing worst-case error
over the adversarial class. Thus, in general,
the best performance achievable in the two worlds cannot be obtained
simultaneously.
	\end{remark}
%\begin{description}
%	\item[General Regime:] Let the set $S$ of pairs $(k,\alpha)$, be $S=\{k\in\setArmsmo,\alpha\geq 0 : \gap^{\alpha}_{(1)}=\frac{\gap^{\alpha}_{(k)}}{k}\}$. This set is not empty as for each $k$ it contains the pair $\left(k,\alpha=\frac{\log k}{\log\left(\frac{\gap{(k)}}{\gap_{(1)}}\right)}\right)$
%	Let us denote $\kappa=\max_{k,\alpha\in S} k^{\frac{1}{\alpha}}$
	%
	\paragraph{\raisebox{.04cm}{\textcolor{bull}{$\blacktriangleright$}}~Flat regime}  We assume all the gaps are equal, 
$k\in\setArms,~ \gap_k=\gap_{1}$. Then, 
$\complexitySR=\complexityBoth=\complexityUnif$. Having $\complexitySR=\complexityBoth$  shows that our  
stochastic BOB lower bound for robust learners (Theorem~\ref{th:lowBOB}) using $\complexityBoth$ is of the same order as the one in the 
strict stochastic setting~\citep{Audibert10BA} using $\complexitySR$.   
In this stochastic regime, \RULE{} is optimal while being robust to an adversary.
\paragraph{\raisebox{.04cm}{\textcolor{bull}{$\blacktriangleright$}}~Super-linear gaps} Let $2\in\argmin_{k\in\setArms} (\gap_{(k)}/k).$ 
This holds if 
$\gap_{(k)}=k\gap_{(1)}$ for every rank $k$ with $3\leq k\leq\nArms$,
and $\gap_{(1)}\leq 1/(8\nArms)$. Then, 
$\complexitySR=\complexityBoth=(2/\nArms) \complexityUnif$. 
Again, our BOB lower bound is of the same order as in 
the strict stochastic setting. This seems to indicate that in this case, BOB is achievable. However, it is not achieved by the uniform 
learner that is clearly suboptimal. This motivates the design of a learner that adapts its sampling strategy. Intuitively,  the learner can  
identify bad arms quickly
and  start focusing early on the best arms without incurring 
high variance on its estimates for them.
%
%\todov{think more about the right 2s!}
\paragraph{\raisebox{.04cm}{\textcolor{bull}{$\blacktriangleright$}}~Square-root gaps}   We assume $2\!\in\!\argmin_{k\in\setArms} (\gap^2_{(k)}/k).$ 
%This means that for all $k\in\setArms$, $\gap_{(1)}\leq \frac{\gap{(k)}}{\sqrt{k}}$. 
\newline
Choose a rank $j\in\argmin_{k\in\setArms} (\gap_{(k)}/k)$. 
For some constant $c$, let $\gap_{(1)}=c\gap_{(j)}/j$. We have 
$c\leq \sqrt{2j}$ because  $\gap^2_{(1)}/2 \leq \gap^2_{(j)}/j$ 
as  $2\in\argmin_{k\in\setArms} (\gap^2_{(k)}/k)$.
Therefore,  \[\complexitySR=\frac{2}{\gap^2_{(2)}}=\frac{2}{\gap^2_{(1)}}= \frac{2}{\gap_{(1)}} \frac{j}{c\gap_{(j)}}=\frac{2\complexityBoth}{c}=\frac{2\complexityUnif}{\nArms}\cdot\]	
We can get $c=\sqrt{2j}=\sqrt{2\nArms}$.
This happens if 
$\sqrt{\nArms/2}\gap_{(1)}=\gap_{(\nArms)}$ and
$\sqrt{k/2}\gap_{(1)}\geq\gap_{(k)}$ for ranks $k$ with $3\leq k<\nArms$. 
Then,  we get  $\sqrt{\nArms/2}\complexitySR=\complexityBoth$.
$\complexityBoth$ is $\sqrt{\nArms/2}$ larger than the complexity of
the strictly stochastic setting.
Intuitively,  the learner needs to spend some time to 
identify the `square-root gaps' suboptimal arms before starting 
to focus on the best arms. This makes it suffer an additional amount
of variance on its estimates for the best arms.
%\vspace{.1cm}
%
%\textbullet~ If all the gaps are such that for all $i$ , $\gap_i/i$ are
% equal (or more generally if, which means
% $\gap_1$ is quite small comparatively to the other gaps, the gaps in the example are increasing rapidly : linearly), we have the lowerbound 
%  $\complexitySR=\complexityBoth=\frac{1}{\nArms}\complexityUnif	$
%   the normal lowerbound again
%
%\textbullet~ If all the gaps are such that for all $i$ , $\gap^2_i/i$ are
% equal (or, more generally, If $1=\argmin_k \gap^2_k/k$, which 
% means arm 1 is quite small comparatively but as small as in
%  the previous case!, the gaps are increasing in a square root fashion ), in this case we have  $K=\argmin_k \gap_k/k$ we have the lowerbound 
%  $\complexitySR=\frac{1}{\sqrt{\nArms}}\complexityBoth=\frac{1}{\nArms}\complexityUnif	$ 
%  which is a new lowerbound larger than in the traditional case so 
%  thats cool. Indeed note that this is a case where the definition of the 
%  classical complexity is $\exp(-N\gap^2_1)$
%  
%  \todovout{\textbullet~ Think about other examples, what can we say in general? We can add something bout general geometric setting where all $\frac{i}{\gap^\alpha_i}$ are the same.}

%
%\begin{remark}
%The assumption $\nArms\exp\left(-\frac{1}{32}\gap^2_{(1)}\timeHorizon\right)\leq 1/32$ is mild. 
%Essentially, at most, it  asks for $\timeHorizon$ to be large 
%enough so that the stochastic problem is solvable within the budget $\timeHorizon$.
%% $\timeHorizon$ is required to be large enough here.
%\end{remark}%
%

\vspace{.5cm}
\noindent\textbf{Sketch of the proof} 
\textit{(full proof in 
	Appendix~\ref{app:prooflowBOB})}:
The idea is again to construct two environments that are difficult for the
learner to distinguish but have different best arms. This time, one environment
is stochastic and the other is adversarial. If the learner identifies the best
arm with very high probability in the stochastic environment, it cannot
reliably recognize that a different recommendation is needed in the
adversarial environment. Consequently, sufficiently good stochastic
performance forces a substantial probability of error in the adversarial
environment.

Our proof of the lower bound uses some arguments of purely stochastic best-arm identification 
lower bounds of~\cite{Audibert10BA} and~\cite{Carpentier16TB}.
We were also inspired by the lower bound of~\cite{Auer16AA} 
for the BOB question for the cumulative regret. However, 
our specific construction is new.
%\todov{thats a long sketch! should we shortening it or remove it? or no}

Fix the reference gaps, a learner, and a horizon $\timeHorizon$
satisfying Theorem~\ref{th:lowBOB}.
We construct a stochastic and an adversarial problem.
Between the two problems, only one arm differs.
Let us define $\nArms$ base Bernoulli distributions with means
$\mean_1^{\basePro}=1/2$ and
$\mean_j^{\basePro}=1/2-\gap_j$ for $j\in\setArmsmo$.
Choose any maximizing index
$i\in\arg\max_{j\in\setArms}j/\gap_j$; since $\gap_1=\gap_2$, $i\geq2$.
The initial phase is the prefix of $\timeHorizon_i$ rounds, where
\[
\dividefac_i=\min\left\{1,\frac{2\gap_1}{\gap_i}\right\},
\qquad
\timeHorizon_i=\left\lceil\timeHorizon\dividefac_i\right\rceil
\leq\timeHorizon.
\]
If $\timeHorizon_i=\timeHorizon$, as when $i=2$, the second phase is empty.

By Dirichlet's box principle, for $i\geq3$ at least one arm, denoted $\refarm$, among
$\lceil i/2\rceil,\ldots,i$ has expected pull count at most
$2\timeHorizon_i/i$. The maximizing property of $i$ makes its gap
comparable to $\gap_i$. For $i=2$, take $\refarm=2$. In both cases,
\[
\Exp_{\basePro}[\pullsNumber_{\refarm}(\timeHorizon_i)]
\leq\frac{2\timeHorizon_i}{i},\qquad
\frac{\gap_i}{2}\leq\gap_{\refarm}\leq\gap_i,\qquad\refarm\neq1.
\]
This lightly sampled arm is chosen from expectations under \basePro{},
before sampling the constructed tables.

The stochastic problem \stoPro{} is made by only modifying the original
Bernoulli distribution of arm $\refarm$, setting its mean to
$1/2+\gap_1/4$. Its distributions are stationary, and $\refarm$ is
uniquely best. The adversarial problem \advPro{} samples $\gainVector$ randomly
before interaction, independently of the learner's randomness.
It follows \basePro{} through the first $\timeHorizon_i$ rounds and
\stoPro{} afterward; all other arms keep their distributions throughout.
To make the resulting lower bound as strong as possible, we seek a
short initial phase to limit the distinguishing information.
However, it must be long enough to keep arm $1$ uniquely best on a
high-probability good-table event and preserve the prescribed gaps
after reward fluctuations. The chosen fraction provides this margin
while remaining of order $\gap_1/\gap_i$, up to the cap at one.
The appendix verifies that \stoPro{} and every good \advPro{} table
belong to $\Gaps_4$.

Let $P_A,P_S$ be the distributions of the learner's observed actions,
rewards, and final recommendation under \advPro{} and \stoPro{}.
To compare these observations, we use the change-of-measure
method of \citet*{Kaufmann16CBI}, in the form proved in
Lemma~\ref{l:transcriptChangeMeas}.
The lemma expresses the relative entropy of the two observation
distributions as the expected number of informative pulls times
the information supplied by each pull:
\[
D(P_A\Vert P_S)
=\Exp_{\basePro}[\pullsNumber_{\refarm}(\timeHorizon_i)]
\operatorname{kl}\left(\frac12-\gap_{\refarm},\frac12+\frac{\gap_1}{4}\right).
\]
Here the \advPro{} and \basePro{} interaction laws agree through the
prefix. The appendix gives, for universal constants $C,C'$,
\[
D(P_A\Vert P_S)
\leq\frac{C\timeHorizon_i\gap_i^2}{i}
\leq\frac{C'\timeHorizon\gap_1\gap_i}{i}
=\frac{C'\timeHorizon}{\complexityBoth^{\basePro}}.
\]
We want the lower bound on the learner's error to be as large
as possible. When the learner meets the prescribed guarantee
on \stoPro{}, a smaller upper bound on $D(P_A\Vert P_S)$ forces
a larger lower bound on its error in \advPro{}: its observations
provide less evidence for distinguishing the two problems.
This explains our choice of $i$: maximizing $i/\gap_i$ minimizes
the expression $\timeHorizon\gap_1\gap_i/i$ in the bound above.
For this maximizing index,
$\complexityBoth^{\basePro}=i/(\gap_1\gap_i)$, which gives the
dependence on $\timeHorizon/\complexityBoth^{\basePro}$.
The $\gap_j$ are the reference gaps; after reordering the constructed
\stoPro{} gaps, the appendix shows
$\complexityBoth^{\stoPro}\leq16\complexityBoth^{\basePro}$.
Thus the information is also bounded by a universal constant times
$\timeHorizon/\complexityBoth^{\stoPro}$.

Recommending arm $1$ is an error in \stoPro{}.
If its error satisfies Theorem~\ref{th:lowBOB}'s premise, this
recommendation is very unlikely under $P_S$. The lemma's event
consequence prevents it from being sufficiently likely under $P_A$
for reliable identification of arm $1$, forcing a positive constant
adversarial error. Concentration controls the bad tables, and the
appendix extracts a deterministic $\gainVector\in\Gaps_4$ with error
at least $1/16$.
The constructed instances may depend on $\timeHorizon$; this suffices
against guarantees required uniformly over instances and budgets.
%

%% file: upperBoundBOB.tex
%
% !TEX root = Main.tex
\section{A simple robust parameter-free algorithm for stochastic \& adversarial rewards}
\label{s:UpperBOB}
In this section, we present a new learner 
and analyze its theoretical performance against any i.i.d.\, stochastic 
problem or any adversarially designed rewards.

We call the algorithm \ProbabilityONE{}, denote it by \Pone{}, %\todovout{is this english?, 'denote it'?},
and detail it in Figure~\ref{fig:probaoneAlgo}.
Intuitively,  \Pone{} pulls the estimated best arm with 
``probability'' one, the estimated second best arm with ``probability'' 
one half, the estimated third best arm with ``probability'' one 
third, and so on until pulling the estimated worst arm with ``probability'' one over $\nArms$.\textsuperscript{\ref{foot:break}} 
In order to get proper probabilities, we need to normalize  
them by the $\nArms$-th harmonic number %$\bar\log(\nArms)$ where %\todo{the algorithm could be defined so that the two best arm are pulled equally to make a bit more sense but I guess it doesn't really matter}
$\bar\log \nArms =\sum_{k=1}^\nArms (1/k)$, where $\bar\log \nArms\leq \log\nArms+1$ for all positive integers $\nArms$. \Pone{} uses a Zipf distribution over the estimated ranks
with an exponent of $1$~\citep{Powers98AE}.
% \todov{the relation between $\bar\log(\nArms)$ and $\log(\nArms)$ should be written somewhere}
% \todovout{michal, no? figure[H] for this algo? :) MV: like this?}
\begin{figure}[H]
	\centering
	\framebox{
		~~~~~~~~~    \begin{minipage}{.9\textwidth}
			%    Given: the number of arms $\nArms$ \\[.05cm]
			\vspace{.15cm}
			\textbf{For} $t=1,2, \ldots$\vspace{.1cm}
			
			$\qquad$\raisebox{.02cm}{\textcolor{bull}{$\blacktriangleright$}}~  Sort and rank the arms by decreasing 
			order of $\estCumulGainVector_{\cdot,t-1}$: Rank arm $k$ as $\tilde{\langle k \rangle}_t\in[K]$\footnote{Equalities between arms
				or comparisons with arms that have not been pulled 
				yet are broken arbitrarily.\label{foot:break}}.\vspace{.2cm}
			
			$\qquad$\raisebox{.02cm}{\textcolor{bull}{$\blacktriangleright$}}~  Select arm $\pulledArm_t\in [\nArms]$ with 
			$\learnerDist_{k,t}\triangleq\Pro\left(\pulledArm_t=k\mid\mathcal F_{t-1}\right) 
			\triangleq \displaystyle \frac{1}{~\tilde{\langle k \rangle}^{\vphantom{X}}_t~\bar\log\nArms~}$ 
			for all $k\in\setArms$. Here $\mathcal F_{t-1}$ is the information
			available before sampling, including the resolved ranking but excluding
			the fresh action draw and current rewards.\vspace{.2cm}
			
			\textbf{Recommend}, at any given round $t$,  $\recomArm_t \triangleq  \argmax_{k\in\setArms}~ 
			\estCumulGainVector_{k,t}$.
		\end{minipage}
	}
	\caption{The \ProbabilityONE{} (\Pone{}) algorithm}\label{fig:probaoneAlgo}
\end{figure}
\noindent
The estimate used in \Pone{} is 
$\estCumulGainVector_{k,t-1}$ for arm $k$ at round $t$. 
\Pone{} is heavily inspired  by Successive Rejects (\SR{}) of~\citet{Audibert10BA}, as both rank the
arms and attempt to allocate to arm~$k$ a number of pulls 
$\tcO\left( \timeHorizon/\langle k\rangle\right)$ according to its true rank $\langle k\rangle$. 
Our new learner is parameter-free and naturally anytime 
(agnostic of $\timeHorizon$).
% It only needs to know  $\nArms$. 
Like \SR{}, \Pone{} requires neither prior knowledge
nor estimation of the complexity parameters. 
Unlike \SR{}, \Pone{} does not divide the game into different 
sampling phases. \Pone{} follows the same deliberately simple rank-based
sampling rule at every round. %Therefore \Pone{} is a natural 
%anytime algorithm not needing to know the time horizon 
%$\timeHorizon$. 
%Additionally it makes the sampling sub-routine very 
%versatile and adaptive to other potential applications and settings. 

%\todom{if it is anytime, we remove $n$ from the algo\\ Victor: I removed it, but now  Is it confusing with the return $J_n$?}
As discussed in Section~\ref{s:FormuBOB}, in order to minimize the misidentification error in the stochastic case, it is crucial 
to limit 
the variance of the estimators for the best arms.
Therefore  \Pone{},  from
its very first pull, pulls more (i.e., with higher probability) the arms that are 
estimated to be among the best. 
First, this comes with almost no cost: Indeed, pulling 
the estimated best arm with probability 
$1/\bar\log\nArms$ does not prevent from pulling 
all the arms almost uniformly and more precisely with probability at 
least  $1/(\nArms\bar\log\nArms).$
Therefore, no suboptimal arm is actually left out in 
the early chase for the best arm, and the bounds $1/\learnerDist_{k,t}$
on the conditional variances are at most $\bar\log\nArms$ times
the corresponding bound under uniform sampling.
Second, it gives the learner the flexibility to adapt 
to different types of stochastic problems with 
different gap regimes.   If in a setup,  some arms 
are clearly suboptimal, it is helpful to pull the clearly best 
arm more and right from the beginning. This is a more flexible behavior than the one of algorithms that  are fixing the number of pulls for each arm during a fixed period  in advance. 

The key idea is that when \Pone{} ranks a promising arm
highly, it samples that arm more frequently,
reducing the variance of its importance-weighted
estimates. This, in turn, helps improve subsequent
rankings. We capture this feedback through virtual
comparison times that are introduced only for the
analysis and are not used by the algorithm.

Additionally, compared to a fixed-phase algorithm,  our
analysis is also more flexible:
We can analyze, for instance, the quality of the estimated ranking $\tilde{\langle\cdot\rangle}$ and therefore the adaptive sampling procedure of the arms 
at any round.
Actually, these rounds called \emph{comparison rounds}, 
can be chosen in a problem-dependent manner, in order to 
minimize the final probability of error. This is reflected in the complexity measure  present in the upper bound,
which is defined as the minimum over admissible choices
of comparison times.
\Pone{} remains agnostic to the optimal `virtual' phases,
which are introduced only for the analysis. We now define this 
new notion of complexity. First, we define  the proportion of rounds for 
comparison $\propTimeVec$  in a space~$\propTimeVecSpa.$
Let $\propTimeVecSpa\triangleq\left\{\propTimeVec
\in[0,1]^\nArms
: %\sum_{i=2}^\nArms \propTime_i \leq 1,
\timeHorizon\propTime_i\in\Integer ,  \forall i \in[\nArms],
1=\propTime_1=\propTime_2\geq\propTime_3\geq\ldots
\geq\propTime_{\nArms}>\propTime_{\nArms+1}=0
\right\}$. The complexity associated with the \Pone{} 
learner is $\complexityProOne$ and is defined 
first as
%\todov{explain what this is!}
%
%\todomi{the constants in the complexity term do not seem to match the one in the proof of Lemma 11}
\[
\complexityProOne(\propTimeVec)
\triangleq
\max_{k\in\setArms}
\frac{
	\sum_{i=\langle k \rangle}^{\nArms}
	(\propTime_{i}-\propTime_{i+1})
	i
	+
	\frac{1}{24}\nArms\propTime_{\langle k \rangle}\gap_{k}}
{\propTime^2_{\langle k \rangle}\gap^2_{k}}\bar\log\nArms
\quad\text{and then,}\quad
\complexityProOne
\triangleq
\min_{\propTimeVec\in\propTimeVecSpa}
\complexityProOne(\propTimeVec).
\]
%
%$\complexityProOne(\propTimeVec)$ has a similar shape to the variance complexity in~\cite{Gabillon11MB} (See a detailed discussion there). 
In  $\complexityProOne(\propTimeVec)$, for  arm $k$,  the term
$    \sum_{i=\langle k \rangle}^{\nArms}
(\propTime_{i}-\propTime_{i+1})
i$  corresponds to the sum of 
variances $\tcO(i)$, for $i\in[\langle k \rangle:\nArms]$,  during a proportion of time %\todom{this does not give a complete intuition of $\complexityProOne$}\todom{what is the variance of order $i$? may you mean by magnitude $i$? }
$\propTime_{i}-\propTime_{i+1}$ between the comparison 
rounds $i+1$ and~$i$. Indeed, 
%as discussed 
%later in the sketch of proof, 
on the good-ranking event described below, for $i<\nArms$ and any arm $k$ satisfying
$\mean_{(1)}-\mean_k<\gap_{(i+1)}/2$, we have
$\tilde{\langle k\rangle}_t\leq i$ during
$\timeHorizon\propTime_{i+1}<t\leq\timeHorizon\propTime_i$.
For $i=\nArms$, the same rank bound is automatic.
This leads, on these events and at the corresponding sampling rounds,
to $\learnerDist_{k,t}\geq1/(i\bar\log\nArms)$ and therefore, as
$\estGainVector_{k,t}=\gainVector_{k,t}\indicator{\pulledArm_t=k}/\learnerDist_{k,t}$,
to a conditional variance of $\estGainVector_{k,t}$ at most
$i\bar\log\nArms$.
In the denominator, $\propTime_{\langle k\rangle}$ specifies the
comparison time $\timeHorizon\propTime_{\langle k\rangle}$, the number
of elapsed rounds available before that comparison.
The additional term involving $\nArms/24$ comes from
the linear correction in the Bernstein--Freedman
concentration inequality.
\begin{restatable}[\textcolor{titleTh}{Upper bounds for \Pone{}}]{theorem}{lightres}    \label{th:UpBOB}
	For any stochastic  problem \stoPro{}
	with  complexity $\complexityProOne$
	and for any $\gainVector$ fixed by an oblivious adversary 
	with  complexity $\complexityUnif(\gainVector)$,
	%with $\gainVector_{k,t}\in [0,1]$ for all 
	%$t\in[\timeHorizon]$ and for all $k\in\setArms$,
	the  probabilities of error of~\Pone{}, denoted 
	$\ProErr_{\stoPro}(\timeHorizon)$ and $\ProErr_{\advPro(\gainVector)}(\timeHorizon)$
	in their respective environment, for any~$\timeHorizon$, simultaneously verify  	 
	\[
	\ProErr_{\stoPro}(\timeHorizon)
	\leq
	2K^3n\exp\left(-\frac{\timeHorizon}{128\complexityProOne}\right)\quad \text{and}
	\quad
	\ProErr_{\advPro(\gainVector)}(\timeHorizon)
	\leq
	\nArms\exp\left(-\frac{3\timeHorizon}
	{40\bar\log(\nArms)\complexityUnif(\gainVector)}\right)\!\cdot
	\]
\end{restatable}
%
%\todomi{on this page variance of the arms is used ambiguously or not clearly to describe the variance of the estimates }
\noindent\textbf{Sketch of the proof} 
\textit{(full proof in 
	Appendix~\ref{app:proofupBOB})}:
For the \textit{adversarial case}, it is enough that the learner pulls 
each arm with probability at least $1/(\nArms\bar\log\nArms)$ to obtain
the same complexity $\complexityUnif$ as \RULE,  up to a factor $\log\nArms$.
For the \textit{stochastic case},
we consider $\nArms-1$  arbitrary `virtual' 
phases that each  ends at round 
$\timeHorizon_i=\timeHorizon\propTime_i$. These comparison times are fixed
for the analysis and optimized only after deriving the final error bound.
Note that \Pone{} is oblivious to these values.
The phases are indexed in reverse order,
from $\nArms$ down to $2$.
Intuitively, $n_i$ is a round after which we expect the 
following event~$\eventbob_i$ to happen with  
high probability: 
%$\cO\left(\exp(-\timeHorizon/\complexityProOne)\right)$
For all integer $u$ with $\phaseTime_i<u\leq\timeHorizon+1$,
the ranking from estimates through $u-1$ satisfies
$\tilde{\langle k\rangle}_u\leq i-1<i$ whenever
$\mean_{(1)}-\mean_k<\gap_{(i)}/2$.
At $u=\timeHorizon+1$, this is the virtual final ranking consistent
with the recommendation.
The important consequence is that, on this event, any such arm $k$
is pulled with $\learnerDist_{k,t}\geq1/((i-1)\bar\log\nArms)$ at
sampling rounds $\phaseTime_i<t\leq\timeHorizon$, giving conditional
variance at most $(i-1)\bar\log\nArms$ in its estimates
$\estGainVector_{k,t}$.
Reducing these variances leads 
in turn to  better estimates in the 
rest of the game.
The proof works iteratively over the phases. We bound the probability
that $\eventbob_i$ fails together with the required earlier ranking
guarantees. On these good-ranking events, the conditional variance is
controlled, so martingale concentration bounds their joint occurrence
with a later comparison error. Summing these joint error probabilities
gives a bound on the probability of an incorrect final recommendation,
based on $\timeHorizon_2=\timeHorizon$ rounds of observations.

To bound $\ProErr_{\stoPro}(\timeHorizon)$,
we use the Bernstein inequality for martingale  differences. This inequality takes into account the 
variances and holds despite the dependencies of the random 
variables $\estGainVector_{k,t}$.

When minimizing $\complexityProOne$ by choosing $\propTime_k$, for $k\in\setArms$, we need to trade off between
short phases, possibly meaning not enough samples 
to discriminate the suboptimal arms (the denominator term of the suboptimal arms is small) and long phases,
which means that in early stages, the best arms are considered as badly 
ranked for a long time and the variance of their mean estimators is increasing 
with the length of the early phases (the numerator term for the best arms is larger).
%\todov{this previous paragraph should be written better}
%
\begin{restatable}{corollary}{cocores}\label{coco} %\todom{we have HBOB and HSTOBOB, aren't they the same? \\ Victor : Yes! Hbob is for stochastic setting \\ BUT WHY in THM 4 there is additionl superscript STO }
	For $\timeHorizon\geq\gap_{(\nArms)}/\gap_{(1)}$, the complexity
	$\complexityProOne$ of \Pone{} matches the complexity $\complexityBoth$
	from the lower bound of Theorem~\ref{th:lowBOB}
	up to log factors; more precisely,
	\[
	\complexityProOne\leq2\complexityBoth(\bar\log\nArms)^2
	=\cO\left(\complexityBoth\log^2\nArms\right)\!.
	\]
	For every positive integer horizon $\timeHorizon$, the stochastic
	probability of error of \Pone{} satisfies
	\[
	\ProErr_{\stoPro}(\timeHorizon)
	\leq2\nArms^3\timeHorizon\exp\left(
	-\frac{\timeHorizon}{256\complexityBoth(\bar\log\nArms)^2}
	\right)\!.
	\]
\end{restatable}
%\todom{this setting of a minimizes P1}
\noindent The result of Corollary~\ref{coco} is obtained by setting
$\propTime_k=\left\lceil\timeHorizon\gap_{(1)}/\gap_{(k)}\right\rceil/\timeHorizon,$
$\forall k\in\setArms$.
We round each ideal comparison time upward to the next round.
This changes only the analysis, leaving the learner and its horizon
unchanged, and costs at most a factor two in the complexity in the
stated budget range. For smaller budgets, the displayed probability
bound is trivial. Notice that the same gap ratio also guides the lower-bound
construction of Theorem~\ref{th:lowBOB}. For its reference gaps,
we use $\dividefac_i=\min\{1,2\gap_1/\gap_i\}$ to leave enough
margin in the adversarial gaps after allowing for random
fluctuations. The full proof  of Corollary~\ref{coco}
is in Appendix~\ref{s:moreOnRem}, 
where we also discuss the relation between 
$\complexityProOne$ and $\complexitySR$ and 
$\complexityBoth$ for different regimes of the 
gaps.
Corollary~\ref{coco} demonstrates that \Pone{} achieves the best that can be wished 
for in the two worlds, up to log factors.
%
%
%\todov{Michal, do you agree with the following remark 8?}
% \todomi{Let's discuss the detail.
% First, what happens to the STO part? 
%Second, why in the ADV 1/2 of P1 and 1/2 of RULE, gives you this?}
\begin{remark}
	In the adversarial case, a  modification to $\Pone{}$ could lead to a similar upper bound as for \RULE{},  where  $\complexityUnif(\gainVector)$ appears instead of $\complexityUnif(\gainVector)\bar\log \nArms$.
	At each round, with probability one half we can play according to \RULE{} and otherwise use \Pone{}. The importance-weighted estimator uses the actual sampling probability $\learnerDist_{k,t}$ of this mixture. We keep the  recommendation $\recomArm_\timeHorizon=\argmax_{k\in\setArms}~ 
	\estCumulGainVector_{k,\timeHorizon}$.
	Such a mixture could retain useful stochastic adaptation.
\end{remark}
%
%\todov{we will probably have to remove the following remark}
\begin{remark}
	We studied the hard (adversarial) and easy data (stochastic) settings. 
	However, as discussed by~\cite{Seldin14OP}~and~\cite{Allesiardo17NS}, we can consider intermediate settings of difficulty.
	First, one could investigate whether the result in the stochastic case holds up
	to constants when the gaps of the arms do not change by more than the 
	same numerical factor during the game.
	More generally, we could design variants of $\complexityProOne$ under 
	the assumption that after some round $t'$ the (ground-truth) ranks of all the 
	arms are upper bounded each by a constant. %\todom{I changed to upper bounded each by a constant. OK victor ?}
	A variant of \Pone{} might adapt after~$t'$ and rank every arm at most according to its 
	maximum rank. Such results might even hold in 
	a case of a change of the identity of the best arm in the game.
	%
	%Then a careful reader will se that one can easily  redefine the notion of gaps and have a notion of maximal ranking to extend our results. 
	%We look at a version of the problem where the best arm is define as in the stochastic problem, just not necessarily stochastic.
	%So we have a distribution at each time and arm $\nu_{k,t}$ and a best arm at each time step $\bestArm_{t}=\argmax_{k\in\setArms}\Exp\left[\sum_{t'=1}^t \mean_{k}(t)\right]$
	%\todov{is this well defined?}
	%If there exist a time $t_s$, maximal rankings $r_{k,s}$ and minimal gaps such that for all times $t$, $t\geq t_s=\propTime_s\timeHorizon$, $\propTime_s\in[0,1]$ we have that for all $k\in\setArms$ both $\langle k \rangle_t \leq r_{k,s}$ and
	%$\langle k \rangle_t \leq r_{k,s}$  
	%
\end{remark}%

%% file: Extensions.tex
% !TEX root = Main.tex
\section{Possible extensions of \Pone{}}\label{s:exte}
%
%As \Pone{} and its sampling routine are simple and 
%parameter-free, they have the potential to be adapted 
%to variants of the fixed-budget setting of our paper. 
%it could be reused in numerous other variations of the setting of the best arm identification problem where additionally one wants to be robust to potential adversarial chosen rewards. 
%In this section, we discuss some of them.

%\subsection{Fixed Confidence setting}\label{s:exteFC}
%
\paragraph{Fixed-confidence}
%\todov{Well i am not sure anymore about this one!! need to think more! the stoping criterion is of importance?}
%Our results naturally extend to the fixed-confidence 
%setting .
{\emergencystretch=2em
In the fixed-confidence setting for i.i.d.\,stochastic rewards (\citealp{Maron93HR}, \citealp{Even-Dar06AE,Mnih08EB,
	Kalyanakrishnan12PA,Kaufmann13IC,Garivier16OB}), the goal is to 
design a learner that stops as soon as possible 
and returns the best arm with a fixed confidence.
Let~$\widetilde{n}$ be the round when the algorithm stops and 
$\recomArm_{\widetilde{n}}$ its returned arm. Given a 
confidence level $\delta$, the learner has to guarantee 
that $\Pro\left(\recomArm_{\widetilde{n}}\neq 
\bestArmSto\right) \leq \delta$. The performance of 
the learner is then measured by its \textit{sample 
	complexity}, which is the number of rounds $\widetilde{n}$ before stopping, 
either in expectation or in high probability.\par}

Mimicking the Hoeffding and Bernstein Races~\citep{Maron93HR,
	Mnih08EB}, we could design Freedman Races based on \Pone{}. %\todom{FReed\\ Victor, why not ? thats just the name of the algorithm} 
All the arms would be pulled according to 
the parameter-free sampling routine of \Pone{}. No arm could ever be discarded because we could happen to face an adversary.
%\todom{discuss with Victor and rewrite}
The learner would stop using the \Pone{} routine based on having the confidence intervals from the Bernstein concentration 
inequality for martingales and in particular, it would stop when the confidence interval for the empirically best 
arm is separated from the confidence intervals for all the other arms.
For this fixed-confidence variant of  \Pone{}, we could reuse the proof 
techniques developed for the fixed-budget setting and bound the 
accumulation of variance of the estimates. Then, in the stochastic case, we conjecture 
that the expected 
sample complexity of such  algorithm is $\tcO\left(\complexityBoth\log(1/\delta)\right)$, up to log factors. %\todom{either remove tilde or the log }

For the adversarial case, we can consider
an infinite sequence of rewards $\gainVector$ fixed by 
the adversary for all arms.
%The learner try to stop as earny
%We could probi
%measure of  sample complexity 
% (as small as possible) on the rewards up to round $t$, 
% 
Assume that the sample complexity  on the rewards up to round $t$, 
is bounded by some $\n(\gainVector_{[t]})$, the smaller the better. 
%,  such that he can
A natural target would be that  %\todom{not too clear}
if at any round $t$, 
$\gainVector_{[t]}$  verifies 
$\n(\gainVector_{[t]}) \leq t$, 
then with probability $1-\delta$, the learner can both 
stop and recommend the best arm $\bestArm_{\gainVector_{[t]}}$ 
at round $t$.
We conjecture that our Freedman algorithm could satisfy this 
requirement with the complexity of uniform allocation,
$\tcO\left(\complexityUnif(\gainVector_{[t]})\log(1/\delta)\right)\!.$ %$\todom{either remove tilde or the log }. 
Note that the learner could possibly never stop. 
%It can continuously scan all the rewards and, in case of a change of the best arm, find all the rounds where different best arms can be identified with probability larger than $1-\delta$.
%
%this measure of complexity at a time $t$ then the learner is able to by itself decide to recommend an arm and do so with a probability of error less than the targeted error.
%
% 
%\todov{if we have time we should write it but maybe we can just claim its easy because it is}
%
%\subsection{Stream of data}
%\noindent\textbf{Stream of Rewards, Best Arm in a Window. }
%
%
%\subsection{$m$-set, Threshold Bandit and Active Anomaly Detection}
\paragraph{Streams, windows, thresholds, $m$-set, and active 
	anomaly detection}
\Pone{} could be adapted to
recommend  the  best arm in 
the latest window of $W$ rounds
ending at $t\geq W$. One could investigate bounds with $W$ replacing $\timeHorizon$ when a variant of \Pone{} recommends the estimated best arm in that window. \Pone{} could also be adapted to identify $m$ best arms out of $\nArms$ as did~\cite{Bubeck12MI}.
A key would be to redefine the gaps around the boundary of the set of $m$ best arms.
We could also extend \Pone{} to a setting where the rewards converge~\citep{Li16HA}.
\cite{Locatelli16AO} defined the gaps with respect 
to a given fixed threshold and the objective is to determine which arms have 
a mean higher than the threshold. Again, our approach could apply. Moreover, it could prove to be a good 
robust approach to the problems that are linked to the 
threshold bandit problem as discussed by~\citet[Section~3]{Locatelli16AO}, one of them being the \emph{active anomaly detection} \citep{carpentier2014extreme}. %\todovout{would p1 be applicable to the extreme bandit setting? MV: yes possibly, but here the link is "active anomay detection"}.
Indeed, in adversarial anomaly detection, the learner might be monitoring 
different streams of non-stochastic rewards and could potentially detect an 
anomaly if one of the streams outputs a reward signal that is on average larger
than a given threshold during a time window period $W.$ A variant of \Pone{} could then be a robust anomaly detector and could be also adaptive to easy data.
\paragraph{Adaptive adversaries} Our upper bounds might extend 
naturally  to adaptive adversaries given the following condition:  
$\complexityUnif$ is an upper bound on the complexity of 
all $\gainVector$ that the adaptive adversary can possibly generate. A possible route would be to replace the Bernstein concentration inequality in Theorem~\ref{th:UPARU} by the Bernstein concentration inequality for martingales.
\vfil

%% file: appendices.tex
% !TEX root = Main.tex

%
\input{appExp}
\noindent In the remainder of the appendix, for  a random variable $X$, we denote its variance  by $\var_X$.
Moreover, we also write that a bounded random variable
$X\in [\lowBrv_X,\upBrv_X]$  has a range $\range_X=\upBrv_X-\lowBrv_X.$ 
%\todom{this is kinda outta place now, depending on the spacing move before A
%otherwise start with the introductory " "}
%\todom{change to $\ell$}
\input{appUpperUniform}
%
\input{appTechLem}
%
\input{appLOWAd}
%
\input{appLOWBOB}
%
\input{appUPPBOB}%

%% file: appExp.tex
%
% !TEX root = Main.tex 

\section{Experiments in the stochastic setting} 
\label{app:Exp}
We complement our theoretical results with experiments in the stochastic setting.
%We propose a few simple experiments to illustrate our theoretical analysis. 
We  reuse the experimental setups of~\cite{Audibert10BA} in Experiments~\ref{fp:A} to~\ref{fp:G}. We only consider Bernoulli distributions and the optimal arm always has mean
$1/2$. Each experiment corresponds to a different situation for the gaps. They are either clustered in a few groups
or distributed according to an arithmetic or geometric progression. Experiment~\ref{fp:H} reuses the `square-root gap' scenario when $\complexityBoth = \sqrt{\nArms/2}\complexitySR$ as detailed in Remark~\ref{rem:LOWbobCOMP}. The experimental setups are given below. %\todom{no reference from the main paper}
%\todom{adversarial experiments}
%\todom{did Peter Auer give any experiments}
\begin{itemize}[leftmargin=*]
	\item[] 
	\begin{fp}{fp:A}{}
\textit{One group of bad arms},
	$K= 20$, $\mu_{2:20}= 0.4 \equiv \forall j\in \{2,\ldots,20\}, \mu_j= 0.4$
\end{fp}
	%Experiment 1: 
		\item[]  \begin{fp}{fp:B}{} \textit{Two groups of bad arms},
	$K	= 20$, $\mu_{2:6} =  0.42$, $\mu_{7:20} =  0.38$.\end{fp}
		\item[] \begin{fp}{fp:C}{} \textit{Geometric progression},
	$K	= 4$, $\mu_{i} = 0.5 -(0.37)^i$, $i\in \{2,	3,4\}$ \end{fp}
		\item[] 	\begin{fp}{fp:D}{}
	\textit{6
	arms divided into three groups},	
	$K	= 6$, $\mu_2 = 0.42$, 	$\mu_{3:4}= 0.4$, $\mu_{5:6}= 0.35$\end{fp}
		\item[] 	\begin{fp}{fp:E}{} \textit{Arithmetic progression},
	$K 	= 15$, $\mu_i= 0.5 -0.025i$, $i\in \{	2,\ldots,15\}$ \end{fp}
		\item[] 	\begin{fp}{fp:F}{} \textit{2 good arms and a large group of bad arms}, $K= 20$, $\mu_2 = 0.48$, $\mu_{3:20}= 0.37$\end{fp}
		\item[]  	\begin{fp}{fp:G}{} \textit{Three groups of bad arms},
	$K = 30$, $\mu_{2:6}	= 0.45$, $\mu_{7:20} = 0.43$,
	$\mu_{21:30}= 0.38$	 \end{fp}
			\item[]  \begin{fp}{fp:H}{} \textit{Square-root gaps}
			$K	= 100$, $\mu_{i} = 0.5 - 0.25\sqrt{i/(2K)}$, $i\in [2:100]$ \end{fp}
\end{itemize}
In Table~\ref{tabloo}, we report the complexities $\complexitySR$, $\complexityBoth$, and $\complexityUnif$ computed in these experimental setups. Unsurprisingly, in Experiments~\ref{fp:A}, \ref{fp:C}, and~\ref{fp:E} we recover $\complexitySR=\complexityBoth$ and in Experiment~\ref{fp:H}, we have $\complexityBoth = \sqrt{\nArms/2}\complexitySR$. Experiments~\ref{fp:B}, \ref{fp:D}, \ref{fp:F}, and~\ref{fp:G} then give an idea about the behavior of  $\complexitySR$, $\complexityBoth$, and  $\complexityUnif$ with respect to each other. %  \todom{why we call it Canonical problems and the paper?} 
%\todom{lets have the background colors of the 3 columns  same as in graphs}
\begin{table}[H]
		\center
\begin{tabular}{|>{\columncolor[gray]{.97}}l>{\columncolor{pearDark!30}}r
		>{\columncolor{pearYe!30}}r
		>{\columncolor{pearThree!30}}r|}
	\hline 
	\cellcolor{gray!40}\textbf{Experimental setup} & \cellcolor{pearDark!70}$\boldsymbol{\textcolor{white}{\complexitySR}}$ & \cellcolor{pearYe!70}$\boldsymbol{\textcolor{white}{\complexityBoth}}$ & \cellcolor{pearThree!70}$\boldsymbol{\textcolor{white}{\complexityUnif}}$\\
	\hline
	\ref{fp:A}. One group of bad arms& 2000 & 2000 & 2000 \\
	\ref{fp:B}. Two groups of bad arms &1389 & 2083 & 3125\\
	\ref{fp:C}.	Geometric progression& 5540 & 5540 & 11080\\
	\ref{fp:D}. 6 	arms divided into three groups &400 & 500 & 938\\
	\ref{fp:E}.	Arithmetic progression& 3200 & 3200 & 24000\\
	\ref{fp:F}. 2 good arms and a large group of bad arms &5000 & 7692 &50000\\
	\ref{fp:G}.	Three groups of bad arms &4082 & 5714 & 12000\\
	\ref{fp:H}.	Square-root gaps &3200 & 22627 & 160000\\
	\hline
\end{tabular}
\caption{Comparing complexities  $\complexitySR$, $\complexityBoth$, and  $\complexityUnif$.}
\label{tabloo}\vspace{-.4cm}
\end{table}
\noindent
Figure~\ref{exp} reports the empirical error rates of~\SR, \Pone, and \textit{static uniform allocation} across the eight experimental settings. The static uniform allocation is not the algorithm \RULE. \RULE{} samples an arm uniformly at random while the \textit{static} uniform allocation
simply allocates $n/K$ pulls to each arm deterministically.
The empirical results follow very closely our theoretical findings as the empirical behavior in Figure~\ref{exp} mimics the behavior of the complexities in Table~\ref{tabloo}. Following~\cite{Audibert10BA}, we chose the horizon~$\timeHorizon$ to be of the order of the complexity $\complexity_1=\sum_{k\in\setArms}(1/\gap^2_k),$ where $\complexitySR\leq\complexity_1\leq\complexitySR\bar\log\nArms.$

%Note that, inspired by \Pone, one can design a sibling \textit{static} algorithm where the arm pulled at round $t$ would be $\pulledArm_t=\arg\min\limits_{k\in\setArms}  \tilde{\langle k \rangle}^{\vphantom{X}}_t \pullsNumber_{k}(t)$.
%This is a simple anytime and parameter free algorithm designed specially for the stochastic case. This algorithm is not robust to adversarial rewards. Its interest, in the stochastic case, is that it does not follow any phase scheduled or fixed in advance and will more naturally adapt to the data. It is expected to have the same expected behavior as \Pone{} but without the extra variance added by sampling arm at random. Empirically this sibling algorithm performs very well in the 8 experiments of this section, outperforming systematically \SR, \Pone, and the \textit{static uniform allocation}. Note that this static-\Pone would have the same theoretical guaranties as \SR{} except that there would be an extra $\log(K)$ deterioration: $e_{static-\Pone(n)}=\cO \left(e^{-\frac{n}{\complexitySR\log^2{K}}}\right)$. 

\begin{figure}[H]
	\center
	\includegraphics[width =.46\textwidth]{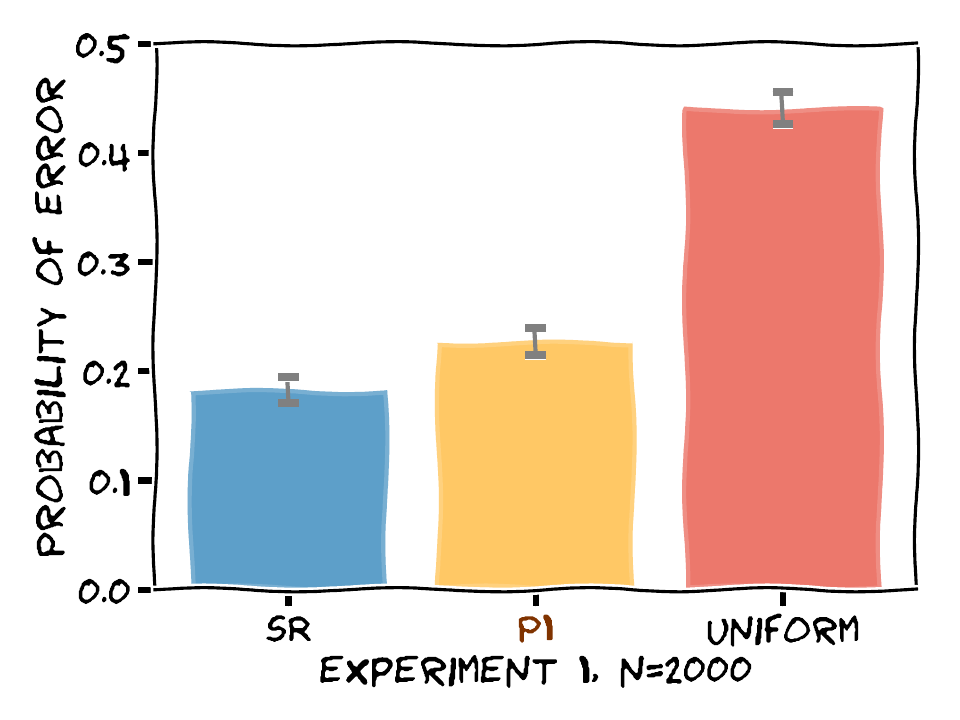}$\qquad$\includegraphics[width =.46\textwidth]{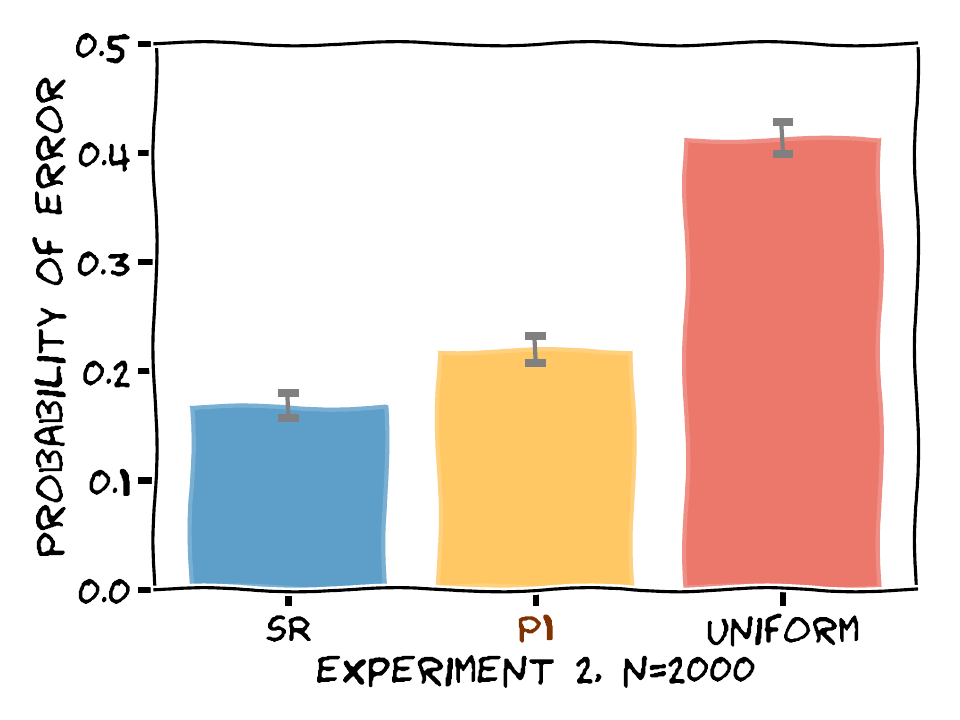}
		\includegraphics[width =.46\textwidth]{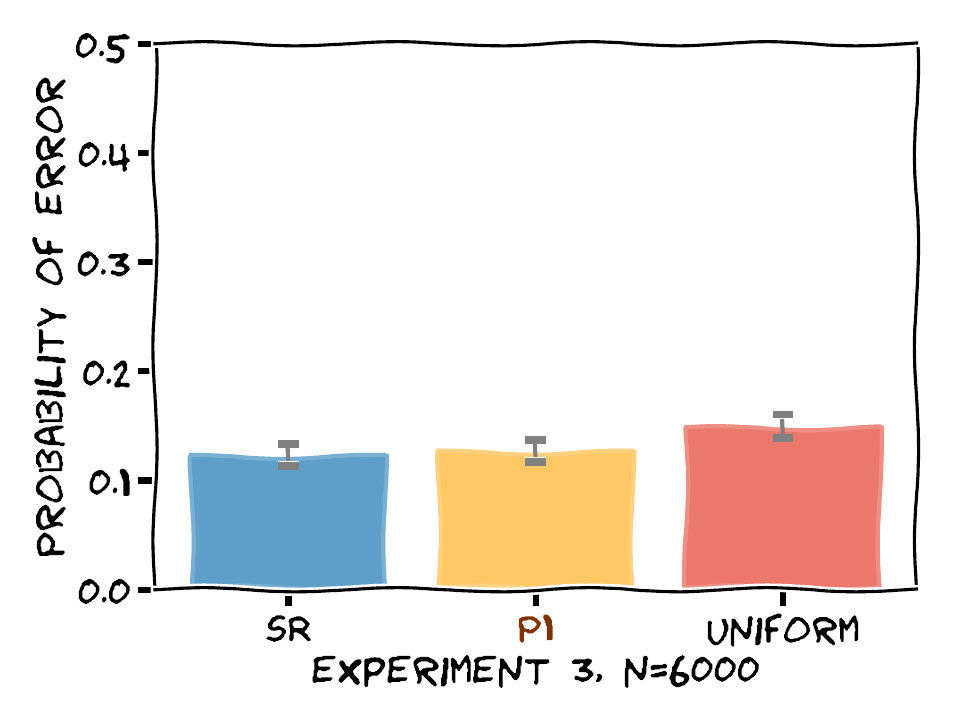}$\qquad$\includegraphics[width =.46\textwidth]{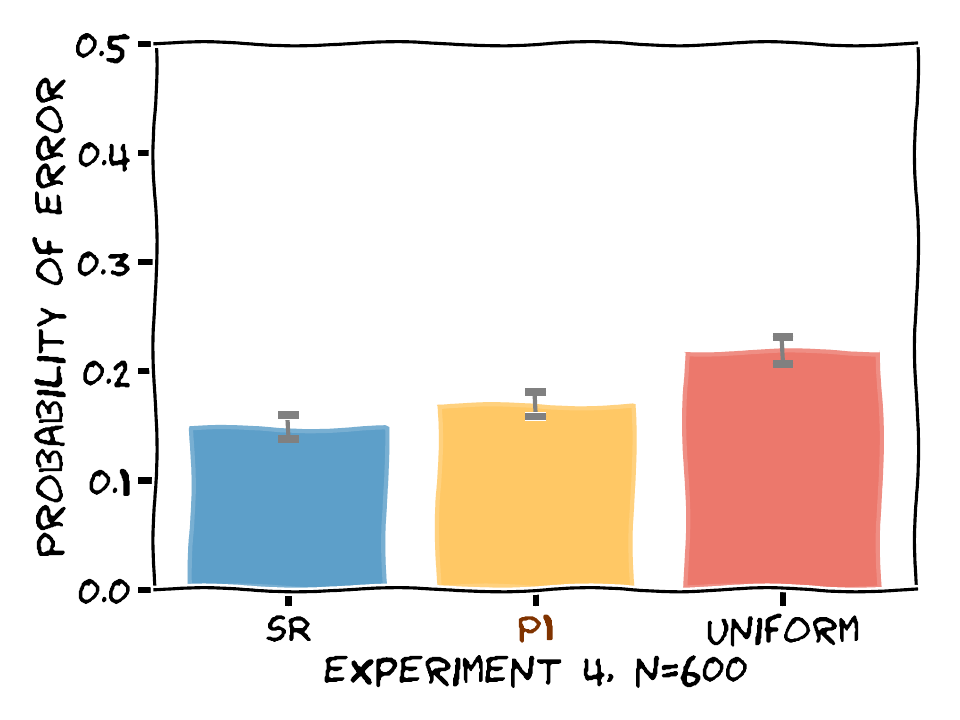}
			\includegraphics[width =.46\textwidth]{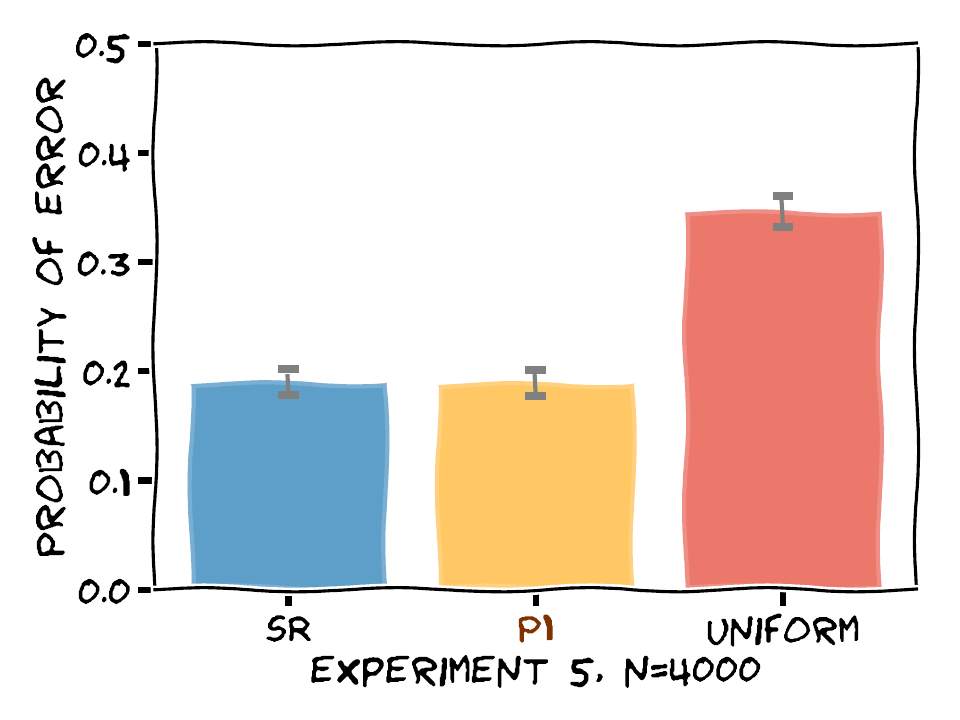}$\qquad$\includegraphics[width =.46\textwidth]{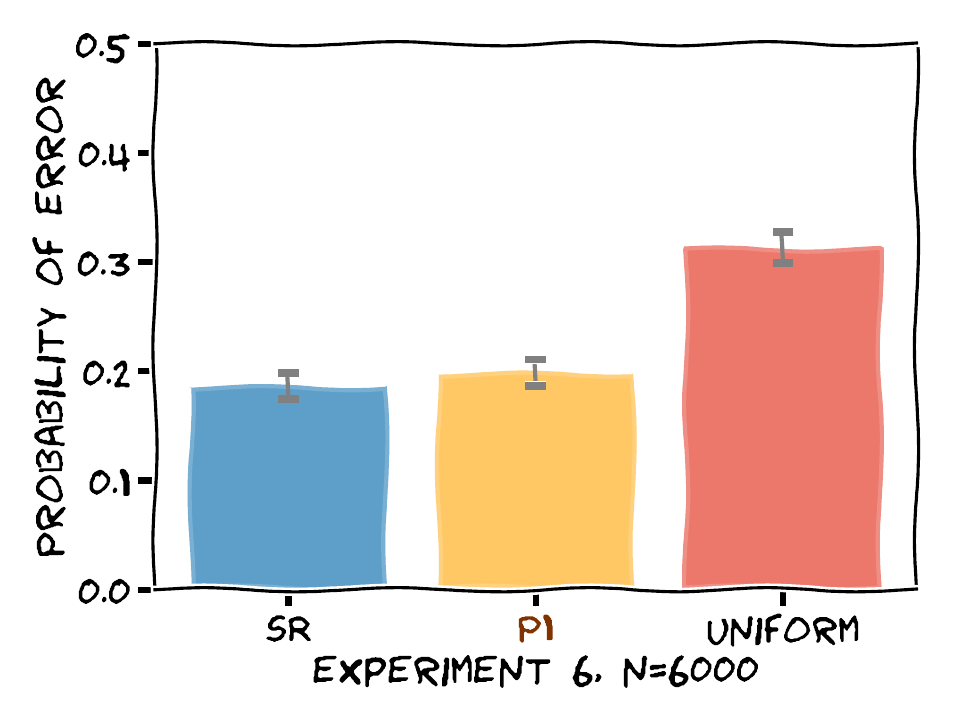}
	\includegraphics[width =.46\textwidth]{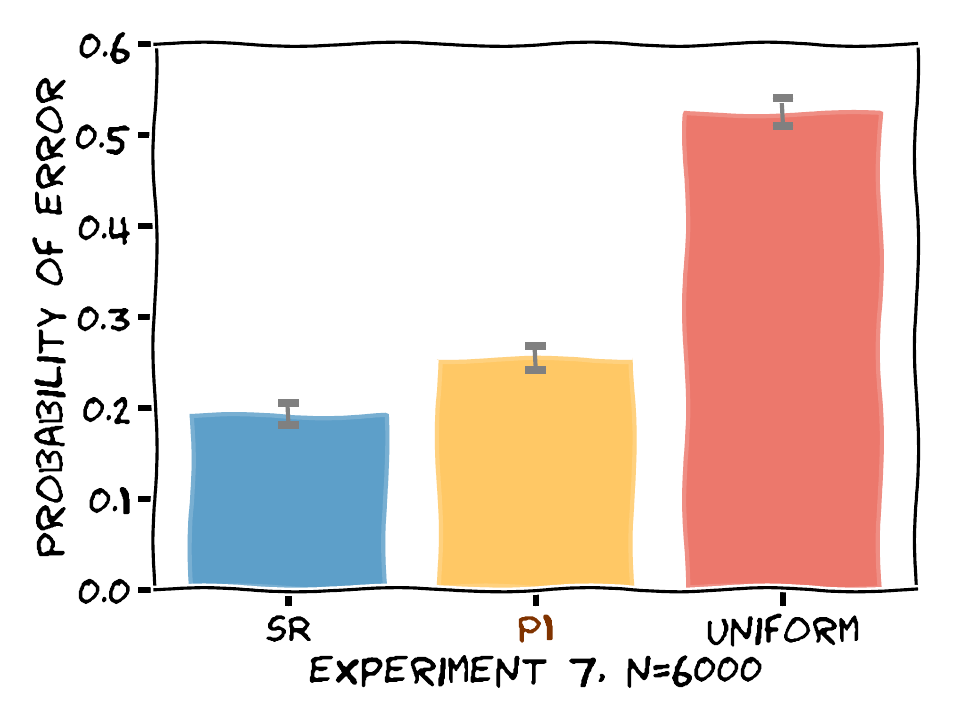}$\qquad$\includegraphics[width =.46\textwidth]{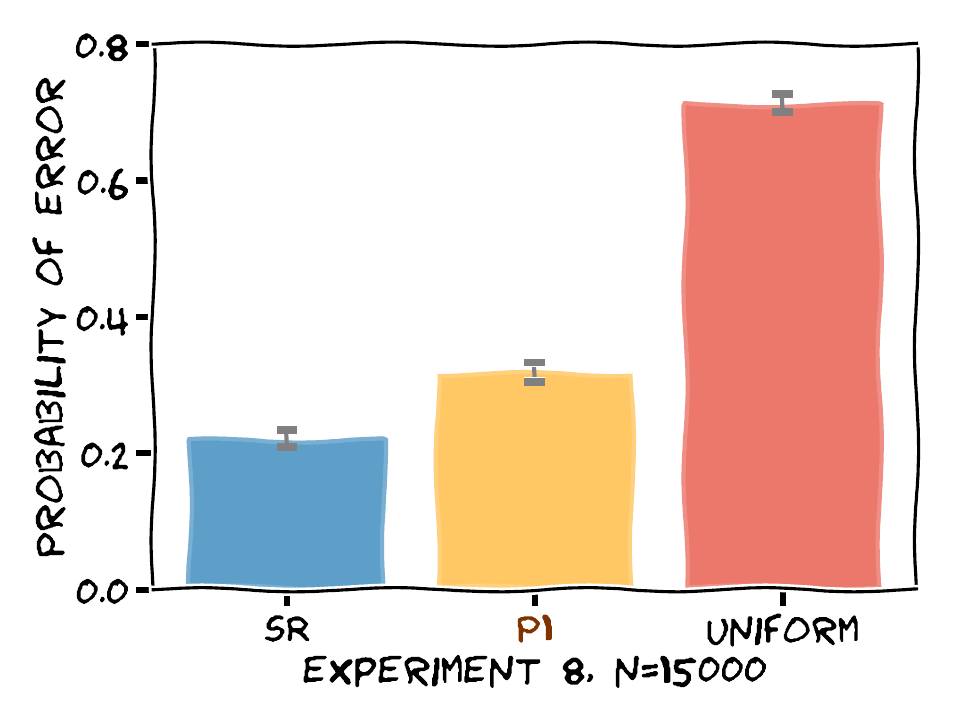}
	\caption{Probabilities of error of \SR, \Pone, and the static uniform allocation.}\label{exp}
\end{figure}

%% file: appUpperUniform.tex
% !TEX root = Main.tex 
\section{Upper bound on the probability of error of \RULE{} in the 
	general adversarial case} %(Theorem~\ref{th:UPARU}) }
\label{app:proofupadPRO}
\setcounter{scratchcounter}{\value{theorem}}\setcounter{theorem}{\the\numexpr\getrefnumber{th:UPARU}-1}
\tata*
\setcounter{theorem}{\the\numexpr\value{scratchcounter}}
\begin{proof}%[Theorem~\ref{th:UPARU}]
Fix a deterministic oblivious table
$\gainVector\in[0,1]^{\nArms\times\timeHorizon}$ and relabel its unique
best arm as $1$. Only the learner is random. Its fresh uniform draws
make the selection indicators independent across rounds for each
fixed arm $k$, with Bernoulli parameter $1/\nArms$.
Thus $\estGainVector_{k,t}=\nArms\gainVector_{k,t}\indicator{\pulledArm_t=k}$.
Put
\[
X_{k,t}=\estGainVector_{k,t}-\gainVector_{k,t},\qquad
M_k(\timeHorizon)=\sum_{t=1}^{\timeHorizon}X_{k,t}
=\estCumulGainVector_{k,\timeHorizon}-\cumulGain_{k,\timeHorizon},
\]
where $\cumulGain_{k,\timeHorizon}=\sum_{t=1}^{\timeHorizon}\gainVector_{k,t}$.
For each fixed arm, the centered variables are independent and satisfy
\begin{align*}
\Exp[X_{k,t}]&=0,\qquad |X_{k,t}|\leq\nArms,\\
\var_{X_{k,t}}
&=\nArms^2\gainVector_{k,t}^2\frac1{\nArms}
\left(1-\frac1{\nArms}\right)
=(\nArms-1)\gainVector_{k,t}^2\leq\nArms.
\end{align*}
Hence $\sum_{t=1}^{\timeHorizon}\var_{X_{k,t}}\leq\timeHorizon\nArms$.

For this table,
$\timeHorizon\gap_j^{\gainVector}
=\cumulGain_{1,\timeHorizon}-\cumulGain_{j,\timeHorizon}$ for $j\neq1$,
and $\gap_1^{\gainVector}=\min_{j\neq1}\gap_j^{\gainVector}$.
The recommendation satisfies the error-event inclusion
\[
\{\recomArm_{\timeHorizon}\neq1\}
\subseteq
\{M_1(\timeHorizon)\leq-\timeHorizon\gap_1^{\gainVector}/2\}
\ \cup\!
\bigcup_{j=2}^{\nArms}
\{M_j(\timeHorizon)\geq\timeHorizon\gap_j^{\gainVector}/2\}.
\]
Indeed, outside the right-hand side, for every $j\neq1$,
\begin{align*}
\estCumulGainVector_{1,\timeHorizon}
&>\cumulGain_{1,\timeHorizon}-\timeHorizon\gap_1^{\gainVector}/2
\geq\cumulGain_{1,\timeHorizon}-\timeHorizon\gap_j^{\gainVector}/2\\
&=\cumulGain_{j,\timeHorizon}+\timeHorizon\gap_j^{\gainVector}/2
>\estCumulGainVector_{j,\timeHorizon}.
\end{align*}
Thus arm $1$ is the unique estimated maximizer, so estimated-score
ties cause no difficulty.

Bernstein's inequality for the independent centered variables gives,
for each sign separately and $z>0$,
\[
\Pro(\pm M_k(\timeHorizon)\geq z)
\leq\exp\left(-\frac{z^2}
{2(\timeHorizon\nArms+\nArms z/3)}\right).
\]
With $z=\timeHorizon\gap_k^{\gainVector}/2$ and
$0<\gap_k^{\gainVector}\leq1$, this yields
\begin{align*}
\Pro\left(\pm M_k(\timeHorizon)
\geq\timeHorizon\gap_k^{\gainVector}/2\right)
&\leq\exp\left(-\frac{\timeHorizon(\gap_k^{\gainVector})^2}
{8\nArms(1+\gap_k^{\gainVector}/6)}\right)\\
&\leq\exp\left(-\frac{3\timeHorizon(\gap_k^{\gainVector})^2}
{28\nArms}\right).
\end{align*}
A union bound over the $\nArms$ error events, one best-arm lower tail
and one upper tail for each competitor, therefore gives
\begin{align*}
\ProErr_{\advPro(\gainVector)}(\timeHorizon)
&\leq\sum_{k=1}^{\nArms}
\exp\left(-\frac{3\timeHorizon(\gap_k^{\gainVector})^2}
{28\nArms}\right)\\
&\leq\nArms\exp\left(-\frac{3\timeHorizon(\gap_1^{\gainVector})^2}
{28\nArms}\right)
=\nArms\exp\left(-\frac{3\timeHorizon}
{28\complexityUnif(\gainVector)}\right),
\end{align*}
using $\gap_k^{\gainVector}\geq\gap_1^{\gainVector}$ and the existing
definition $\complexityUnif(\gainVector)=\nArms/(\gap_1^{\gainVector})^2$.
\end{proof}%
%    \todomi{Michal: we could have $3K$, instead of $8K$, no?\\
%    Victor: I made the computation clearer and now found a 40! do you agree?\\
%    Michal: 32! You forgot additional $(/2)$ in right term of the denominator
%    }
%
%
%

%% file: appTechLem.tex
%
% !TEX root = Main.tex 
\section{Change of measure}
\label{app:changeMeasureLem}
The next lemma is a finite-horizon version of the observed-history
change-of-measure method; see
\citet[Lemma~1 and Appendix~A]{Kaufmann16CBI}.
We give a self-contained proof allowing common finite-support reward
laws and include the Bretagnolle--Huber testing bound; see
\citet[Theorem~14.2]{Lattimore20BA}.

\begin{lemma}[\textcolor{titleTh}{Change of measure for observed transcripts}]
\label{l:transcriptChangeMeas}
Fix a positive integer horizon $\timeHorizon$, $\nArms\geq2$ arms,
an arm $k\in\setArms$, and a deterministic set of rounds
$\mathcal T\subseteq\{1,\ldots,\timeHorizon\}$, possibly the entire
horizon. The arm and set are fixed before interaction. In each of two
environments $A$ and $B$, the reward-table entries are independent,
independent of the learner's internal randomness, and have
finite-support distributions in $[0,1]$. At the pairs $(k,t)$ with
$t\in\mathcal T$, environment $A$ uses Bernoulli$(u)$ and environment
$B$ uses Bernoulli$(v)$, with fixed $u,v\in(0,1)$. At every other
arm-round pair, the two distributions are identical; these common
distributions may depend on the arm and round and may be deterministic.
Since there are finitely many entries, a finite reward alphabet
contains all their supports in both environments, including any
values other than $0$ and $1$.

Run the same adaptive randomized learner with the same horizon in both
environments. Its actions use only past actions, observed rewards,
and internal randomness, which may persist across rounds. It is not
given the environment's identity. Its final recommendation
$\recomArm_{\timeHorizon}\in\setArms$ also uses only the observed
interaction and internal randomness.
Let $\sample_t=\gainVector_{\pulledArm_t,t}$ be its observed reward and set
\[
Z=(\pulledArm_1,\sample_1,\ldots,
\pulledArm_{\timeHorizon},\sample_{\timeHorizon},\recomArm_{\timeHorizon}).
\]
Only observed rewards belong to $Z$, not the complete table.
Let $P_A$ and $P_B$ be its laws on the finite transcript space
$\mathcal Z$, and let $\Exp_A$ denote expectation in environment $A$.
Define
\[
N_k(\mathcal T)=\sum_{t\in\mathcal T}\indicator{\pulledArm_t=k}.
\]
Use natural logarithms and binary relative entropy
\[
\operatorname{kl}(a,b)
=a\log\frac{a}{b}+(1-a)\log\frac{1-a}{1-b},
\qquad a\in[0,1],\quad b\in(0,1),
\]
where zero-weight terms are zero. For laws with common support on
$\mathcal Z$, define
\[
D(P_A\Vert P_B)
=\sum_{z:P_A(z)>0}P_A(z)\log\frac{P_A(z)}{P_B(z)}.
\]
Then:
\begin{enumerate}[label=(\roman*)]
\item $P_A$ and $P_B$ have the same support.
\item Their relative entropy satisfies
\[
D(P_A\Vert P_B)=\Exp_A[N_k(\mathcal T)]\operatorname{kl}(u,v).
\]
\item For every event $W$ determined by the observed transcript, set
$p=P_A(W)$ and $q=P_B(W)$. If $0<q<1$, then
\[
\operatorname{kl}(p,q)\leq D(P_A\Vert P_B),
\qquad
p\log(1/q)\leq\Exp_A[N_k(\mathcal T)]\operatorname{kl}(u,v)+\log2.
\]
If $q=0$, then $p=0$; if $q=1$, then $p=1$.
\item For every observed-transcript event $W$, the
Bretagnolle--Huber testing bound holds:
\begin{equation}\label{eq:transcriptTesting}
P_A(W^c)+P_B(W)\geq\frac12\exp\left(-D(P_A\Vert P_B)\right).
\end{equation}
\end{enumerate}
The event must be an observed-transcript event; this statement does
not directly apply to events involving unobserved table entries.
\end{lemma}

\begin{proof}
Draw all learner randomness as a single seed $R$ before interaction,
independently of the reward table. Once this seed is fixed, actions,
internal state updates, and the final recommendation are determined
by the observed history. This includes randomness retained across
rounds and final randomization. For a fixed
transcript $z=(i_1,x_1,\ldots,i_{\timeHorizon},x_{\timeHorizon},j)$,
let $C_z$ be the set of seeds for which the learner produces the
recorded actions and recommendation when sequentially supplied the
recorded observations. Set $w(z)=\Pro_R(C_z)$ and let
$f^E_{j,t}(x)$ be the probability mass function of entry $(j,t)$
in environment $E$ on the finite reward alphabet.
For each seed in $C_z$, the observed table coordinates $(i_t,t)$
are distinct and independent. Their probability of matching the
recorded rewards is the product of their probability masses;
all unobserved entries integrate to one. Integrating over these seeds
therefore gives, for $E\in\{A,B\}$,
\[
P_E(z)=w(z)\prod_{t=1}^{\timeHorizon}
f^E_{i_t,t}(x_t).
\]
The factor $w(z)$ is the same in both environments and need not sum
to one over transcripts. At each arm-round pair, the two reward
distributions have the same support: they are either identical or
Bernoulli with parameters in $(0,1)$. Although common reward factors
can be zero, the factorization therefore implies
\[
P_A(z)>0\quad\Longleftrightarrow\quad P_B(z)>0.
\]
This proves common support without requiring the learner to assign
positive probability to every action.

On this common positive support, $w(z)$ and all observed reward
factors are positive. The learner factor and identical reward factors
cancel from the likelihood ratio, leaving only observed pulls of arm
$k$ at rounds in $\mathcal T$. Thus
\[
\log\frac{P_A(Z)}{P_B(Z)}
=\sum_{t\in\mathcal T}\indicator{\pulledArm_t=k}
\left[\sample_t\log\frac{u}{v}
+(1-\sample_t)\log\frac{1-u}{1-v}\right].
\]
At each round, condition on the seed, the past interaction, and the
selected arm. Under $A$, if $t\in\mathcal T$ and $\pulledArm_t=k$,
the current observed reward is Bernoulli$(u)$, since its table entry
is independent of the seed and past. Taking expectations yields
\[
D(P_A\Vert P_B)
=\Exp_A\left[\sum_{t\in\mathcal T}
\indicator{\pulledArm_t=k}\operatorname{kl}(u,v)\right]
=\Exp_A[N_k(\mathcal T)]\operatorname{kl}(u,v).
\]
Common deterministic rounds contribute zero additional information,
even though the learner's actions and final recommendation may
depend on their rewards.

For $0<q<1$, regard $W$ as a subset of $\mathcal Z$ and let
$\mathcal S$ be the common positive support. Applying the finite
log-sum inequality separately to $W$ and its complement gives
\begin{align*}
D(P_A\Vert P_B)
&=\sum_{z\in W\cap\mathcal S}P_A(z)\log\frac{P_A(z)}{P_B(z)}
+\sum_{z\in\mathcal S\setminus W}P_A(z)\log\frac{P_A(z)}{P_B(z)}\\
&\geq p\log\frac{p}{q}+(1-p)\log\frac{1-p}{1-q}
=\operatorname{kl}(p,q).
\end{align*}
Zero-mass groups are interpreted by continuity. Moreover,
\begin{align*}
\operatorname{kl}(p,q)
&=p\log(1/q)+(1-p)\log(1/(1-q))\\
&\quad+p\log p+(1-p)\log(1-p)\\
&\geq p\log(1/q)-\log2.
\end{align*}
The second term in the first line is nonnegative, and binary entropy
is maximized at $p=1/2$ with value $\log2$. Endpoint values of $p$ use the
zero-weight convention. This proves the event consequence.
When $q=0$ or $q=1$, common support gives respectively $p=0$ or $p=1$
without evaluating undefined logarithms.

For the testing bound, on the common positive support $\mathcal S$,
define
\[
a=\sum_{z\in\mathcal S}\sqrt{P_A(z)P_B(z)},\qquad
s=\sum_{z\in\mathcal S}\min\{P_A(z),P_B(z)\}.
\]
Jensen's inequality gives
\[
\log a
=\log\Exp_A\left[\sqrt{\frac{P_B(Z)}{P_A(Z)}}\right]
\geq-\frac12D(P_A\Vert P_B),
\]
so $a^2\geq\exp(-D(P_A\Vert P_B))$. By Cauchy--Schwarz,
\[
a^2\leq
\left(\sum_{z\in\mathcal S}\min\{P_A(z),P_B(z)\}\right)
\left(\sum_{z\in\mathcal S}\max\{P_A(z),P_B(z)\}\right)
=s(2-s)\leq2s.
\]
Finally, $P_A(W^c)+P_B(W)\geq s$, proving
Equation~\ref{eq:transcriptTesting}. This also covers $W$ empty or
equal to the whole transcript space without taking logarithms of
event probabilities.

The distinguishing information equals the expected number of
informative pulls times the information supplied by each pull.
For $\mathcal T=\{1,\ldots,m\}$, $N_k(\mathcal T)=\pullsNumber_k(m)$.
\end{proof}

%% file: appLOWAd.tex
% !TEX root = Main.tex
\section{Lower bound in the general adversarial setting} %(Theorem~\ref{th:LOWadPRO}) }
\label{app:prooflowadPRO}
	\setcounter{scratchcounter}{\value{theorem}}\setcounter{theorem}{\the\numexpr\getrefnumber{th:LOWadPRO}-1}
\toto*
\setcounter{theorem}{\the\numexpr\value{scratchcounter}}
\begin{proof}%[Theorem~\ref{th:LOWadPRO}]
The argument compares two environments that require different best-arm
recommendations but provide insufficient information for the learner to
reliably distinguish them. We quantify this difficulty using the distributions
of the learner's observed interactions, rather than the complete reward
tables, since unobserved rewards cannot help the learner identify the
environment.

Fix the reference gaps
$0<\gap_1=\gap_2\leq\cdots\leq\gap_{\nArms}\leq1/8$,
a learner, and a positive integer horizon $\timeHorizon$ satisfying
Theorem~\ref{th:LOWadPRO}. Write
$H=\complexityUnif=\nArms/\gap_1^2$ for the reference complexity.
The learner may know the horizon. The constructions and selected arm
below may depend on the learner and horizon, but not on its realized
random choices. Set
\[
m=\lfloor\timeHorizon/2\rfloor,\qquad \beta=m/\timeHorizon.
\]
The horizon assumption gives
\[
\timeHorizon\gap_1^2\geq128\log(128\nArms)\geq1.
\]
Since $\gap_1\leq1/8$, we have $\timeHorizon\geq64$, and therefore
$1\leq m<\timeHorizon$. Moreover,
\[
\frac{11}{24}\leq\frac12-\frac{1}{2\timeHorizon}
\leq\beta\leq\frac12,
\]
where the first inequality uses $\timeHorizon\geq12$.

Consider a base problem \basePro{} with independent Bernoulli rewards
of means $\mean_1^{\basePro}=1/2$ and
$\mean_j^{\basePro}=1/2-\gap_j$ for $j\in\setArmsmo$.
Let $\pullsNumber_j(m)$ be the number of pulls of arm $j$ through
round $m$. Since the expected pull counts of arms in $\setArmsmo$
sum to at most $m$, choose a fixed arm $\refarm\in\setArmsmo$ such that
\begin{equation}\label{eq:notSoMuchPullsAD}
\Exp_{\basePro}\left[\pullsNumber_{\refarm}(m)\right]
\leq\frac{m}{\nArms-1}\leq\frac{\timeHorizon}{\nArms}.
\end{equation}
The last inequality follows from $m\leq\timeHorizon/2$ and
$\nArms\leq2(\nArms-1)$ for $\nArms\geq2$.
This arm is chosen from expectations under \basePro{}, before either
adversarial table is sampled. The initial phase is the prefix
$\earlyPhase=\{1,\ldots,m\}$.

Set
\[
c=\gap_{\refarm}-\gap_1.
\]
Since $\gap_1\leq\gap_{\refarm}\leq1/8$, we have $0\leq c\leq1/8$.
Construct two distributions $\advPro_1$ and $\advPro_2$ over complete
reward tables, sampled before interaction and independently of the
learner's internal randomness. In $\advPro_1$, the prefix follows
\basePro{} on all arms. In $\advPro_2$, the prefix also follows
\basePro{}, except that arm $\refarm$ has Bernoulli mean
$1/2-\gap_{\refarm}+2\gap_1$. All prefix entries are independent
across arms and rounds. In both environments the second phase has
$m$ rounds and is identical and deterministic: for $m<t\leq2m$,
\[
\gainVector_{\refarm,t}=c,\qquad
\gainVector_{j,t}=0\quad(j\neq\refarm).
\]
For $2m<t\leq\timeHorizon$, set $\gainVector_{j,t}=0$ for every arm
$j$. Thus even horizons give the original two equal phases, while
odd horizons append one all-zero final round. These rewards lie in
$[0,1]$. The original learner, configured for horizon $\timeHorizon$,
chooses an arm in every round, including the optional last round,
and recommends an arm only after round $\timeHorizon$.
Let $\Pro_\ell$ and $\Exp_\ell$ denote the full table-and-learner
probabilities and expectations in $\advPro_\ell$, for $\ell\in\{1,2\}$.
On tables with tied maximizing arms, define $\bestArm_{\gainVector}$
as the smallest maximizing index.

Write $M_j^\ell=\Exp_\ell[\cumulGain_j]$, where
$\cumulGain_j=\sum_{t=1}^{\timeHorizon}\gainVector_{j,t}$.
The identity $mc=m(\gap_{\refarm}-\gap_1)$ gives the exact expected
totals below; the optional final zero round contributes nothing:
\begin{align*}
M_1^1=M_1^2&=m/2,\\
M_j^1=M_j^2&=m(1/2-\gap_j)
\quad(j\notin\{1,\refarm\}),\\
M_{\refarm}^1
&=m(1/2-\gap_{\refarm})+mc
=m(1/2-\gap_1),\\
M_{\refarm}^2
&=m(1/2-\gap_{\refarm}+2\gap_1)+mc
=m(1/2+\gap_1).
\end{align*}
The common deterministic second phase places arm $\refarm$ just below
arm $1$ in the first expected table and just above it in the second. The expected average
gap between them is $\beta\gap_1$, irrespective of the original
size of $\gap_{\refarm}$, while the environments differ only in the
initial rewards of the lightly sampled arm.

For $\ell\in\{1,2\}$ define the full-table good event
\[
\eventc_\ell=\left\{\forall j\in\setArms,\quad
|\cumulGain_j-M_j^\ell|\leq\frac{\timeHorizon\gap_1}{16}\right\}.
\]
Only the first $m$ rewards of each arm are random. Two-sided Hoeffding
and a union bound over the complete arm sequences give
\begin{equation}\label{eq:goodEventGenAd}
\Pro_\ell(\eventc_\ell^c)
\leq2\nArms\exp\left(-\frac{2(\timeHorizon\gap_1/16)^2}{m}\right)
\leq2\nArms\exp\left(-\frac{\timeHorizon\gap_1^2}{64}\right),
\end{equation}
since $m\leq\timeHorizon/2$. These events involve the full reward
table and are generally not determined by the learner's observations.

On either good event, every normalized pairwise difference deviates
from its expectation by at most $\gap_1/8$.
On $\eventc_1$, the expected differences from arm $1$ are
$\beta\gap_1$ for arm $\refarm$ and $\beta\gap_j$ for
$j\notin\{1,\refarm\}$. Since
$\beta-1/8\geq1/3$, $\beta+1/8\leq5/8$, and $\gap_1\leq\gap_j$,
\[
\frac{\gap_1}{3}\leq
\frac{\cumulGain_1-\cumulGain_{\refarm}}{\timeHorizon}
\leq\frac{5\gap_1}{8},\qquad
\frac{\gap_j}{3}\leq
\frac{\cumulGain_1-\cumulGain_j}{\timeHorizon}
\leq\frac{5\gap_j}{8}
\quad(j\notin\{1,\refarm\}).
\]
All these differences are positive, so arm $1$ is uniquely best.
Its own gap is the minimum of these differences; the lower bounds
and the upper bound from arm $\refarm$ give
\[
\frac{\gap_1}{3}\leq\gap_1^{\gainVector}\leq\frac{5\gap_1}{8}.
\]
Thus every label satisfies the defining bounds for $\Gaps_3$, with
$\refarm$ as the exceptional label, and every table in $\eventc_1$
belongs to $\Gaps_3$.

On $\eventc_2$, the expected differences from arm $\refarm$ are
$\beta\gap_1$ for arm $1$ and $\beta(\gap_j+\gap_1)$ for
$j\notin\{1,\refarm\}$. Consequently,
\[
\frac{\gap_1}{3}\leq
\frac{\cumulGain_{\refarm}-\cumulGain_1}{\timeHorizon}
\leq\frac{5\gap_1}{8},\qquad
\frac{\gap_j}{3}\leq
\frac{\cumulGain_{\refarm}-\cumulGain_j}{\timeHorizon}
\leq\frac{9\gap_j}{8}
\quad(j\notin\{1,\refarm\}).
\]
For the latter lower bound,
$\beta(\gap_j+\gap_1)-\gap_1/8
\geq(\beta-1/8)\gap_j\geq\gap_j/3$.
For the upper bound, $\beta\leq1/2$ and $\gap_1\leq\gap_j$ give
$\beta(\gap_j+\gap_1)+\gap_1/8\leq9\gap_j/8$.
Arm $\refarm$ is uniquely best, and its own gap, the minimum of
its differences from every other arm, satisfies
\[
\frac{\gap_1}{3}\leq\gap_{\refarm}^{\gainVector}\leq\frac{5\gap_1}{8}.
\]
The upper bound uses the difference from arm $1$. Label $1$ meets
its ordinary reference-gap requirement, all other labels except
$\refarm$ meet theirs, and $\refarm$ meets the exceptional requirement.
Every table in $\eventc_2$ therefore also belongs to $\Gaps_3$.

Let $P_1,P_2$ be the observed-transcript laws from
Lemma~\ref{l:transcriptChangeMeas} under $\advPro_1,\advPro_2$,
distinguished from the full laws $\Pro_1,\Pro_2$.
Apply the lemma with $A=\advPro_1$, $B=\advPro_2$,
$\mathcal T=\{1,\ldots,m\}$, and
\[
u=\frac12-\gap_{\refarm},\qquad
v=\frac12-\gap_{\refarm}+2\gap_1.
\]
All table entries are independent of one another and of the learner's
randomness. The common distributions have finite support, including
the identical deterministic second phase and optional last round,
and $u,v\in[3/8,5/8]\subset(0,1)$. This entire continuation contributes
no additional likelihood-ratio factors.
The same learner is run with the same horizon.
The lemma gives
\[
D(P_1\Vert P_2)
=\Exp_1[\pullsNumber_{\refarm}(m)]\operatorname{kl}(u,v)
=\Exp_{\basePro}[\pullsNumber_{\refarm}(m)]\operatorname{kl}(u,v).
\]
The second equality holds because $\advPro_1$ and \basePro{} have
identical interaction laws through the prefix, including the learner's
adaptive actions. No law is conditioned on a good-table event.

Using $\log z\leq z-1$,
\[
\operatorname{kl}(u,v)
\leq u\left(\frac{u}{v}-1\right)
+(1-u)\left(\frac{1-u}{1-v}-1\right)
=\frac{(u-v)^2}{v(1-v)}.
\]
Here $|u-v|=2\gap_1$ and $v(1-v)\geq15/64$, so
$\operatorname{kl}(u,v)\leq(256/15)\gap_1^2$.
Set $x=\timeHorizon/H=\timeHorizon\gap_1^2/\nArms>0$.
Equation~\ref{eq:notSoMuchPullsAD} yields
\begin{equation}\label{eq:transcriptKLGenAd}
D(P_1\Vert P_2)
\leq\frac{256}{15}\gap_1^2\frac{\timeHorizon}{\nArms}
=\frac{256}{15}x\leq32x.
\end{equation}
For the transcript event $W=\{\recomArm_{\timeHorizon}=1\}$,
the testing bound in Equation~\ref{eq:transcriptTesting} now gives
\[
P_1(\recomArm_{\timeHorizon}\neq1)+P_2(\recomArm_{\timeHorizon}=1)
\geq\frac12\exp(-D(P_1\Vert P_2))\geq\frac12\exp(-32x).
\]
The learner may recommend any arm; recommending arm $1$ is simply
one error event in the second environment on its good tables.

Write $\varepsilon=\exp(-32x)$ and
$b_\ell=\Pro_\ell(\eventc_\ell^c)$.
Both theorem assumptions are needed to compare the bad-table
probabilities with this testing bound. Equation~\ref{eq:goodEventGenAd}
and $\nArms\exp(-\timeHorizon\gap_1^2/128)\leq1/128$ give
\begin{align*}
b_\ell
&\leq2\nArms\exp(-\timeHorizon\gap_1^2/64)\\
&=2\left[\nArms\exp(-\timeHorizon\gap_1^2/128)\right]
\exp(-\timeHorizon\gap_1^2/128)\\
&\leq\frac1{64}\exp(-\timeHorizon\gap_1^2/128)
\leq\frac{\varepsilon}{64}.
\end{align*}
The last inequality uses $\nArms\geq4096$, which implies
$32/\nArms\leq1/128$. Thus bad tables cost only a fixed fraction
of the exponentially small testing bound.

For a fixed table $\gainVector$, let
$\ProErr_{\advPro(\gainVector)}(\timeHorizon)$ denote error over the
learner's internal randomness, using the best-arm convention above.
Define its error mass on good tables by
\[
R_\ell=\Exp_{\gainVector\sim\advPro_\ell}
\left[\ProErr_{\advPro(\gainVector)}(\timeHorizon)
\indicator{\gainVector\in\eventc_\ell}\right].
\]
The unique best arms on the good events give
\begin{align*}
R_1&=\Pro_1(\{\recomArm_{\timeHorizon}\neq1\}\cap\eventc_1),\\
R_2&=\Pro_2(\{\recomArm_{\timeHorizon}\neq\refarm\}\cap\eventc_2)
\geq\Pro_2(\{\recomArm_{\timeHorizon}=1\}\cap\eventc_2).
\end{align*}
Subtracting the bad-table probabilities therefore yields
\begin{equation}\label{eq:goodErrorGenAd}
\begin{aligned}
R_1+R_2
&\geq P_1(\recomArm_{\timeHorizon}\neq1)
+P_2(\recomArm_{\timeHorizon}=1)-b_1-b_2\\
&\geq\frac{\varepsilon}{2}-\frac{2\varepsilon}{64}
=\frac{15\varepsilon}{32}.
\end{aligned}
\end{equation}
All joint events involving $\eventc_\ell$ use the full law $\Pro_\ell$;
no independence from the recommendation is required.

Each table distribution has finite support: only its $m\nArms$
prefix entries are random Bernoulli variables and its suffix is fixed.
Each good subset of that support is nonempty, since
$b_\ell\leq\varepsilon/64<1$. Choose a deterministic table
$\gainVector^\ell$ maximizing
$\ProErr_{\advPro(\gainVector)}(\timeHorizon)$ over the good support of
$\advPro_\ell$. Then
\[
R_\ell\leq\Pro_\ell(\eventc_\ell)
\ProErr_{\advPro(\gainVector^\ell)}(\timeHorizon)
\leq\ProErr_{\advPro(\gainVector^\ell)}(\timeHorizon).
\]
Equation~\ref{eq:goodErrorGenAd} implies
\[
\max_{\ell\in\{1,2\}}
\ProErr_{\advPro(\gainVector^\ell)}(\timeHorizon)
\geq\frac{R_1+R_2}{2}
\geq\frac{15}{64}\exp\left(-\frac{32\timeHorizon}{H}\right).
\]
Both tables belong to $\Gaps_3$ and have different unique best arms,
$1$ and $\refarm$. Since $15/64\geq1/128$, this bound implies
\[
\max_{\ell\in\{1,2\}}\ProErr_{\gainVector^\ell}(\timeHorizon)
\geq\min\left\{\frac1{128}\exp\left(-\frac{32\timeHorizon}
{\complexityUnif}\right),\frac1{32}\right\},
\]
which is the claimed conclusion. The extraction uses finite-support
maximization and requires neither conditional probabilities on null
events nor a coupling of the two table draws.
\end{proof}%

%% file: appLOWBOB.tex
% !TEX root = Main.tex
%

\section{Lower bound in the best of both worlds} %setting (Theorem~\ref{th:lowBOB}) }
\label{app:prooflowBOB}
\setcounter{scratchcounter}{\value{theorem}}\setcounter{theorem}{\the\numexpr\getrefnumber{th:lowBOB}-1}
\thi*
\setcounter{theorem}{\the\numexpr\value{scratchcounter}}
\begin{proof}%[Theorem~\ref{th:lowBOB}]\
As in Theorem 2, we exploit the fact that two environments with different best
arms can produce observations that are difficult to distinguish. Here, the
stochastic guarantee forces the learner to recommend one arm with very high
probability, while the adversarial environment requires a different
recommendation. We use a change-of-measure argument to quantify this conflict,
then show that it persists for a deterministic adversarial reward table.

Fix the reference gaps defining $\Gaps_4$, a learner, and a
horizon $\timeHorizon$ satisfying the hypothesis of the theorem.
Throughout the proof, the learner is run with this fixed horizon,
which it may know. The constructions below may depend on this
learner and on $\timeHorizon$.
Let $\pullsNumber_k(t)$ be the number of times arm $k$ has been
pulled by the end of round $t$.
	Let us consider a base problem, denoted \basePro{}, with $\nu_1,\ldots, \nu_\nArms$, 
	Bernoulli distributions with mean $\mean^{\basePro}_1,
	\ldots, \mean^{\basePro}_\nArms$ 
	such that $\mean^{\basePro}_1 \triangleq 1/2$, and for all $k\in[2,\nArms], 
	\mean^{\basePro}_k \triangleq  1/2-\gap_k$ 
	(the gaps specified in the claim of the theorem).
	The expectation and probability with respect to the
	learner and the samples of this problem are 
	$\Exp_{\basePro}$ and $\Pro_{\basePro}$.% with $\gap_k\geq 0$ for .
	
	With respect to this base problem, choose any maximizing index $i\in\arg\max_{j\in\setArms}\frac{j}{\gap_j}$.
 %denotes the rank of a suboptimal arm in the	base problem. 
\noindent\textbf{Choice of the initial phase length.}
We choose the length of the initial phase so as to maximize the resulting
lower bound on the learner's error. The stochastic and adversarial
constructions differ only during this phase. A shorter phase reduces
the phase-length contribution $\timeHorizon_i\gap_i^2/i$ to the
information-theoretic estimate and therefore strengthens the resulting
error lower bound.

However, the initial phase must be sufficiently long to ensure that
the adversarial problem belongs to the prescribed class $\Gaps_4$.
Its expected gaps must remain large enough after accounting for
fluctuations of the cumulative rewards.

We therefore seek the shortest initial phase, up to universal constant
factors, that guarantees these gap constraints uniformly over arms
satisfying the selection bounds. The choice below uses a convenient
factor $2$ to provide additional gap margin. Together with the phase
bound below, it keeps $\timeHorizon_i\gap_i^2/i$ bounded by a universal
constant times $\timeHorizon/\complexityBoth^{\basePro}$.

Set
\[
\dividefac_i\triangleq\min\left\{1,\frac{2\gap_1}{\gap_i}\right\},
\qquad 0<\dividefac_i\leq1.
\]
The initial phase $\earlyPhase_i$ is the prefix of rounds
$t=1,\ldots,\timeHorizon_i$, where
\[
\timeHorizon_i\triangleq\left\lceil\timeHorizon\dividefac_i\right\rceil
=\min\left\{\timeHorizon,
\left\lceil\frac{2\timeHorizon\gap_1}{\gap_i}\right\rceil\right\}.
\]
The equality uses the fact that $\timeHorizon$ is an integer;
in particular, $1\leq\timeHorizon_i\leq\timeHorizon$.

The theorem's horizon condition implies
\[
\timeHorizon\gap_1^2\geq128\log(32\nArms)\geq1.
\]
Writing $r=\timeHorizon\gap_1/\gap_i$, since
$\gap_1\gap_i\leq1/64$, we have
\[
r=\frac{\timeHorizon\gap_1^2}{\gap_1\gap_i}\geq1,
\qquad
\frac{\gap_i}{\timeHorizon\gap_1}
=\frac{\gap_1\gap_i}{\timeHorizon\gap_1^2}
\leq\frac{1}{64}\leq\frac18.
\]
Thus $\gap_i/\timeHorizon\leq\gap_1/8$, and the capped ceiling gives
\[
\timeHorizon_i\leq\lceil2r\rceil\leq2r+1\leq3r
=\frac{3\timeHorizon\gap_1}{\gap_i}.
\]

Since $\gap_1=\gap_2>0$, we have $1/\gap_1<2/\gap_2$.
Thus $i\geq2$ for every choice of maximizing index.

	Now we want to identify an arm with a gap of the same order as $\gap_i$ that the 
	learner will not pull
	very much in the base problem \basePro{}
	during the period~$\earlyPhase_i$. This arm will be the one 'attacked' by the adversary to confuse the learner as the learner is not pulling it very much.

If $i=2$, choose $\refarm=2$. Since at most one pull of arm $2$
is made in each round, $\pullsNumber_2(n_i)\leq n_i$ almost surely,
and hence
\[
\Exp_{\basePro}\left[\pullsNumber_2(n_i)\right]
\leq n_i=\frac{2n_i}{i}.
\]
Under the phase definition above, $i=2$ gives $\dividefac_i=1$
and $n_i=\timeHorizon$.
The second phase is then empty, and the \advPro{} construction below
coincides with \basePro{} throughout the horizon. This empty second
phase is allowed by the construction.

If $i\geq3$, consider the arms with indices in
$[\lceil i/2\rceil:i]$. There are $\lfloor i/2\rfloor+1>i/2$
such arms. Since
\[
\sum_{k\in\setArms}\Exp_{\basePro}\left[\pullsNumber_k(n_i)\right]=n_i,
\]
they cannot all have expected pull counts strictly greater than $2n_i/i$.
Choose any $\refarm\in[\lceil i/2\rceil:i]$ whose expected pull count
is at most $2n_i/i$. In this case $\refarm\geq\lceil i/2\rceil\geq2$.

 By construction we have $\gap_{\refarm}\leq\gap_i$ and
	\begin{equation}\label{eq:notSoMuchPulls}
	\Exp_{\basePro}\left[\pullsNumber_{\refarm}(n_i) \right]\leq \frac{2n_i}{i}\cdot
	\end{equation}

In both cases, $2\leq\refarm\leq i$ and $\refarm\geq i/2$.
The maximizing property of $i$ gives
$i/\gap_i\geq\refarm/\gap_{\refarm}$, so
 \begin{equation}\label{eq:zhihan}
     \gap_i\leq\frac{i}{\refarm}\gap_{\refarm}\leq2\gap_{\refarm}.
 \end{equation}
Equality is permitted in this gap comparison. The resulting bound
$\gap_{\refarm}/\gap_i\geq1/2$ is sufficient for the expected-gap
lower bound below.

	Now we construct an i.i.d.\,stochastic problem, denoted \stoPro{}, where
	the arm distributions are the same as in the \basePro{}
	problem, except for arm $\refarm$, $\mu^{\stoPro}_{k}\triangleq \mu^{\basePro}_{k} $ 
	for all $k\neq \refarm$.
	%    We distinguish two cases: if the  $\refarm=1$ then
	We set in  \stoPro{},         
	$\mu^{\stoPro}_{\refarm} \triangleq 1/2 + \gap_1/4$.

The best arm in the constructed \stoPro{} problem is
$\bestArmSto\triangleq\refarm$. Its gaps, indexed by the original arm
labels, satisfy
\[
\gap_1^{\stoPro}=\gap_{\refarm}^{\stoPro}=\frac{\gap_1}{4},
\]
and, for every $j\in\setArms\setminus\{1,\refarm\}$,
\[
\gap_j\leq\gap_j^{\stoPro}
=\gap_j+\frac{\gap_1}{4}\leq\frac54\gap_j.
\]
Thus label $1$ has gap $\gap_1/4$, the selected label $\refarm$ is
the exceptional label with the same gap, and every remaining label
satisfies the ordinary bounds $[\gap_j/4,4\gap_j]$.
Consequently, $\stoPro\in\Gaps_4$.

We next compare the complexity of the constructed stochastic problem
with that of \basePro{}. Since the construction changes the ordering
of the arms, this comparison must use the sorted stochastic gaps.
Let $h_1\leq\cdots\leq h_\nArms$ be the sorted values of the labelled
gaps $\gap_j^{\stoPro}$, and define the local quantity
\[
Q\triangleq\max_{j\in\setArms}\frac{j}{\gap_j},
\qquad \complexityBoth^{\basePro}=\frac{Q}{\gap_1},
\qquad Q\geq\frac{2}{\gap_1}.
\]
The sorted stochastic gaps consist of two copies of $\gap_1/4$,
followed by the ordered list $\gap_j+\gap_1/4$ for
$j\in[2:\nArms]\setminus\{\refarm\}$.
Deleting the entry for $\refarm$ from the ordered list with labels
$2,\ldots,\nArms$ moves a remaining position by at most one.
This comparison remains valid when gaps tie. Hence
\[
h_1=h_2=\frac{\gap_1}{4},
\qquad
\gap_{r-1}+\frac{\gap_1}{4}\leq h_r
\leq\gap_r+\frac{\gap_1}{4}
\quad(3\leq r\leq\nArms).
\]
When $\nArms=2$, there are no ranks $r\geq3$; only the first two
ranks are used below.

For the upper bound, the first two ranks satisfy
\[
\max_{r\in\{1,2\}}\frac{r}{h_r}
=\frac{8}{\gap_1}\leq4Q.
\]
For each rank $r\geq3$, we have
\[
\frac{r}{h_r}\leq\frac{r}{\gap_{r-1}}
=\frac{r}{r-1}\frac{r-1}{\gap_{r-1}}
\leq\frac32Q\leq4Q.
\]
Thus, bounding every rank in $\setArms$ gives
\[
\complexityBoth^{\stoPro}
=\frac{4}{\gap_1}\max_{r\in\setArms}\frac{r}{h_r}
\leq\frac{4}{\gap_1}(4Q)
=16\complexityBoth^{\basePro}.
\]
For the lower bound, $h_r\leq2\gap_r$ for every rank $r$:
for $r=1,2$ this follows from $h_r=\gap_1/4$ and
$\gap_1=\gap_2$, while for $r\geq3$ it follows from
$h_r\leq\gap_r+\gap_1/4$ and $\gap_1\leq\gap_r$.
Consequently,
\[
\complexityBoth^{\stoPro}
\geq\frac{4}{\gap_1}\max_{r\in\setArms}\frac{r}{2\gap_r}
=2\complexityBoth^{\basePro}
\geq\frac12\complexityBoth^{\basePro}.
\]
Combining these bounds, we retain the weaker lower bound in
\begin{equation}\label{eq:newHBOBSTO}
  \complexityBoth^{\basePro} /2
	\leq
        \complexityBoth^{\stoPro} 
		\leq
 16 \complexityBoth^{\basePro} 
\end{equation}

Writing $\complexitySR^{\basePro}$ and $\complexitySR^{\stoPro}$
for the respective \SR{} complexities, the same sorted-gap comparison
and the stronger lower bound above give
\[
\begin{aligned}
\complexitySR^{\stoPro}
&\leq\max\left\{\frac{32}{\gap_1^2},
\frac32\complexitySR^{\basePro}\right\}
\leq16\complexitySR^{\basePro},\\
\frac{\complexityBoth^{\stoPro}}{\complexitySR^{\stoPro}}
&\geq\frac18\frac{\complexityBoth^{\basePro}}{\complexitySR^{\basePro}}.
\end{aligned}
\]
For the square-root profile $\gap_r=\gap_1\sqrt{r/2}$,
$r\in[2:\nArms]$, the ratio for \stoPro{} is at least
$\frac18\sqrt{\nArms/2}$, for every eligible choice of arm~$\refarm$.

	The expectation and probability with respect to the
	learner and the samples of this problem are denoted
	$\Exp_{\stoPro}$ and $\Pro_{\stoPro}$.

	The second bandit problem is the adversarial one, denoted \advPro{}.
Before interaction with the learner, it samples a table
$\gainVector\in\{0,1\}^{\nArms\times\timeHorizon}$.
For any sampled table with several maximizing arms, let
$\bestArm_{\gainVector}$ be the smallest maximizing index.
On $\eventc$ the maximizer is unique, so this convention does
not affect the extracted instance.
The Bernoulli entries are independent across arms and rounds,
and independent of the learner's internal randomness.
The table-generating distribution is fixed once the learner,
horizon, reference gaps, and the choices of $i$ and $\refarm$ are fixed;
it does not adapt to the learner's realized actions.
For each $t\in[\timeHorizon]$ and arm $k$, the $(k,t)$ entry is
Bernoulli with mean $\mean_k^{\advPro}(t)$.
	For all $t$ and $k\neq\refarm$, \advPro{} follows the \basePro{} 
	problem:  $\mu_{k}^{\advPro}(t)\triangleq\mu_{k}^{\basePro}$.
	For arm $\refarm$,
	until the end of phase $\earlyPhase_i$, for all $t$ with $1\leq 
	t\leq n_i$, $\mu^{\advPro}_{\refarm}(t)\triangleq\mu_{\refarm}^{\basePro}$ 
	and then a switch happens, for $n_i< t\leq n$, arm $\refarm$ 
	possesses  the same distributions as in the \stoPro{} problem,  
	$\mu^{\advPro}_{\refarm}(t)\triangleq \mu_{\refarm}^{\stoPro}$.
	The expectation and probability with respect to the
	learner and the samples of this problem are denoted
	$\Exp_{\advPro}$ and $\Pro_{\advPro}$.

Let us compute the expected cumulative gains for this fixed construction.
For each arm $k\in\setArms$, let
$\expectCumul_k\triangleq\sum_{t=1}^{\timeHorizon}\mean_k^{\advPro}(t)$.
Then
\[
\expectCumul_{\refarm}
=\timeHorizon_i\left(\frac12-\gap_{\refarm}\right)
+(\timeHorizon-\timeHorizon_i)\left(\frac12+\frac{\gap_1}{4}\right),
\qquad \expectCumul_1=\frac{\timeHorizon}{2},
\]
and, for $j\in\setArms\setminus\{1,\refarm\}$,
\[
\expectCumul_j=\timeHorizon\left(\frac12-\gap_j\right).
\]
Define the expected average gap
\[
b\triangleq\frac{\expectCumul_1-\expectCumul_{\refarm}}{\timeHorizon}
=\frac{\timeHorizon_i}{\timeHorizon}\gap_{\refarm}
-\left(1-\frac{\timeHorizon_i}{\timeHorizon}\right)\frac{\gap_1}{4}.
\]
If $\timeHorizon_i=\timeHorizon$, then $b=\gap_{\refarm}\geq\gap_1$.
This case also includes a full-horizon phase caused by ceiling rounding.
If $\timeHorizon_i<\timeHorizon$, its definition gives
$\timeHorizon_i\geq2\timeHorizon\gap_1/\gap_i$, so
Equation~\ref{eq:zhihan} yields
\[
b\geq\frac{2\gap_1\gap_{\refarm}}{\gap_i}-\frac{\gap_1}{4}
\geq\frac{3\gap_1}{4}.
\]
For the upper bound, in all cases,
\[
b\leq\frac{\timeHorizon_i}{\timeHorizon}\gap_{\refarm}
\leq\left(\frac{2\gap_1}{\gap_i}+\frac{1}{\timeHorizon}\right)\gap_i
=2\gap_1+\frac{\gap_i}{\timeHorizon}
\leq\frac{17\gap_1}{8}.
\]
Consequently,
\[
\frac{3\gap_1}{4}\leq b\leq\frac{17\gap_1}{8}.
\]

	Let us consider the event
	$\eventc =\left\{ \forall k\in\setArms,  |\cumulGain_k-\expectCumul_{k}| \leq  \timeHorizon \gap_1/16 \right\}
	$.
For a fixed arm $j$, $\cumulGain_j$ is the sum of all $\timeHorizon$
entries in its row of the reward table. These entries are independent
and lie in $[0,1]$, although their means may differ across rounds.
This is the full table sum, including rounds on which the learner
does not pull arm $j$. Two-sided Hoeffding gives
\begin{align*}
\Pro_{\advPro}\left(
|\cumulGain_j-\expectCumul_j|>\frac{\timeHorizon\gap_1}{16}\right)
&\leq2\exp\left(-\frac{2(\timeHorizon\gap_1/16)^2}{\timeHorizon}\right)\\
&=2\exp\left(-\frac{\timeHorizon\gap_1^2}{128}\right).
\end{align*}
Therefore, a union bound over arms yields
\begin{equation}\label{eq:goodEventBOB}
\Pro_{\advPro}(\eventc^c)
\leq2\nArms\exp\left(-\frac{\timeHorizon\gap_1^2}{128}\right)
\leq\frac{1}{16}.
\end{equation}
The last inequality uses the theorem hypothesis
\[
\nArms\exp\left(-\frac{\timeHorizon\gap_1^2}{128}\right)\leq\frac{1}{32}.
\]
Consequently, $\Pro_{\advPro}(\eventc)\geq15/16$.
%	\todov{here i put a power of 3 to be safe and i ttink it doesnt matter in the proof because we put it as an hypothesis of the theorem that $n$ is large enough}

On $\eventc$, each difference $\cumulGain_1-\cumulGain_j$
deviates from $\expectCumul_1-\expectCumul_j$ by at most
$\timeHorizon\gap_1/8$.
For the selected arm, the bounds on $b$ therefore give
\[
\frac{5\gap_1}{8}
\leq\frac{\cumulGain_1-\cumulGain_{\refarm}}{\timeHorizon}
\leq\frac{9\gap_1}{4}.
\]
For every $j\in\setArms\setminus\{1,\refarm\}$, we have
\[
\gap_j-\frac{\gap_1}{8}
\leq\frac{\cumulGain_1-\cumulGain_j}{\timeHorizon}
\leq\gap_j+\frac{\gap_1}{8}.
\]
Since $\gap_j\geq\gap_1$, this implies
\[
\frac78\gap_j
\leq\frac{\cumulGain_1-\cumulGain_j}{\timeHorizon}
\leq\frac98\gap_j.
\]
All these differences are strictly positive, so arm $1$ is the
unique best arm on $\eventc$. Thus the displayed differences are
the gaps $\gap_j^{\gainVector}$ for $j\neq1$.
The best arm's own gap is
\[
\gap_1^{\gainVector}
=\min_{j\in\setArms\setminus\{1\}}
\frac{\cumulGain_1-\cumulGain_j}{\timeHorizon}.
\]
The lower bounds for all other arms and the upper bound for the
selected arm give
\[
\frac{5\gap_1}{8}\leq\gap_1^{\gainVector}\leq\frac{9\gap_1}{4}.
\]
The bounds for labels $1$ and $\refarm$ lie in
$[\gap_1/4,4\gap_1]$, and the bounds for every remaining label $j$
lie in $[\gap_j/4,4\gap_j]$.
Hence every table $\gainVector\in\eventc$ belongs to $\Gaps_4$,
with $\refarm$ as the exceptional label in the class definition.

	Note that therefore the only difference between the two problems 
	\stoPro{} and \advPro{} is for  arm $\refarm$ during phase $\earlyPhase_i$.

The stochastic problem has best arm $\refarm$, whereas arm $1$ is
uniquely best in the adversarial construction on $\eventc$.
We compare the distributions of the learner's observations to quantify
the difficulty of distinguishing these problems.

Let $P_A$ and $P_S$ denote the laws of the observed transcript $Z$
from Lemma~\ref{l:transcriptChangeMeas} under randomized \advPro{}
and \stoPro{}, respectively. These transcript laws are distinguished
from the full table-and-learner law $\Pro_{\advPro}$ used for $\eventc$,
which also involves unobserved table entries.
Both constructions have independent Bernoulli entries, independent
of the learner's internal randomness, with all parameters in
$[3/8,17/32]\subset(0,1)$. The same learner is run with the same
horizon, and the only differing arm-round pairs are arm $\refarm$
through the initial prefix of $\timeHorizon_i$ rounds.
The arm and prefix are fixed before interaction.
Apply the lemma with $A=\advPro$, $B=\stoPro$,
$\mathcal T=\{1,\ldots,\timeHorizon_i\}$, and
\[
u=\frac12-\gap_{\refarm},\qquad v=\frac12+\frac{\gap_1}{4}.
\]
It gives
\[
D(P_A\Vert P_S)
=\Exp_{\advPro}\left[\pullsNumber_{\refarm}(\timeHorizon_i)\right]
\operatorname{kl}(u,v).
\]
Through this fixed prefix, \advPro{} and \basePro{} have identical
reward laws and use the same learner. Their interaction laws agree
through $\timeHorizon_i$, including the learner's adaptive actions.
Consequently,
\[
\Exp_{\advPro}\left[\pullsNumber_{\refarm}(\timeHorizon_i)\right]
=\Exp_{\basePro}\left[\pullsNumber_{\refarm}(\timeHorizon_i)\right],
\]
and hence
\[
D(P_A\Vert P_S)
=\Exp_{\basePro}\left[\pullsNumber_{\refarm}(\timeHorizon_i)\right]
\operatorname{kl}(u,v).
\]

For $u,v\in(0,1)$, the inequality $\log z\leq z-1$ gives
\begin{align*}
\operatorname{kl}(u,v)
&\leq u\left(\frac{u}{v}-1\right)
+(1-u)\left(\frac{1-u}{1-v}-1\right)\\
&=\frac{(u-v)^2}{v(1-v)}.
\end{align*}
Here $\gap_1\leq\gap_{\refarm}\leq1/8$, so
\[
|u-v|\leq\frac54\gap_{\refarm},\qquad
v(1-v)=\frac14-\frac{\gap_1^2}{16}\geq\frac{255}{1024}.
\]
Thus
\[
\operatorname{kl}(u,v)
\leq\frac{1600}{255}\gap_{\refarm}^2\leq8\gap_{\refarm}^2.
\]
Using Equation~\ref{eq:notSoMuchPulls}, $\gap_{\refarm}\leq\gap_i$,
the phase-length bound, and
$\complexityBoth^{\basePro}=i/(\gap_1\gap_i)$, we obtain
\begin{equation}\label{eq:transcriptKLBOB}
\begin{aligned}
D(P_A\Vert P_S)
&\leq8\gap_{\refarm}^2\frac{2\timeHorizon_i}{i}\\
&\leq\frac{16\gap_i^2\timeHorizon_i}{i}\\
&\leq\frac{48\timeHorizon\gap_1\gap_i}{i}\\
&=\frac{48\timeHorizon}{\complexityBoth^{\basePro}}\\
&\leq\frac{768\timeHorizon}{\complexityBoth^{\stoPro}}.
\end{aligned}
\end{equation}
The last inequality uses Equation~\ref{eq:newHBOBSTO}:
$\complexityBoth^{\stoPro}\leq16\complexityBoth^{\basePro}$ implies
$1/\complexityBoth^{\basePro}\leq16/\complexityBoth^{\stoPro}$.
The chosen phase is long enough to satisfy the adversarial gap
constraints, while its length bound keeps this information cost at
most a universal constant times $\timeHorizon/\complexityBoth^{\stoPro}$.

Assume the stochastic-error premise of Theorem~\ref{th:lowBOB}, and set
\[
x=\frac{\timeHorizon}{\complexityBoth^{\stoPro}}>0,\qquad
p=P_A(\recomArm_{\timeHorizon}=1),\qquad
q=P_S(\recomArm_{\timeHorizon}=1).
\]
Since the unique best arm in \stoPro{} is $\refarm\neq1$,
\[
q\leq\ProErr_{\stoPro}(\timeHorizon)
\leq\frac{1}{64}\exp(-2048x).
\]
If $q=0$, common support in Lemma~\ref{l:transcriptChangeMeas} gives
$p=0$. Otherwise $0<q<1$, and its event consequence for
$W=\{\recomArm_{\timeHorizon}=1\}$, together with
Equation~\ref{eq:transcriptKLBOB}, gives
\[
p\log(1/q)\leq D(P_A\Vert P_S)+\log2\leq768x+\log2.
\]
The stochastic guarantee also gives
$\log(1/q)\geq2048x+\log64=2048x+6\log2$. Therefore,
\[
p\leq\frac{768x+\log2}{2048x+6\log2}\leq\frac38.
\]
The denominator is positive, $768=(3/8)2048$, and
$\log2\leq(3/8)6\log2$, which proves the last inequality.
The same conclusion holds when $q=0$.
This bounds the probability of recommending arm $1$; that arm need
not be best on tables outside $\eventc$.

Return now to the full table-and-learner probability law, and write
\[
\ProErr_{\advPro}(\timeHorizon)
=\Pro_{\advPro}\left(\recomArm_{\timeHorizon}\neq\bestArm_{\gainVector}\right),
\]
with the fixed table tie-breaking convention above.
On $\eventc$, arm $1$ is uniquely best, and
$\Pro_{\advPro}(\recomArm_{\timeHorizon}=1)=p$ because this is a
transcript event. Equation~\ref{eq:goodEventBOB} therefore yields
\begin{equation}\label{eq:randomizedErrorBOB}
\begin{aligned}
\ProErr_{\advPro}(\timeHorizon)
&\geq\Pro_{\advPro}\left(\{\recomArm_{\timeHorizon}\neq1\}\cap\eventc\right)\\
&\geq\Pro_{\advPro}(\recomArm_{\timeHorizon}\neq1)
-\Pro_{\advPro}(\eventc^c)\\
&=1-p-\Pro_{\advPro}(\eventc^c)\\
&\geq1-\frac38-\frac{1}{16}\\
&=\frac{9}{16}\geq\frac18.
\end{aligned}
\end{equation}
No independence between the recommendation and $\eventc$ is needed.

Equation~\ref{eq:randomizedErrorBOB} gives the lower bound needed
for the extraction of a deterministic table:
\[
\ProErr_{\advPro}(\timeHorizon)\geq1/8.
\]
For a fixed table $\gainVector$, write
\[
\ProErr_{\gainVector}(\timeHorizon)
\triangleq\ProErr_{\advPro(\gainVector)}(\timeHorizon),
\]
where the probability is over the learner's internal randomness.
The random \advPro{} model is a distribution over fixed tables,
independent of that randomness, so
\[
\ProErr_{\advPro}(\timeHorizon)
=\Exp_{\gainVector\sim\advPro}
\left[\ProErr_{\gainVector}(\timeHorizon)\right].
\]
Suppose for contradiction that $\ProErr_{\gainVector}(\timeHorizon)<1/16$
for every $\gainVector\in\eventc$. Splitting this expectation using
indicators and applying Equation~\ref{eq:goodEventBOB} gives
\begin{align*}
\ProErr_{\advPro}(\timeHorizon)
&=\Exp_{\gainVector\sim\advPro}\left[
\ProErr_{\gainVector}(\timeHorizon)\indicator{\gainVector\in\eventc}\right]
+\Exp_{\gainVector\sim\advPro}\left[
\ProErr_{\gainVector}(\timeHorizon)\indicator{\gainVector\notin\eventc}\right]\\
&\leq\frac{1}{16}\Pro_{\advPro}(\eventc)
+\Pro_{\advPro}(\eventc^c)\\
&=\frac{1}{16}+\frac{15}{16}\Pro_{\advPro}(\eventc^c)\\
&\leq\frac{1}{16}+\frac{15}{256}\\
&=\frac{31}{256}<\frac18.
\end{align*}
This contradicts the stated input $\ProErr_{\advPro}(\timeHorizon)\geq1/8$.
Hence there exists a deterministic table $\gainVector\in\eventc$ with
$\ProErr_{\advPro(\gainVector)}(\timeHorizon)\geq1/16$.
The gap calculation above shows that this table belongs to $\Gaps_4$.
\end{proof}%

%% file: appUPPBOB.tex
% !TEX root = Main.tex
%
\section{Upper bound in best of both worlds} %Setting (Theorem~\ref{th:UpBOB}) }
\label{app:proofupBOB}
\paragraph{Stochastic setup.}
We work in the bounded i.i.d.\,stochastic setting of
Section~\ref{s:FormuBOB}, with arm laws $\nu_h$ on $[0,1]$ and
means $\mean_h$. For the analysis, relabel the arms so that
$\mean_1>\mean_2\geq\cdots\geq\mean_{\nArms}$, ordering equal
suboptimal means arbitrarily and consistently. Thus
$0<\gap_1=\gap_2\leq\cdots\leq\gap_{\nArms}\leq1$.
This does not affect the learner's rankings or tie resolution.
Fix any deterministic admissible schedule
$\propTimeVec\in\propTimeVecSpa$ and write
$\phaseTime_r=\timeHorizon\propTime_r$, so that
\[
\phaseTime_1=\phaseTime_2=\timeHorizon
\geq\phaseTime_3\geq\cdots\geq\phaseTime_{\nArms}\geq1,
\qquad \phaseTime_{\nArms+1}=0.
\]
All comparison times are integers. The schedule may depend on the
instance and horizon, but not on the realized observations.
Coincident comparison times are allowed, and sums over empty time
intervals are zero.

For $u\in[\timeHorizon]$, the rank $\tilde{\langle k\rangle}_u$
is the learner's ranking from $\estCumulGainVector_{\cdot,u-1}$.
At the unplayed time $\timeHorizon+1$, use a decreasing-score ranking
of $\estCumulGainVector_{\cdot,\timeHorizon}$ that places the actual
recommendation $\recomArm_{\timeHorizon}$ first among tied maximizers;
order other ties arbitrarily. This is only an analysis convention:
no additional pull is made, and neither the recommendation nor its
tie rule is changed.
Intuitively, after comparison time $\phaseTime_i$, we expect arms
whose means are within $\gap_i/2$ of the best mean to occupy the
first $i-1$ ranks. For $i\in[2:\nArms]$, define
\begin{equation}\label{eq:eventbob}
\eventbob_i
\triangleq
\left\{
\begin{array}{l}
\forall u\in\Integer\text{ with }\phaseTime_i<u\leq\timeHorizon+1,
\ \forall k\in\setArms:\\
\displaystyle
\mean_1-\mean_k<\frac{\gap_i}{2}
\implies\tilde{\langle k\rangle}_u<i
\end{array}
\right\}.
\end{equation}
Set $\eventbob_{\nArms+1}=\Omega$ and abbreviate
$E_i=\bigcap_{p=i+1}^{\nArms+1}\eventbob_p$.
Since only arm $1$ satisfies the eligibility condition for
$\eventbob_2$, and $\phaseTime_2=\timeHorizon$, we have
$\eventbob_2=\{\recomArm_{\timeHorizon}=1\}$.

Let $\mathcal F_{s-1}$ be the information available immediately
before drawing $\pulledArm_s$, including past actions and observed
rewards and the current resolved ranking. These information sets
form a filtration by including any next-round tie resolution in
$\mathcal F_s$, without revealing future action-sampling coins or
current rewards before sampling. Thus the ranks and the conditional
probabilities $\learnerDist_{h,s}$ are predictable. The fresh action
draw is independent of the current rewards given the past, and the
reward observed when arm $h$ is selected has conditional law $\nu_h$.
For each arm $h$, put
\[
X_{h,s}=\estGainVector_{h,s}-\mean_h,\qquad
M_h(t)=\sum_{s=1}^t X_{h,s}
=\estCumulGainVector_{h,t}-t\mean_h,\qquad
V_h(t)=\sum_{s=1}^t\Exp[X_{h,s}^2\mid\mathcal F_{s-1}].
\]
The importance-weighted estimator and $0\leq Y\leq1$ give
\begin{align*}
\Exp[X_{h,s}\mid\mathcal F_{s-1}]&=0,\\
\Exp[X_{h,s}^2\mid\mathcal F_{s-1}]
&=\frac{\Exp_{Y\sim\nu_h}[Y^2]}{\learnerDist_{h,s}}-\mean_h^2
\leq\frac1{\learnerDist_{h,s}},
\qquad
|X_{h,s}|\leq\frac1{\learnerDist_{h,s}}
\leq\nArms\bar\log\nArms.
\end{align*}
In particular, $M_h$ is a martingale and $V_h$ is its predictable
variance process; no Bernoulli assumption is needed.

\paragraph{Concentration on good-ranking events.}
We use the following fixed-time Bernstein--Freedman inequality
\citep{Freedman75OT}: for martingale differences bounded in absolute
value by $B>0$, their sum $M_t$ and predictable variance $V_t$ satisfy,
for deterministic $z,v>0$,
\[
\Pro(M_t\geq z,\ V_t\leq v)
\leq\exp\left(-\frac{z^2}{2(v+Bz/3)}\right).
\]
The same bound holds for $-M_t$. Indeed, for $0<\lambda<3/B$,
\[
\Exp[e^{\lambda X_s}\mid\mathcal F_{s-1}]
\leq\exp\left(
\frac{\lambda^2\Exp[X_s^2\mid\mathcal F_{s-1}]}
{2(1-\lambda B/3)}\right),
\]
so $\exp\{\lambda M_t-\lambda^2 V_t/[2(1-\lambda B/3)]\}$
has expectation at most one. Markov's inequality with
$\lambda=z/(v+Bz/3)<3/B$ yields the stated bound.
If a good-ranking event $E$ implies $V_t\leq v$, then
\[
\{M_t\geq z\}\cap E\subseteq\{M_t\geq z,\ V_t\leq v\},
\]
and hence the same bound applies to its probability. The event $E$
may involve observations after time $t$ or have probability zero:
we never condition the stochastic law on $E$.

\begin{lemma}\label{l:errTwoArms}
For the fixed schedule $\propTimeVec$, let $i\in[2:\nArms]$,
$j\in[i:\nArms]$, and let $k\in\setArms$ satisfy
$\mean_1-\mean_k<\gap_i/2$. For every integer $t$ with
$\phaseTime_i\leq t\leq\timeHorizon$, playing~\Pone{},
\[
\Pro\left(
\{\estCumulGainVector_{k,t}\leq\estCumulGainVector_{j,t}\}
\cap E_i\right)
\leq2\exp\left(-\frac{\timeHorizon}
{128\complexityProOne(\propTimeVec)}\right).
\]
\end{lemma}
\begin{proof}
First, the assumptions give $k<i\leq j$ and
\[
\mean_k-\mean_j
=\gap_j-(\mean_1-\mean_k)
>\gap_j-\frac{\gap_i}{2}
\geq\frac{\gap_j}{2}>0.
\]
Splitting a comparison error between the lower tail of the better
arm and the upper tail of the worse arm, we obtain
\begin{equation}\label{eq13}
\begin{split}
\Pro\left(
\{\estCumulGainVector_{k,t}\leq\estCumulGainVector_{j,t}\}
\cap E_i\right)
\leq{}&\Pro\left(E_i\cap\{-M_k(t)>t\gap_j/4\}\right)\\
&+\Pro\left(E_i\cap\{M_j(t)>t\gap_j/4\}\right).
\end{split}
\end{equation}
Define the doubled-gap rank
\[
j^+=\max\{r\in[j:\nArms]:\gap_r\leq2\gap_j\}.
\]
The set contains $j$, so $j\leq j^+\leq\nArms$ and
$\gap_{j^+}\leq2\gap_j$. If $j^+<\nArms$, then
$\gap_j<\gap_{j^++1}/2$. When $j^+=\nArms$, the universal rank
bound $\nArms$ will suffice. Thus tied gaps and exact doubled-gap
equality are included without requiring any additional arm.

We bound both terms in Equation~\ref{eq13} using the same
virtual-phase variance calculation. Write
\[
S_r=\sum_{\ell=r}^{\nArms}\ell(\phaseTime_\ell-\phaseTime_{\ell+1}),
\]
and take $(h,r)=(k,i)$ for the better arm and $(h,r)=(j,j^+)$
for the worse arm. In either case $r\geq i$ and
$t\geq\phaseTime_i\geq\phaseTime_r$.
If $r<\nArms$, the corresponding mean-gap condition is
$\mean_1-\mean_h<\gap_{r+1}/2$: for $h=k$ this follows from
$\gap_i\leq\gap_{i+1}$, and for $h=j$ it follows from the
definition of $j^+$.
On $E_i$, in each block
$\phaseTime_{\ell+1}<s\leq\phaseTime_\ell$, $r\leq\ell\leq\nArms$,
event $\eventbob_{\ell+1}$ gives
$\tilde{\langle h\rangle}_s\leq\ell$ when $\ell<\nArms$;
the bound is automatic when $\ell=\nArms$.
After $\phaseTime_r$, event $\eventbob_{r+1}$ similarly gives rank
at most $r$, with the universal bound used when $r=\nArms$.
All events used have indices larger than $i$ and belong to $E_i$.
Thus good earlier rankings increase sampling probabilities and
reduce the conditional variance. Summing the per-round bounds over
the disjoint phase blocks and the continuation interval yields,
on $E_i$,
\begin{align*}
V_h(t)
&\leq\bar\log\nArms\left[
\sum_{\ell=r}^{\nArms}\ell(\phaseTime_\ell-\phaseTime_{\ell+1})
+r(t-\phaseTime_r)\right]\\
&=\bar\log\nArms\,[S_r+r(t-\phaseTime_r)].
\end{align*}
Empty blocks contribute zero, including when comparison times coincide.

Apply the joint-event concentration bound with
$z=t\gap_r/8$, $B=\nArms\bar\log\nArms$, and the positive
variance allowance $v=\bar\log\nArms\,[S_r+r(t-\phaseTime_r)]$.
For either sign separately, it gives
\begin{align*}
\Pro\left(E_i\cap\{\pm M_h(t)\geq t\gap_r/8\}\right)
&\leq\exp\left(
-\frac{(t\gap_r/8)^2}
{2\bar\log\nArms\,[S_r+r(t-\phaseTime_r)]
+(\nArms\bar\log\nArms)t\gap_r/12}\right)\\
&=\exp\left(
-\frac{t^2\gap_r^2}
{128\bar\log\nArms\,[S_r+r(t-\phaseTime_r)
+(\nArms/24)t\gap_r]}\right).
\end{align*}
Here the factor $2$ multiplies the entire variance allowance;
$128=8^2\cdot2$, and the coefficient $\nArms/24$ comes from $Bz/3$.

The identity
$S_r=r\phaseTime_r+\sum_{\ell=r+1}^{\nArms}\phaseTime_\ell
\geq r\phaseTime_r$ implies, for $t\geq\phaseTime_r$,
\[
S_r+r(t-\phaseTime_r)\leq\frac{t}{\phaseTime_r}S_r.
\]
Consequently,
\begin{align*}
\frac{t^2\gap_r^2}
{128\bar\log\nArms\,[S_r+r(t-\phaseTime_r)
+(\nArms/24)t\gap_r]}
&\geq
\frac{\phaseTime_r^2\gap_r^2}
{128\bar\log\nArms\,[S_r+(\nArms/24)\phaseTime_r\gap_r]}\\
&\geq\frac{\timeHorizon}{128\complexityProOne(\propTimeVec)}.
\end{align*}
For the last step, the existing complexity definition, after
substituting $\phaseTime_r=\timeHorizon\propTime_r$, reads
\[
\complexityProOne(\propTimeVec)
=\timeHorizon\bar\log\nArms\,
\max_{r\in[\nArms]}
\frac{S_r+(\nArms/24)\phaseTime_r\gap_r}
{\phaseTime_r^2\gap_r^2}.
\]
Finally, $t\gap_j/4\geq t\gap_i/8$ for the better-arm tail,
and $t\gap_j/4\geq t\gap_{j^+}/8$ for the worse-arm tail.
The two one-sided joint bounds therefore control the two terms
in Equation~\ref{eq13}; summing them proves the claim for the
fixed schedule $\propTimeVec$.
\end{proof}
\begin{lemma}\label{l:errPOGerrPO}
For the fixed admissible schedule $\propTimeVec$, playing~\Pone{}
in the stochastic setting above, we have for every $i\in[2:\nArms]$,
\[
\Pro(\eventbob_i^c\cap E_i)
\leq2\nArms^2\timeHorizon
\exp\left(-\frac{\timeHorizon}
{128\complexityProOne(\propTimeVec)}\right).
\]
\end{lemma}
\begin{proof}
If $\eventbob_i$ fails, there are an integer $u$ with
$\phaseTime_i<u\leq\timeHorizon+1$ and an arm $k$ satisfying
$\mean_1-\mean_k<\gap_i/2$ but $\tilde{\langle k\rangle}_u\geq i$.
Necessarily $k<i$: otherwise $k\geq i\geq2$ and
$\mean_1-\mean_k=\gap_k\geq\gap_i$, contradicting eligibility.
At least $i-1$ arms precede $k$ in the estimated ranking, and
only $i-2$ other arms can have true labels in $[1:i-1]$.
Thus some arm $j\in[i:\nArms]$ precedes $k$. Since the ranking
uses the estimates through round $u-1$, this implies
$\estCumulGainVector_{k,u-1}\leq\estCumulGainVector_{j,u-1}$,
including when the estimated scores tie.

Since $\phaseTime_i\leq u-1\leq\timeHorizon$,
Lemma~\ref{l:errTwoArms} at $t=u-1$ gives
\[
\Pro\left(
\{\estCumulGainVector_{k,u-1}\leq\estCumulGainVector_{j,u-1}\}
\cap E_i\right)
\leq2\exp\left(-\frac{\timeHorizon}
{128\complexityProOne(\propTimeVec)}\right).
\]
There are at most $i-1$ eligible arms $k$, $\nArms-i+1$
competitors $j$, and $\timeHorizon-\phaseTime_i+1\leq\timeHorizon$
ranking times $u$, since $\phaseTime_i\geq1$.
A union bound gives
\begin{align*}
\Pro(\eventbob_i^c\cap E_i)
&\leq2(i-1)(\nArms-i+1)(\timeHorizon-\phaseTime_i+1)
\exp\left(-\frac{\timeHorizon}
{128\complexityProOne(\propTimeVec)}\right)\\
&\leq2\nArms^2\timeHorizon
\exp\left(-\frac{\timeHorizon}
{128\complexityProOne(\propTimeVec)}\right).
\end{align*}
\end{proof}%
	\setcounter{scratchcounter}{\value{theorem}}\setcounter{theorem}{\the\numexpr\getrefnumber{th:UpBOB}-1}
\lightres*
\setcounter{theorem}{\the\numexpr\value{scratchcounter}}
\begin{proof} %[Theorem~\ref{th:UpBOB}] 
We consider two cases separately.
	%Let us consider the $a^*_i$ that maximizes the proof quantity
	%
	%
	\paragraph{The i.i.d.~stochastic case}
For the fixed schedule $\propTimeVec$, the terminal-ranking
convention gives
\[
\ProErr_{\stoPro}(\timeHorizon)
=\Pro(\recomArm_{\timeHorizon}\neq1)
=\Pro(\eventbob_2^c).
\]
If $\eventbob_2$ fails, take the largest failed index $i\in[2:\nArms]$.
Every event with a larger index holds, since
$\eventbob_{\nArms+1}$ is certain.
Hence
\[
\eventbob_2^c\subseteq
\bigcup_{i=2}^{\nArms}(\eventbob_i^c\cap E_i).
\]
A union bound and Lemma~\ref{l:errPOGerrPO} give
\begin{align*}
\ProErr_{\stoPro}(\timeHorizon)
&\leq\sum_{i=2}^{\nArms}\Pro(\eventbob_i^c\cap E_i)\\
&\leq2(\nArms-1)\nArms^2\timeHorizon
\exp\left(-\frac{\timeHorizon}
{128\complexityProOne(\propTimeVec)}\right)\\
&\leq2\nArms^3\timeHorizon
\exp\left(-\frac{\timeHorizon}
{128\complexityProOne(\propTimeVec)}\right).
\end{align*}
This holds for every admissible schedule; its virtual comparison
times do not affect \Pone{} or the final error probability. The admissible set is finite and nonempty:
its entries lie in $\{1/\timeHorizon,\ldots,1\}$ and the all-ones
schedule is allowed. Positive comparison times and gaps make
$\complexityProOne(\propTimeVec)$ finite and positive.
Choosing a minimizing schedule gives
\[
\ProErr_{\stoPro}(\timeHorizon)
\leq2\nArms^3\timeHorizon
\exp\left(-\frac{\timeHorizon}{128\complexityProOne}\right).
\]

	\paragraph{The adversarial case}
Fix a deterministic oblivious table
$\gainVector\in[0,1]^{\nArms\times\timeHorizon}$ and relabel its
unique best arm as $1$ for analysis. Only the learner is random.
It suffices to use the sampling probability guaranteed to every arm;
no favorable ranking event is needed. With the information before
each action defined as above, the adaptive probabilities satisfy
\[
\learnerDist_{k,t}
=\Pro(\pulledArm_t=k\mid\mathcal F_{t-1})
=\frac1{\tilde{\langle k\rangle}_t\bar\log\nArms}
\geq\frac1{\nArms\bar\log\nArms}
\]
on every history. Put
\[
\cumulGain_{k,t}=\sum_{s=1}^t\gainVector_{k,s},\qquad
Y_{k,t}=\estGainVector_{k,t}-\gainVector_{k,t},\qquad
M_k^{\gainVector}(t)=\sum_{s=1}^tY_{k,s}
=\estCumulGainVector_{k,t}-\cumulGain_{k,t}.
\]
Since the table is deterministic, centering at $\gainVector_{k,t}$
is appropriate. Conditionally, the sampling indicator is Bernoulli
with parameter $\learnerDist_{k,t}$, so
\begin{align*}
\Exp[Y_{k,t}\mid\mathcal F_{t-1}]&=0,\\
\Exp[Y_{k,t}^2\mid\mathcal F_{t-1}]
&=\gainVector_{k,t}^2\left(\frac1{\learnerDist_{k,t}}-1\right)
\leq\frac1{\learnerDist_{k,t}}\leq\nArms\bar\log\nArms,\\
|Y_{k,t}|&\leq\frac1{\learnerDist_{k,t}}\leq\nArms\bar\log\nArms.
\end{align*}
Thus $M_k^{\gainVector}$ is a martingale with predictable variance
\[
V_k^{\gainVector}(\timeHorizon)
=\sum_{t=1}^{\timeHorizon}\Exp[Y_{k,t}^2\mid\mathcal F_{t-1}]
\leq\timeHorizon\nArms\bar\log\nArms
\quad\text{almost surely}.
\]

Recall that
$\timeHorizon\gap_j^{\gainVector}
=\cumulGain_{1,\timeHorizon}-\cumulGain_{j,\timeHorizon}$ for $j\neq1$
and $\gap_1^{\gainVector}=\min_{j\neq1}\gap_j^{\gainVector}$.
Hence $0<\gap_1^{\gainVector}\leq\gap_j^{\gainVector}\leq1$ and
$\complexityUnif(\gainVector)=\nArms/(\gap_1^{\gainVector})^2$.
The error event satisfies
\[
\{\recomArm_{\timeHorizon}\neq1\}
\subseteq
\{M_1^{\gainVector}(\timeHorizon)\leq-\timeHorizon\gap_1^{\gainVector}/2\}
\ \cup\!
\bigcup_{j=2}^{\nArms}
\{M_j^{\gainVector}(\timeHorizon)\geq\timeHorizon\gap_j^{\gainVector}/2\}.
\]
Indeed, outside the right-hand side, for every $j\neq1$,
\begin{align*}
\estCumulGainVector_{1,\timeHorizon}
&>\cumulGain_{1,\timeHorizon}-\timeHorizon\gap_1^{\gainVector}/2
\geq\cumulGain_{1,\timeHorizon}-\timeHorizon\gap_j^{\gainVector}/2\\
&=\cumulGain_{j,\timeHorizon}+\timeHorizon\gap_j^{\gainVector}/2
>\estCumulGainVector_{j,\timeHorizon}.
\end{align*}
Thus arm $1$ is the unique estimated maximizer, regardless of tie rules.

Apply the Bernstein--Freedman inequality above with
$B=\nArms\bar\log\nArms$, $v=\timeHorizon\nArms\bar\log\nArms$,
and $z=\timeHorizon\gap_k^{\gainVector}/2$.
The variance bound holds almost surely, so for either sign separately,
\begin{align*}
\Pro\left(\pm M_k^{\gainVector}(\timeHorizon)
\geq\timeHorizon\gap_k^{\gainVector}/2\right)
&\leq\exp\left(
-\frac{(\timeHorizon\gap_k^{\gainVector}/2)^2}
{2[\timeHorizon\nArms\bar\log\nArms
+(\nArms\bar\log\nArms)\timeHorizon\gap_k^{\gainVector}/6]}\right)\\
&=\exp\left(-\frac{\timeHorizon(\gap_k^{\gainVector})^2}
{8\nArms\bar\log\nArms(1+\gap_k^{\gainVector}/6)}\right)\\
&\leq\exp\left(-\frac{3\timeHorizon(\gap_k^{\gainVector})^2}
{28\nArms\bar\log\nArms}\right),
\end{align*}
since $1+\gap_k^{\gainVector}/6\leq7/6$.
There are $\nArms$ one-sided error events: one lower tail for arm $1$
and one upper tail for each other arm. A union bound gives
\begin{align*}
\ProErr_{\advPro(\gainVector)}(\timeHorizon)
&\leq\nArms\exp\left(-\frac{3\timeHorizon(\gap_1^{\gainVector})^2}
{28\nArms\bar\log\nArms}\right)\\
&=\nArms\exp\left(-\frac{3\timeHorizon}
{28\bar\log\nArms\complexityUnif(\gainVector)}\right)\\
&\leq\nArms\exp\left(-\frac{3\timeHorizon}
{40\bar\log\nArms\complexityUnif(\gainVector)}\right),
\end{align*}
as claimed.
\end{proof}
\section{On the complexities of $\complexityProOne$, $\complexitySR$, and 
	$\complexityBoth$}  %Proof of  Corollary~\ref{coco}}
\label{s:moreOnRem}
%\todomi{right reference to the remark in the title?}
%\todom{recheck the derivations below in detail}
%
%
 %\begin{remark} 
 We now show that for $\timeHorizon\geq\gap_{(\nArms)}/\gap_{(1)}$,
	$\complexityProOne = \cO\left(\complexityBoth\log^2\nArms\right)$. 
	This demonstrates that \Pone{} achieves the best 
	that can be wished for in the two worlds, up to log
	 factors.
	The extra $\log\nArms$ from the $\log^2\nArms$ is not always
	present  and we report an 
	even more detailed discussion on the three different regimes of the
	gaps used in 
	Remark~\ref{rem:LOWbobCOMP}  at the end of the section.
	%\todov{its not he right number for the corrolary numbering!}	
	\setcounter{scratchcounter}{\value{theorem}}\setcounter{theorem}{\the\numexpr\getrefnumber{coco}-1}
	\cocores*
	\setcounter{theorem}{\the\numexpr\value{scratchcounter}}
\begin{proof}%[Proof of Corollary~\ref{coco}]
Recall that $\gap_{(1)}=\gap_{(2)}$, and let
$Q=\max_{1\leq j\leq\nArms}j/\gap_{(j)}$, so that
$\complexityBoth=Q/\gap_{(1)}$. Since the rewards are in $[0,1]$,
$0<\gap_{(1)}\leq\gap_{(r)}\leq1$, $Q\geq2/\gap_{(1)}$, and $Q\geq\nArms$.
Set the ideal fractions $\alpha_r=\gap_{(1)}/\gap_{(r)}$ and the rounded fractions
$\propTime_r=\lceil\timeHorizon\alpha_r\rceil/\timeHorizon$ for
$r\in\setArms$, with $\alpha_{\nArms+1}=\propTime_{\nArms+1}=0$.
Ceilings preserve weak ordering, so
$1=\propTime_1=\propTime_2\geq\cdots\geq\propTime_{\nArms}>0$ and
$1\leq\timeHorizon\propTime_r\leq\timeHorizon$, with
$\timeHorizon\propTime_r\in\Integer$.
Thus $\propTimeVec\in\propTimeVecSpa$, including when comparison rounds coincide.

For an ordered positive vector $b$ with $b_1=b_2=1$ and $b_{\nArms+1}=0$, let
\[
\Phi(b)=\bar\log\nArms\,
\max_{1\leq r\leq\nArms}
\frac{\sum_{j=r}^{\nArms}j(b_j-b_{j+1})
+(\nArms/24)b_r\gap_{(r)}}{b_r^2\gap_{(r)}^2}.
\]
This is the algebraic expression defining $\complexityProOne(\propTimeVec)$;
we use it for nonintegral comparison times only as an algebraic comparison.
The identity
\[
\sum_{j=r}^{\nArms}j(b_j-b_{j+1})
=r b_r+\sum_{j=r+1}^{\nArms}b_j
\]
makes all coefficients in the numerator nonnegative.

Since $\alpha_r\gap_{(r)}=\gap_{(1)}$, the second numerator term divided by
the denominator is $\nArms/(24\gap_{(1)})\leq\complexityBoth/24$.
For the first term, when $r\geq2$, the identity gives
\begin{align*}
\frac{1}{\gap_{(1)}}\sum_{j=r}^{\nArms}j(\alpha_j-\alpha_{j+1})
&=\frac{r}{\gap_{(r)}}+\sum_{j=r+1}^{\nArms}\frac{1}{\gap_{(j)}}\\
&\leq Q\left(1+\sum_{j=r+1}^{\nArms}\frac{1}{j}\right)
\leq Q\left(\bar\log\nArms-\frac12\right).
\end{align*}
For $r=1$, the weighted sum equals that for $r=2$ because $\alpha_1=\alpha_2$.
The tail sums are empty when $\nArms=2$, so this case is also covered.
It follows that
\[
\Phi(\alpha)\leq\complexityBoth\bar\log\nArms
\left(\bar\log\nArms-\frac12+\frac{1}{24}\right)
\leq\complexityBoth(\bar\log\nArms)^2.
\]

If $\timeHorizon\geq\gap_{(\nArms)}/\gap_{(1)}$, then
$\timeHorizon\alpha_r\geq1$ for every rank, and
\[
\alpha_r\leq\propTime_r<\alpha_r+\frac{1}{\timeHorizon}\leq2\alpha_r.
\]
Moving each comparison by less than one round therefore makes the telescoped
numerator at most twice its value at $\alpha$, while the denominator is
no smaller. Using the admissible rounded schedule, we obtain
\[
\complexityProOne\leq\Phi(\propTimeVec)\leq2\Phi(\alpha)
\leq2\complexityBoth(\bar\log\nArms)^2
=\cO\left(\complexityBoth\log^2\nArms\right),
\]
where the last comparison uses $\bar\log\nArms\leq1+\log\nArms$.
Theorem~\ref{th:UpBOB} then gives
\[
\ProErr_{\stoPro}(\timeHorizon)
\leq2\nArms^3\timeHorizon\exp\left(-\frac{\timeHorizon}{128\complexityProOne}\right)
\leq2\nArms^3\timeHorizon\exp\left(
-\frac{\timeHorizon}{256\complexityBoth(\bar\log\nArms)^2}\right).
\]
For $1\leq\timeHorizon<\gap_{(\nArms)}/\gap_{(1)}$,
$\complexityBoth\geq2/\gap_{(1)}^2$ gives
\[
\frac{\timeHorizon}{256\complexityBoth(\bar\log\nArms)^2}
<\frac{\gap_{(1)}\gap_{(\nArms)}}{512(\bar\log\nArms)^2}\leq\frac{1}{512}.
\]
Since $2\nArms^3\timeHorizon\geq16$ and $\exp(-x)\geq1-x$, the displayed
error bound exceeds one and follows from $\ProErr_{\stoPro}(\timeHorizon)\leq1$.
This proves the error bound for all positive integer horizons, without
asserting the complexity comparison for these smaller budgets.
\end{proof}

\subsection{Relation between 
	$\complexityProOne$, $\complexitySR$, and 
	$\complexityBoth$ for different regimes of the 
	gaps.}
We now study the relation between 
$\complexityProOne$, $\complexitySR$, and 
$\complexityBoth$ for different regimes of the 
gaps.
We use the same examples as the ones used in 
Remark~\ref{rem:LOWbobCOMP}.
We first choose ideal fractions $b$,
with $1=b_1=b_2\geq\cdots\geq b_{\nArms}>b_{\nArms+1}=0$,
and use the algebraic expression $\Phi$ from the preceding corollary proof.
Set $\propTime_r=\lceil\timeHorizon b_r\rceil/\timeHorizon$ for $r\in\setArms$
and $\propTime_{\nArms+1}=0$. Ceilings preserve weak ordering and give
admissible integer times, even when some coincide; only the analysis
changes, while \Pone{} and its horizon are unchanged.
If $\timeHorizon b_{\nArms}\geq1$, the positive telescoped numerator gives
$\complexityProOne\leq\Phi(\propTimeVec)\leq2\Phi(b)$, as proved above.
These schedules compare larger-gap arms earlier, keeping the best two
arms at the final comparison time.
In the stated horizon ranges, the comparisons are at worst
$\cO(\complexityBoth\log^2\nArms)$, and
$\cO(\complexityBoth\log\nArms)$ in the flat and square-root regimes.

\paragraph{\raisebox{.04cm}{\textcolor{bull}{$\blacktriangleright$}}~Flat regime}  All the gaps are equal, that is,
$k\in\setArmsmo,$ we have that $\gap_k=\gap_{1}$. 

\smallskip
\noindent
We choose $b_r=1$, $\forall r\in\setArms$.
All comparisons occur at the final round, and no rounding is needed.
The weighted sum equals $\nArms$ at every rank, and
$\complexitySR=\complexityBoth=\nArms/\gap_{(1)}^2$. Thus the two terms give
\[
\complexityProOne\leq\Phi(b)
=\bar\log\nArms\left(\frac{\nArms}{\gap_{(1)}^2}+\frac{\nArms}{24\gap_{(1)}}\right)
\leq\frac{25}{24}\complexityBoth\bar\log\nArms,
\]
since $\gap_{(1)}\leq1$. The single-logarithmic comparison holds for every
positive integer horizon.

\paragraph{\raisebox{.04cm}{\textcolor{bull}{$\blacktriangleright$}}~Super-linear gaps}  Since $2\in\argmin_{r\in\setArms}(\gap_{(r)}/r)$,
$\gap_{(r)}\geq r\gap_{(1)}/2$ for $r\geq2$ and
$\complexitySR=\complexityBoth=2/\gap_{(1)}^2$.
Also, $\nArms\gap_{(1)}\leq2$ since $\gap_{(\nArms)}\leq1$.

\smallskip
\noindent
Let $b_1=1$ and $b_r=2/r$, $r\in[2:\nArms]$.
This normalizes the reciprocal-rank schedule, with $b_2=1$ and
$b_r\gap_{(r)}\geq\gap_{(1)}$. The large gaps permit early identification
of bad arms.
The second complexity term satisfies
\[
\frac{\nArms}{24b_r\gap_{(r)}}
\leq\frac{\nArms}{24\gap_{(1)}}\leq\frac{\complexityBoth}{24}.
\]
For the first term, telescoping gives, for $r\geq2$,
\[
\sum_{j=r}^{\nArms}j(b_j-b_{j+1})
=2+2\sum_{j=r+1}^{\nArms}\frac1j
\leq2\left(\bar\log\nArms-\frac12\right).
\]
The rank-$1$ sum equals the rank-$2$ sum since $b_1=b_2$.
Using $b_r\gap_{(r)}\geq\gap_{(1)}$ gives
\[
\Phi(b)\leq\complexityBoth\bar\log\nArms
\left(\bar\log\nArms-\frac12+\frac1{24}\right)
\leq\complexityBoth(\bar\log\nArms)^2.
\]
Since $b_{\nArms}=2/\nArms$, rounding gives
\[
\complexityProOne\leq2\complexityBoth(\bar\log\nArms)^2
=\cO(\complexityBoth\log^2\nArms),\qquad\timeHorizon\geq\nArms/2.
\]

\paragraph{\raisebox{.04cm}{\textcolor{bull}{$\blacktriangleright$}}~Square-root gaps} We have that $2\in\argmin_{r\in\setArms}(\gap^2_{(r)}/r)$,
$\sqrt{\nArms/2}\gap_{(1)}=\gap_{(\nArms)}$, and also that
$\sqrt{r/2}\gap_{(1)}\geq\gap_{(r)}$ for each intermediate rank $3\leq r<\nArms$.
The minimizing condition gives the reverse inequality, hence
$\gap_{(r)}=\gap_{(1)}\sqrt{r/2}$ for every $r\geq2$.
Thus $\complexitySR=2/\gap_{(1)}^2$ and
$\complexityBoth=\sqrt{2\nArms}/\gap_{(1)}^2=\sqrt{\nArms/2}\complexitySR$.

\smallskip
\noindent
Let $b_1=1$ and $b_r=\sqrt{2/r}$, $r\in[2:\nArms]$.
This normalizes the square-root schedule, with $b_2=1$ and
$b_r\gap_{(r)}=\gap_{(1)}$.
For $r\geq2$, the integral comparison
$\sum_{j=r+1}^{\nArms}j^{-1/2}\leq2(\sqrt{\nArms}-\sqrt r)$ gives
\[
\sum_{j=r}^{\nArms}j(b_j-b_{j+1})
=\sqrt{2r}+\sum_{j=r+1}^{\nArms}\sqrt{2/j}
\leq2\sqrt{2\nArms}.
\]
The rank-$1$ sum equals the rank-$2$ sum. The first complexity term
is at most $2\complexityBoth$, and the second is
$\nArms/(24\gap_{(1)})\leq\complexityBoth/24$, since
$\gap_{(\nArms)}=\gap_{(1)}\sqrt{\nArms/2}\leq1$.
Thus $\Phi(b)\leq(49/24)\complexityBoth\bar\log\nArms$.
Since $b_{\nArms}=\sqrt{2/\nArms}$, rounding gives
\[
\complexityProOne\leq2\Phi(b)\leq\frac{49}{12}\complexityBoth\bar\log\nArms
=\cO(\complexityBoth\log\nArms),\qquad
\timeHorizon\geq\sqrt{\nArms/2}.
\]
For $\nArms=2$, the tail sums are empty and both normalized schedules
reduce to the flat case, without rounding loss.
The larger $\complexityBoth$ reflects the extra variance incurred before
focusing on the best arms, as discussed in Remark~\ref{rem:LOWbobCOMP}.
For shorter horizons, these complexity bounds are not asserted.
In either nonflat regime the threshold is at most $1/\gap_{(1)}$ and
$\complexityBoth\geq2/\gap_{(1)}^2$, so $\timeHorizon/\complexityBoth<1/2$.
The corresponding derived probability bounds from Theorem~\ref{th:UpBOB} exceed one:
the exponent has magnitude less than $1/256$ and
$2\nArms^3\timeHorizon\geq16$. They follow from
$\ProErr_{\stoPro}(\timeHorizon)\leq1$, as in Corollary~\ref{coco}'s proof.
%\todomi{gramarly it all in the end}

%% file: Main.bbl
\begin{thebibliography}{28}
\providecommand{\natexlab}[1]{#1}
\providecommand{\url}[1]{\texttt{#1}}
\expandafter\ifx\csname urlstyle\endcsname\relax
  \providecommand{\doi}[1]{doi: #1}\else
  \providecommand{\doi}{doi: \begingroup \urlstyle{rm}\Url}\fi

\bibitem[Allesiardo and F{\'e}raud(2017)]{Allesiardo2017SL}
Robin Allesiardo and Rapha{\"e}l F{\'e}raud.
\newblock \href{https://ieeexplore.ieee.org/document/7965962/}{Selection of
  learning experts}.
\newblock In \emph{International Joint Conference on Neural Networks}, 2017.

\bibitem[Allesiardo et~al.(2017)Allesiardo, F{\'e}raud, and
  Maillard]{Allesiardo17NS}
Robin Allesiardo, Rapha{\"e}l F{\'e}raud, and Odalric-Ambrym Maillard.
\newblock \href{https://doi.org/10.1007/s41060-017-0050-5}{The non-stationary
  stochastic multi-armed bandit problem}.
\newblock \emph{International Journal of Data Science and Analytics}, 2017.

\bibitem[Altschuler et~al.(2018)Altschuler, Brunel, and Malek]{Altschuler18BA}
Jason Altschuler, Victor-Emmanuel Brunel, and Alan Malek.
\newblock \href{https://arxiv.org/pdf/1802.09514.pdf}{Best-arm identification
  for contaminated bandits}.
\newblock \emph{arXiv preprint {\normalfont\texttt{arXiv:1802.09514}}}, 2018.

\bibitem[Audibert et~al.(2010)Audibert, Bubeck, and Munos]{Audibert10BA}
Jean-Yves Audibert, S{\'e}bastien Bubeck, and R{\'e}mi Munos.
\newblock
  \href{http://www.learningtheory.org/colt2010/papers/59Audibert.pdf}{Best-arm
  identification in multi-armed bandits}.
\newblock In \emph{Conference on Learning Theory}, 2010.

\bibitem[Auer and Chiang(2016)]{Auer16AA}
Peter Auer and Chao-Kai Chiang.
\newblock \href{https://arxiv.org/pdf/1605.08722.pdf}{An algorithm with nearly
  optimal pseudo-regret for both stochastic and adversarial bandits}.
\newblock In \emph{Conference on Learning Theory and arXiv preprint
  {\normalfont \texttt{arXiv:1605.08722}}}, 2016.

\bibitem[Auer et~al.(2002)Auer, Cesa-Bianchi, Freund, and Schapire]{Auer02NM}
Peter Auer, Nicol{\`o} Cesa-Bianchi, Yoav Freund, and Robert~E. Schapire.
\newblock \href{http://rob.schapire.net/papers/AuerCeFrSc01.pdf}{The
  nonstochastic multi-armed bandit problem}.
\newblock \emph{SIAM Journal on Computing}, 32\penalty0 (1), 2002.

\bibitem[Bubeck et~al.(2013)Bubeck, Wang, and Viswanathan]{Bubeck12MI}
S{\'e}bastian Bubeck, Tengyao Wang, and Nitin Viswanathan.
\newblock \href{http://proceedings.mlr.press/v28/bubeck13.pdf}{Multiple
  identifications in multi-armed bandits}.
\newblock In \emph{International Conference on Machine Learning}, 2013.

\bibitem[Bubeck and Cesa-Bianchi(2012)]{Bubeck12RA}
S{\'e}bastien Bubeck and Nicol{\`o} Cesa-Bianchi.
\newblock \href{https://arxiv.org/pdf/1204.5721.pdf}{Regret analysis of
  stochastic and nonstochastic multi-armed bandit problems}.
\newblock \emph{Foundations and Trends in Machine Learning}, 5\penalty0 (1),
  2012.

\bibitem[Bubeck and Slivkins(2012)]{Bubeck12BB}
S{\'e}bastien Bubeck and Aleksandrs Slivkins.
\newblock \href{http://proceedings.mlr.press/v23/bubeck12b/bubeck12b.pdf}{The
  best of both worlds: stochastic and adversarial bandits}.
\newblock In \emph{Conference on Learning Theory}, 2012.

\bibitem[Bubeck et~al.(2009)Bubeck, Munos, and Stoltz]{Bubeck09PE}
S{\'e}bastien Bubeck, R{\'e}mi Munos, and Gilles Stoltz.
\newblock \href{https://doi.org/10.1007/978-3-642-04414-4_7}{Pure exploration
  in multi-armed bandit problems}.
\newblock In \emph{Algorithmic Learning Theory}, 2009.

\bibitem[Carpentier and Locatelli(2016)]{Carpentier16TB}
Alexandra Carpentier and Andrea Locatelli.
\newblock \href{http://proceedings.mlr.press/v49/carpentier16.pdf}{Tight
  (lower) bounds for the fixed budget best-arm identification bandit problem}.
\newblock In \emph{Conference on Learning Theory}, 2016.

\bibitem[Carpentier and Valko(2014)]{carpentier2014extreme}
Alexandra Carpentier and Michal Valko.
\newblock \href{https://papers.nips.cc/paper/5379-extreme-bandits.pdf}{Extreme
  bandits}.
\newblock In \emph{Neural Information Processing Systems}, 2014.

\bibitem[Even-Dar et~al.(2006)Even-Dar, Mannor, and Mansour]{Even-Dar06AE}
Eyal Even-Dar, Shie Mannor, and Yishay Mansour.
\newblock
  \href{http://jmlr.csail.mit.edu/papers/volume7/evendar06a/evendar06a.pdf}{Action
  elimination and stopping conditions for the multi-armed Bandit and
  reinforcement-learning problems}.
\newblock \emph{Journal of Machine Learning Research}, 7:\penalty0 1079--1105,
  2006.

\bibitem[Freedman(1975)]{Freedman75OT}
David~A. Freedman.
\newblock \href{https://projecteuclid.org/euclid.aop/1176996452}{On tail
  probabilities for martingales}.
\newblock \emph{The Annals of Probability}, pages 100--118, 1975.

\bibitem[Garivier and Kaufmann(2016)]{Garivier16OB}
Aur{\'e}lien Garivier and Emilie Kaufmann.
\newblock \href{http://proceedings.mlr.press/v49/garivier16a.pdf}{Optimal
  best-arm identification with fixed confidence}.
\newblock In \emph{Conference on Learning Theory}, 2016.

\bibitem[Jamieson and Talwalkar(2016)]{Jamieson16NS}
Kevin Jamieson and Ameet Talwalkar.
\newblock \href{http://proceedings.mlr.press/v51/jamieson16.pdf}{Non-stochastic
  best-arm identification and hyperparameter optimization}.
\newblock In \emph{International Conference on Artificial Intelligence and
  Statistics}, 2016.

\bibitem[Kalyanakrishnan et~al.(2012)Kalyanakrishnan, Tewari, Auer, and
  Stone]{Kalyanakrishnan12PA}
Shivaram Kalyanakrishnan, Ambuj Tewari, Peter Auer, and Peter Stone.
\newblock
  \href{https://www.cse.iitb.ac.in/~shivaram/papers/ktas_icml_2012.pdf}{PAC
  subset selection in stochastic multi-armed bandits}.
\newblock In \emph{International Conference on Machine Learning}, 2012.

\bibitem[Karnin et~al.(2013)Karnin, Koren, and Somekh]{Karnin13AO}
Zohar Karnin, Tomer Koren, and Oren Somekh.
\newblock \href{http://proceedings.mlr.press/v28/karnin13.pdf}{Almost optimal
  exploration in multi-armed bandits}.
\newblock In \emph{International Conference on Machine Learning}, 2013.

\bibitem[Kaufmann and Kalyanakrishnan(2013)]{Kaufmann13IC}
Emilie Kaufmann and Shivaram Kalyanakrishnan.
\newblock \href{http://proceedings.mlr.press/v30/Kaufmann13.pdf}{Information
  complexity in bandit subset selection}.
\newblock In \emph{Conference on Learning Theory}, 2013.

\bibitem[Kaufmann et~al.(2016)Kaufmann, Capp{\'e}, and Garivier]{Kaufmann16CBI}
Emilie Kaufmann, Olivier Capp{\'e}, and Aur{\'e}lien Garivier.
\newblock On the complexity of best-arm identification in multi-armed bandit
  models.
\newblock \emph{Journal of Machine Learning Research}, 17\penalty0
  (1):\penalty0 1--42, 2016.
\newblock URL \url{https://www.jmlr.org/papers/v17/kaufman16a.html}.

\bibitem[Lattimore and Szepesv{\'a}ri(2020)]{Lattimore20BA}
Tor Lattimore and Csaba Szepesv{\'a}ri.
\newblock \emph{Bandit Algorithms}.
\newblock Cambridge University Press, 2020.

\bibitem[Li et~al.(2016)Li, Jamieson, DeSalvo, Rostamizadeh, and
  Talwalkar]{Li16HA}
Lisha Li, Kevin Jamieson, Giulia DeSalvo, Afshin Rostamizadeh, and Ameet
  Talwalkar.
\newblock \href{https://arxiv.org/pdf/1603.06560.pdf}{{\tt Hyperband}: A novel
  bandit-based approach to hyperparameter optimization}.
\newblock \emph{arXiv preprint {\normalfont\texttt{arXiv:1603.06560}}}, 2016.

\bibitem[Locatelli et~al.(2016)Locatelli, Gutzeit, and
  Carpentier]{Locatelli16AO}
Andrea Locatelli, Maurilio Gutzeit, and Alexandra Carpentier.
\newblock \href{http://proceedings.mlr.press/v48/locatelli16-supp.pdf}{An
  optimal algorithm for the thresholding bandit problem}.
\newblock In \emph{International Conference on Machine Learning}, 2016.

\bibitem[Maron and Moore(1993)]{Maron93HR}
Oded Maron and Andrew Moore.
\newblock
  \href{https://papers.nips.cc/paper/841-hoeffding-races-accelerating-model-selection-search-for-classification-and-function-approximation.pdf}{
  Hoeffding Races: Accelerating model-selection search for classification and
  function approximation}.
\newblock In \emph{Neural Information Processing Systems}, 1993.

\bibitem[Mnih et~al.(2008)Mnih, Szepesv\'{a}ri, and Audibert]{Mnih08EB}
Volodymyr Mnih, Csaba Szepesv\'{a}ri, and Jean-Yves Audibert.
\newblock \href{http://icml2008.cs.helsinki.fi/papers/523.pdf}{Empirical
  {B}ernstein stopping}.
\newblock In \emph{International Conference on Machine Learning}, 2008.

\bibitem[Powers(1998)]{Powers98AE}
David Powers.
\newblock \href{http://aclweb.org/anthology/W98-1218}{Applications and
  explanations of Zipf's law}.
\newblock In \emph{New methods in language processing and computational natural
  language learning}. Association for Computational Linguistics, 1998.

\bibitem[Seldin and Lugosi(2017)]{Seldin17IP}
Yevgeny Seldin and G{\'{a}}bor Lugosi.
\newblock \href{http://proceedings.mlr.press/v65/seldin17a/seldin17a.pdf}{An
  improved parametrization and analysis of the \texttt{{EXP3++}} algorithm for
  stochastic and adversarial bandits}.
\newblock In \emph{Conference on Learning Theory}, 2017.

\bibitem[Seldin and Slivkins(2014)]{Seldin14OP}
Yevgeny Seldin and Aleksandrs Slivkins.
\newblock \href{http://proceedings.mlr.press/v32/seldinb14.pdf}{One practical
  algorithm for both stochastic and adversarial bandits}.
\newblock In \emph{International Conference on Machine Learning}, 2014.

\end{thebibliography}
